\documentclass{article}
\usepackage[main, final]{neurips_2026}
\usepackage{amsmath,amssymb,amsthm}
\usepackage{mathtools}
\usepackage{hyperref}
\usepackage{enumitem}
\usepackage{natbib}
\setcitestyle{authoryear,square,citesep={,},aysep={,},yysep={,}}

\usepackage[utf8]{inputenc} %
\usepackage[T1]{fontenc}    %

\usepackage{hyperref}       %
\usepackage{url}            %
\usepackage{booktabs}       %
\usepackage{amsfonts}       %
\usepackage{nicefrac}       %
\usepackage[verbose=false]{microtype} %
\usepackage{xcolor}         %
\usepackage{graphicx}
\usepackage{caption}
\usepackage{float}      
\usepackage{amsmath,amssymb,amsfonts}
\usepackage{amsthm}

\usepackage{wrapfig}
\usepackage{booktabs}       %
\usepackage{amsfonts}       %
\usepackage{nicefrac}       %
\usepackage{graphicx}
\usepackage[hypcap=false]{caption}
\usepackage{subcaption} 
\usepackage{float}      
\usepackage{amsmath,amssymb,amsfonts}
\usepackage{amsthm}
\usepackage{colortbl}
\usepackage{array}
\usepackage{tabularx}
\usepackage{titlesec}

\newtheorem{theorem}{Theorem}
\newtheorem{proposition}[theorem]{Proposition}

\newtheorem{remark}{Remark}

\titleformat{\paragraph}[runin]
  {\normalfont\normalsize\bfseries}
  {}
  {0pt}
  {}

\titlespacing*{\paragraph}
  {0pt}
  {1.2ex plus 0.4ex minus 0.2ex}
  {0.6em}
\theoremstyle{plain}
\theoremstyle{remark}
\usepackage{tabularx}
\usepackage[table]{xcolor}
\definecolor{psoftgreen}{RGB}{235,243,252} %
\definecolor{avgblue}{RGB}{255,239,219}    %
\usepackage{adjustbox}
\usepackage{makecell}
\usepackage{placeins}
\usepackage{titletoc}
\usepackage{etoolbox}
\usepackage{enumitem}
\newcommand{\vthead}[1]{\rotatebox[origin=c]{90}{\strut #1}}%
\newcommand{\loft}{\textsc{LOFT}}

\newcommand{\pprin}{P_{\mathrm{prin}}}
\newcommand{\pgrad}{P_{\mathrm{grad}}}
\newcommand{\pskew}{P_{\mathrm{skew}}}
\newcommand{\Torth}{T_{\mathrm{orth}}}
\newcommand{\Tfree}{T_{\mathrm{free}}}
\DeclareMathOperator{\skewsym}{skew}
\newif\ifAppendixSectionBreaks
\AppendixSectionBreaksfalse

\hypersetup{colorlinks=true,linkcolor=red!70!black,linktocpage=false,citebordercolor=blue!70!black,citecolor=blue!70!black,anchorcolor=blue!70!black}

\title{LOFT: Low-Rank Orthogonal Fine-Tuning
via Task-Aware Support Selection}

\author{%
  Lanxin Zhao$^{1}$ \quad Bamdev Mishra$^{2}$ \quad Pratik Jawanpuria$^{3}$ \quad Lequan Lin$^{1}$ \\
  \textbf{Dai Shi}$^{4}$ \quad \textbf{Junbin Gao}$^{1}$ \quad \textbf{Andi Han}$^{1}$ \\
  $^{1}$The University of Sydney, Australia \quad
  $^{2}$Microsoft India \\
  $^{3}$Indian Institute of Technology Bombay, India \quad
  $^{4}$University of Cambridge, UK \\
  \texttt{lzha6608@uni.sydney.edu.au}, 
  \texttt{andi.han@sydney.edu.au} 
}

\begin{document}
\maketitle
\begin{abstract}
Orthogonal parameter-efficient fine-tuning (PEFT) adapts pretrained weights through structure-preserving multiplicative transformations, but existing methods often conflate two distinct design choices: the subspace in which adaptation occurs and the transformation applied within that subspace. This paper introduces LOFT, a low-rank orthogonal fine-tuning framework that explicitly separates these two components. By viewing orthogonal adaptation as a  multiplicative subspace rotation, LOFT provides a unified formulation that recovers the multiplicative core of representative orthogonal PEFT methods, including coordinate-, butterfly-, Householder- (with a learnable support) and principal-subspace-based variants (up to scaling and relaxation). More importantly, this perspective exposes support selection as a central design axis rather than a byproduct of a particular parameterization. We develop a first-order analysis showing that useful adaptation supports should be informed by the downstream training signal, motivating practical task-aware support selection strategies. Across language understanding, visual transfer, mathematical reasoning, and multilingual out-of-distribution adaptation, 
LOFT with task-aware supports improves performance at matched trainable-parameter and transform settings under a shared protocol. 
This suggests that, under a fixed transform and budget, the choice of support affects performance, making systematic support selection an important direction for improving orthogonal PEFT.
\end{abstract}

\section{Introduction}

Foundation models are increasingly being adapted to a wide range of downstream tasks, domains, and deployment environments. As model scale grows, however, maintaining a separate fully fine-tuned model for each task becomes increasingly costly in optimization, storage, and inference, especially in multi-task or resource-constrained settings \citep{hu2021lora,han2024parameter}. Parameter-efficient fine-tuning (PEFT) addresses this problem by keeping the pretrained backbone fixed and learning only a small task-specific modification.

Within reparameterization-based PEFT, two complementary paradigms have emerged. One line is additive and low-rank: LoRA \citep{hu2021lora} models downstream adaptation as a compact low-dimensional perturbation, and follow-up methods such as PiSSA \citep{meng2024pissa} and DoRA \citep{liu2024dora} refine how that perturbation is initialized or decomposed. A second line, which is the focus of this paper, is multiplicative and geometric: orthogonal fine-tuning adapts frozen weights through multiplicative transforms that aim to preserve pretrained structure while still allowing task-specific adaptation \citep{qiu2023controlling,liu2024parameterefficient,ma2024qgoft,yuan2024bridging,wu2025memory}. The appeal of the orthogonal view is geometric: unlike additive low-rank updates, a multiplicative orthogonal transform exactly preserves the structure of the pretrained weight, so adaptation is constrained to act as a change of basis on top of a fixed pretrained geometry. Although these lines are often presented separately, both rely on the same underlying picture: effective adaptation is often concentrated on a low-dimensional subspace, which we call the \emph{adaptation support}, and methods differ primarily in what transformation is allowed to act on that support.

A characteristic feature of existing orthogonal PEFT methods is that the support and the transform are tightly coupled inside a single parameterization.
Block-diagonal transforms \citep{qiu2023controlling}, butterfly factorizations \citep{liu2024parameterefficient}, Givens rotations \citep{ma2024qgoft}, Householder reflections \citep{yuan2024bridging}, and principal-subspace adaptation \citep{wu2025memory} are introduced as different ways to parameterize the orthogonal action, but each of them couples the effective support with a particular transform parameterization. As a result, the support is treated as a built-in property of the method rather than as a design variable that one can compare across methods. This coupling makes it hard to tell whether a reported improvement comes from a better orthogonal parameterization, a different effective support, or a combination of the two.

To make this distinction explicit, we introduce \textbf{Low-Rank Orthogonal Fine-Tuning (LOFT)}, a multiplicative subspace-rotation framework that separates the adaptation support \(P_r\) from the in-subspace transform \(T_r \in \mathbb R^{r \times r}\), with the update as
\(W^+ x=W_0\!\left(I_{d_{\mathrm{in}}}+P_r^{\top}(T_r-I_r)P_r\right) x\),
so that a row-orthonormal matrix \(P_r \in \mathbb R^{r \times d_{\rm in}}\)  determines \emph{where} adaptation acts and \(T_r\) determines \emph{how} the selected subspace is transformed. 
Under this view, several existing  orthogonal PEFT methods can be expressed as specific instances of this framework by varying the support basis, factor width, factor depth, and in-subspace transform class, rather than appearing as fundamentally distinct mechanisms.

This reformulation poses a natural question: \emph{which support is most useful for downstream adaptation, and is there a principled, task-aware way to construct \(P_r\)?} Existing orthogonal PEFT methods choose the support in a task-agnostic way---as a fixed coordinate or butterfly partition of the input space \citep{qiu2023controlling,liu2024parameterefficient,ma2024qgoft}, or as the principal right singular subspace of the pretrained weight \citep{wu2025memory}. Such choices use no information about the downstream task, and there is no a priori reason for the resulting \(P_r\) to align with directions along which the downstream loss is most easily reduced, particularly when the downstream task departs from the pretraining distribution. Motivated by a first-order analysis of the orthogonal-adaptation training dynamics, we propose a complementary, task-aware way to construct \(P_r\) from downstream \textit{gradient} signals computed at the pretrained weight.

Gradient-informed adaptation has several precedents. LoRA-GA \citep{wang2024loraga} initializes an additive low-rank update from the gradient, GoRA \citep{he2025gora} uses gradients for rank allocation and initialization, and GiVA \citep{gangwar2026giva} constructs gradient-informed bases for vector-based adaptation. LOFT differs in the role assigned to the gradient signal: downstream gradients select \emph{where} a multiplicative orthogonal transform acts, so the key quantity to adapt is \(\mathrm{skew}(W_0^\top G)\) rather than \(G\) alone. 
We refer to Appendix~\ref{sec:related_work} for broader positioning of our work.

Our main contributions are summarized as follows:

\begin{itemize}
\item \textbf{Framework.} We introduce LOFT, an orthogonal PEFT framework that decomposes the adaptation into a support \(P_r\) and an in-subspace transform \(T_r\). When \(T_r \) is orthogonal, the framework preserves the pretrained structure exactly. We show existing representative orthogonal PEFT methods can be expressed as specific instances under this framework
(Section~\ref{sec:recoveries_main}).

\item \textbf{Gradient-informed support selection.} Within the framework, we give a first-order analysis showing that the local training signal of an orthogonal update is governed by \(\mathrm{skew}(W_0^{\top}G)\), where $G$ is the gradient at $W_0$, and use this analysis to motivate two practical gradient-informed supports, \textsc{GradSVD} and \textsc{SkewGrad}, as new options for \(P_r\) to the principal-weight support used by existing orthogonal PEFT (such as PSOFT \citep{wu2025memory}). 

\item \textbf{Controlled empirical study.} We use the framework to compare principal and gradient-informed supports at matched trainable-parameter and transform settings under a shared protocol across language understanding, visual transfer, mathematical reasoning, and multilingual adaptation. The experiments indicate that the support is a meaningful design axis: principal-support LOFT matches the closest spectral baseline, PSOFT, up to additional scaling and relaxation, and support choice affects performance under a fixed transform and budget, with the preferred gradient-informed construction depending on the task family.
\end{itemize}

\section{Methodology}

\subsection{Preliminary on Orthogonal Fine-tuning}

We consider orthogonal parameter-efficient adaptation of a pretrained linear weight matrix \(W_0\in\mathbb{R}^{d_{\mathrm{out}}\times d_{\mathrm{in}}}\), and let \(L(W)\) denote the downstream objective evaluated at weight \(W\). We write \(O(d)\) for the orthogonal group and \(\mathrm{Skew}(d):=\{S\in\mathbb{R}^{d\times d}:S^\top=-S\}\) for the space of skew-symmetric matrices. We also define ${\rm skew}(A) = (A  - A^\top)/2$. For a row-orthonormal support \(P_r\in\mathbb{R}^{r\times d_{\mathrm{in}}}\) with \(P_rP_r^\top=I_r\), the matrix \(\Pi_r:=P_r^\top P_r\) is the orthogonal projection onto the selected \(r\)-dimensional right subspace. 

\textbf{Orthogonal Fine-tuning.}
Orthogonal fine-tuning \citep{qiu2023controlling} proposes to adapt the pretrained weight $W_0$ by an orthogonal transformation, i.e., 
\begin{equation*}
    W^+ = W_0 U, \qquad W_0 \in \mathbb R^{d_{\rm out} \times d_{\rm in}},U \in O(d_{\rm in}),
\end{equation*}
which preserves the pairwise neuron similarity of the pretrained model, i.e., $W^+ W^{+\top} = W_0 W_0^\top$. By contrast, the standard  additive adaptation $W^+ = W_0 + \Delta W$ places no structural constraint on $\Delta W$, and thus offers no guarantee on the geometry of the resulting weight matrix. The seminal work, LoRA \citep{hu2021lora}, restricts $\Delta W$ to have rank at most $r$  via the factorization $\Delta W = AB$, achieving parameter efficiency but still lacking geometric control. LOFT builds on the geometry-preserving view of orthogonal fine-tuning, while confining the transformation to an explicit low-dimensional support subspace.

\subsection{LOFT: a Unified Framework for Orthogonal PEFT}
\label{sec:loft_parameterization}

\begin{figure}[t]
\centering
\includegraphics[width=0.97\linewidth]{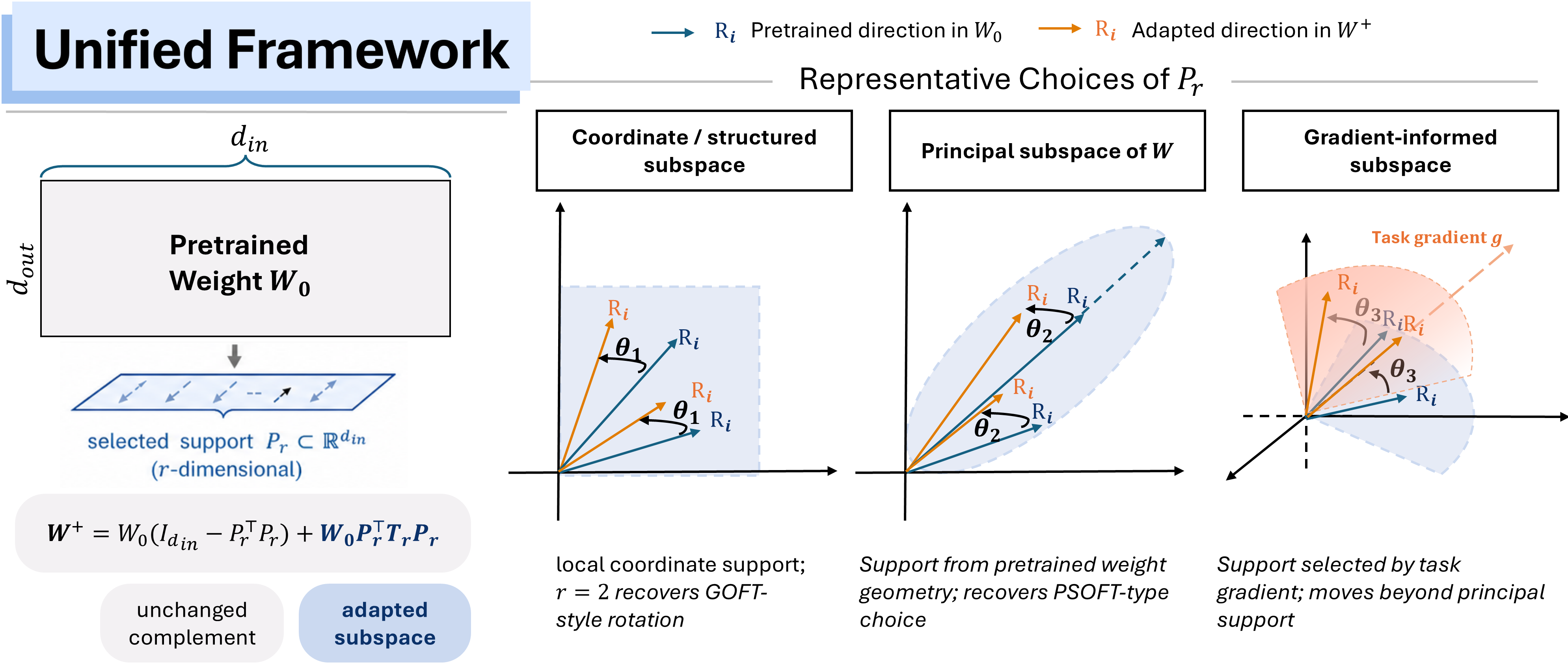}
\captionof{figure}{
\textbf{LOFT as right-subspace adaptation.} Left: the transform \(T_r\) acts only on the selected right support \(P_r\); the orthogonal complement is unchanged. Right: coordinate, principal and gradient-informed supports instantiate different orthogonal PEFT regimes within one right-multiplicative form.
}
\label{fig:unified_framework}
\end{figure}

We formulate parameter-efficient orthogonal adaptation as a multiplicative transform acting on a selected low-dimensional input subspace. Let $P \in O(d_{\rm in})$, and denote $P_r$ as the first $r$ rows of $P$, \(P_r\in\mathbb{R}^{r\times d_{\mathrm{in}}}\), which forms a row-orthonormal support basis with \(P_rP_r^\top=I_r\), and let \(T_r\in\mathbb{R}^{r\times r}\) be the learnable transform within this support. We parameterize the update as
\[
S(P_r,T_r):= P^\top \begin{bmatrix}
    T_r &0 \\
    0 & I_{d_{\rm in} - r}
\end{bmatrix} P =   I_{d_{\mathrm{in}}}+P_r^\top(T_r-I_r)P_r,
\qquad
W^+=W\,S(P_r,T_r).
\]
Equivalently, $W^+
=
W(I_{d_{\mathrm{in}}}-P_r^\top P_r)+WP_r^\top T_rP_r,$
so the selected subspace is transformed by \(T_r\), while its orthogonal complement is left unchanged. This cleanly separates two design choices: $P_r$ determines \textit{where} adaptation acts, and $T_r$ determines \emph{how} the selected subspace is transformed.  Notably, when \(T_r\in O(r)\), \(S(P_r,T_r) \in O(d_{\rm in})\). 

For fixed \(P_r\), \(T_r\mapsto S(P_r,T_r)\) embeds an isomorphic copy of \(O(r)\) as a closed subgroup of \(O(d_{\rm in})\), so the choice of \(P_r\) selects which subgroup training optimizes over.

In the following proposition, we show that the favorable properties of orthogonal fine-tuning, such as spectral norm and rank, are preserved under the LOFT parameterization. Prior work associates these invariants with retaining pretrained representations, training stability and reduced forgetting \citep{qiu2023controlling,qiu2025scalable}.

\begin{proposition}[Geometry Preservation]
\label{prop:loft_orthogonality}
If \(P_rP_r^\top=I_r\) and \(T_r\in O(r)\), then \(S(P_r,T_r)\in O(d_{\mathrm{in}})\). Therefore, for \(W^+=W\,S(P_r,T_r)\),
\[
W^+W^{+\top}=WW^\top,\qquad
\operatorname{rank}(W^+)=\operatorname{rank}(W),\qquad
\sigma_i(W^+)=\sigma_i(W)\ \forall i,
\]
where $\sigma_i$ are singular values,
and consequently \(\|W^+\|_F=\|W\|_F\) and \(\|W^+\|_2=\|W\|_2\).
\end{proposition}

More generally, LOFT can compose multiple orthogonal transforms,
\[
W^+
=
W\prod_{\ell=1}^{L} S\!\left(P_{r_\ell}^{(\ell)},T_{r_\ell}^{(\ell)}\right),
\]
in which each factor lies in an embedded copy of $O(r_\ell)$, $\ell=1,\dots,L$; products across different supports can access transformations beyond any one such subgroup. This exposes multiple design axes: factor depth \(L\), factor width \( \{ r_\ell \}\), and support selection \(\{P_{r_\ell}^{(\ell)} \}\). Figure~\ref{fig:unified_framework} summarizes the decomposition.

\subsection{Recovering Orthogonal PEFT Methods}
\label{sec:recoveries_main}

In this section, we show how the core multiplicative updates of several representative orthogonal fine-tuning methods can be written as special cases of LOFT under different choices of factor depth, factor width, and subspace basis.  
Table~\ref{tab:recoveries} summarizes the design-level correspondences, while Appendix~\ref{app:recoveries} provides expanded derivations.

\begin{center}
\vspace{0.25em}
\captionsetup{type=table, skip=4pt}
\captionof{table}{
LOFT design map of representative orthogonal PEFT branches.
}
\label{tab:recoveries}
\scriptsize
\renewcommand{\arraystretch}{1.08}
\setlength{\tabcolsep}{3.4pt}
\begin{tabularx}{\linewidth}{@{}lcccX@{}}
\toprule
Methods & \(P_r\) & Width/Depth & Transform & Core mechanism \\
\midrule
Full OFT \citep{qiu2023controlling}
& \(I_{d_{\mathrm{in}}}\)
& \(r=d_{\mathrm{in}},\ L=1\)
& \(T_r\in O(d_{\mathrm{in}})\)
& Full orthogonal fine-tuning, \(W^+=WT_r\). \\

\addlinespace[1pt]
Block OFT \citep{qiu2023controlling}
& coordinate blocks
& block width \(b\)
& \(T_b\in O(b)\)
& Disjoint coordinate-block rotations forming a block-diagonal transform. \\

\addlinespace[1pt]
GOFT \citep{ma2024qgoft}
& coordinate pairs
& width \(2\), depth \(L\)
& \(T_2(\theta)\)
& Products of 2D coordinate-plane rotations. \\

\addlinespace[1pt]
BOFT \citep{liu2024parameterefficient}
& butterfly blocks
& width \(b\), multi-stage
& \(T_b\in O(b)\)
& Sparse block rotations composed across butterfly stages. \\

\addlinespace[1pt]
HRA \citep{yuan2024bridging}
& learnable \(u_i^\top\)
& width \(1\), depth \(r\)
& \(T_1=-1\)
& Products of Householder reflections, \(I-2u_i u_i^\top\). \\

\addlinespace[1pt]
PSOFT \citep{wu2025memory}
& \(V_r^\top\)
& \(r<d_{\mathrm{in}},\ L=1\)
& \(T_r\in O(r)\)
& Rotation inside the principal right-singular subspace of \(W_0\). \\
\bottomrule
\end{tabularx}
\vspace{-0.45em}
\end{center}

Several specializations are particularly relevant for subsequent discussions. 
When \(r=2\) and \(P_r\) is formed by two rows of a permutation matrix, LOFT applies a Givens rotation to the selected coordinate plane, thus recovering GOFT \citep{ma2024qgoft}. 
When \(r=1\), \(P_r=u^\top\) with \(\|u\|_2=1\), and \(T_1=-1\), LOFT gives \(S(u^\top,-1)=I-2uu^\top\). 
Thus products of Householder reflections \citep{yuan2024bridging} correspond to LOFT depth with learnable width-one subspace bases and fixed in-subspace reflections. 
When \(P_r\) is selected to be the top right-singular subspace of \(W_0\), i.e., \(P_r=V_r^\top\), where \(W_0=U_r\Sigma_rV_r^\top+U_\perp\Sigma_\perp V_\perp^\top\), LOFT gives \(W^+=U_\perp\Sigma_\perp V_\perp^\top+U_r\Sigma_rT_rV_r^\top\), which reduces to PSOFT \citep{wu2025memory} up to additional scaling and relaxation. Throughout, \emph{recover} refers to this multiplicative core: HRA corresponds to LOFT with a \emph{learnable} width-one support, unlike the supports fixed before training in our experiments, and PSOFT corresponds only up to its additional scaling and relaxation terms.

\subsection{Task-Aware Adaptation Support Selection}
\label{sec:subspace_selection}

Despite the empirical success of parameter-efficient orthogonal fine-tuning, the selection of the support subspace has received little attention.  Existing works (e.g., PSOFT \citep{wu2025memory}) are mostly task-agnostic, with the support determined by pretrained structure. This offers no guarantee that the chosen subspace aligns with the downstream  signal. \textit{In the following, we provide a principled way to determine the support based on the initial training signal.}

Specifically, we  parameterize the low-dimensional orthogonal transform $T_r = Q(E) \in O(r)$ for $E \in {\rm Skew}(r)$, where $Q(\cdot):  {\rm Skew}(r) \rightarrow O(r)$ such that $Q(0) = I_r, {\rm d} Q(t E)/ {\rm d} t\vert_{t=0} = E$. In Riemannian-optimization terms, these two conditions state that $Q$ is a retraction on $O(r)$ at the identity \citep{absil2008optimization,boumal2014manopt}. The matrix exponential and a suitably rescaled Cayley map \citep{qiu2023controlling,wu2025memory} are two such retractions. Consider the proposed LOFT update \(W^+=W_0(I+P_r^\top(Q(E)-I_r)P_r)\), where we optimize over $E \in {\rm Skew}(r)$ from zero initialization because when $E = 0$, $W^+ = W_0$. 
The next proposition derives the gradient with respect to $E$ at initialization.

\begin{proposition}[Initial Training Gradient]
\label{prop:signal-strength}
Let \(W_0 \) be the pretrained weight, and let \(G=\nabla_W L|_{W=W_0}\) be the gradient of downstream loss at $W_0$. We define $\widetilde L(E_t) := L \big(W_0(I + P_r^\top(Q(E_t)-I_r)P_r) \big)$.  For any \(E\in\mathrm{Skew}(r)\), it satisfies that
\[
\frac{d}{dt}
L\!\left(W_0\!\left(I+P_r^\top(Q(tE)-I_r)P_r\right)\right)\Big|_{t=0}
=
\left\langle P_r \mathrm{skew}(W_0^\top G) P_r^\top,E\right\rangle .
\]
Equivalently, on \(\mathrm{Skew}(r)\), the projected gradient and its norm satisfy
\[
\nabla_{E_t} \widetilde L|_{E_t=0}=P_r \mathrm{skew}(W_0^\top G) P_r^\top,
\qquad
\|\nabla_{E_t} \widetilde L|_{E_t=0}\|_F^2=\|P_r  \mathrm{skew}(W_0^\top G) P_r^\top\|_F^2
\leq
2\sum_{k=1}^{\lfloor r/2\rfloor}\mu_k^2.
\]
where the nonzero eigenvalues of \(\mathrm{skew}(W_0^\top G) \) are \(\{\pm i\mu_k\}\) with \(\mu_1\geq\mu_2\geq\cdots\geq0\). The bound is attained when the row space of \(P_r\) contains an invariant subspace of \(\mathrm{skew}(W_0^\top G)\) associated with its largest \(\lfloor r/2\rfloor\) skew-eigenvalue pairs.
\end{proposition}

Proposition \ref{prop:signal-strength} proves that the training signal (measured by gradient norm) is maximized when the support $P_r$ spans  the top invariant subspace of $\mathrm{skew}(W_0^\top G)$.\footnote{We remark that a different parameterization of the low-dimensional orthogonal transform may induce a different support. Nevertheless, the same principle can be applied to derive a gradient-informed support selection criterion.} The proposition covers only the orthogonal, skew-parameterized transform \(T_r=Q(E)\); the unconstrained transform \(\Tfree\) used in our ablations lies outside it.

\textbf{Gradient-based support selection.}
Motivated by the first-order analysis, we propose two gradient-based support selection strategies that capture downstream task information.
\begin{itemize}[leftmargin=0.3in]
\vspace{-5pt}
    \item \textsc{SkewGrad}: choosing $P_r$ to span the top invariant subspace of $\mathrm{skew}(W_0^\top G)$. 

    \item \textsc{GradSVD}: choosing $P_r$ to be the top right singular subspace of the gradient $G$.
\vspace{-5pt}
\end{itemize}
In our \textit{implementation}, \(G\) is accumulated once over a small number of calibration mini-batches (Appendix~\ref{app:grad_support}), and the resulting support is fixed during training. Section~\ref{sect:add} reports the one-off construction time and transient memory separately from the training loop.

LOFT  is implemented as a fixed-support PEFT adapter on the target linear layers. The adapter applies \(x\mapsto S(P_r,T_r) x  = x+ P_r^\top(T_r-I_r)P_r x\) before the frozen linear projection, which is equivalent to using the merged weight \(W^+=WS(P_r,T_r)\). At inference time, LOFT performs the additive update $\Delta W = W^+-W = WP_r^\top(T_r-I_r)P_r$. We can show that $\operatorname{rank}(\Delta W)\le r$.
Thus we can view LOFT as multiplicative from its parameterization but low-rank in its induced additive update.

\begin{remark}[Scope of the first-order criterion]
\label{rmk:opt}
Existing works \citep{malladi2023kernel,jang2024lora}  have shown that fine-tuning often operates in a near-linear regime: the model remains close to pretrained initialization and is well-approximated by its first-order expansion, the so-called lazy-training or NTK regime \citep{jacot2018neural}. Under gradient flow in this regime, the loss evolves as $L(W^+) \approx L(W) - t \| \nabla_{E_t} \widetilde L \vert_{E_t = 0}  \|_F^2$. Proposition~\ref{prop:signal-strength} is, however, initialization-local and first-order: it identifies the support with the largest initial signal and does not imply optimality over a finite training horizon. Whether the initial ranking stays predictive is an empirical question that depends on how the gradient evolves during training (Section~\ref{sec:training_dynamics}). In contrast, choosing $P_r$ as the top subspace of $W_0$ (as in PSOFT \citep{wu2025memory}) maximizes pretrained-weight energy $\| W_0 P_r^\top \|^2_F$, which is sub-optimal unless the downstream gradient $G$ aligns with the dominant right-singular directions of $W_0$. In practice, this alignment is not guaranteed.
\end{remark}

\section{Experiments}

Our experiments evaluate the support-selection view of orthogonal PEFT across encoder-only language understanding, mathematical reasoning and visual transfer, organized around the questions: \textbf{(1)} Does LOFT recover the principal-subspace orthogonal regime represented by PSOFT, up to scaling and relaxation? \textbf{(2)} Do task-aware supports improve over principal-weight supports at matched trainable-parameter and transform settings?
Throughout, we report trainable parameters, peak memory, and task performance. The central empirical observation is that LOFT's gains stem from selecting a task-aligned support, not merely from a new orthogonal parameterization.

\textbf{Notation.}
For LOFT, \(P_r\) denotes the  adaptation support and \(T_r\) the in-subspace transform. We write \(P_r=\pprin\) for the principal-weight support from the top right singular subspace of \(W_0\), \(P_r=\pgrad\) for the GradSVD support from the top right singular subspace of the downstream gradient \(G\), and \(P_r=\pskew\) for the SkewGrad support from the top invariant subspace of \(\mathrm{skew}(W_0^\top G)\). 
In result tables, \(\loft_{r}^{P}\) denotes a LOFT adapter with rank \(r\) and support \(P\).
Main tables list only the support \(P_r\). Full \(P_r/T_r\) ablations are reported for GLUE and VTAB in Appendix~\ref{app:support_transform_ablations}.

\textbf{Baselines.}
Across benchmarks, we compare LOFT with representative reparameterization-based PEFT methods. The low-rank family includes LoRA, PiSSA, DoRA, and LoRA-XS \citep{hu2021lora,meng2024pissa,liu2024dora,balazy2024loraxs}. The orthogonal family includes GOFTv2/qGOFTv2, BOFT, OFTv2, and PSOFT \citep{ma2024qgoft,liu2024parameterefficient,qiu2025scalable,wu2025memory}. GOFTv2 and qGOFTv2 are the labels \citet{wu2025memory} use for their reimplementations of GOFT and qGOFT \citep{ma2024qgoft}. Rows marked $^{\dagger}$ in Tables~\ref{tab:glue_main} and~\ref{tab:vtab_main} are published values from \citet{wu2025memory}; all other rows are measured in our codebase. Benchmark-specific protocol differences and additional ablations are described in the corresponding appendix sections.

The code is available at \url{https://github.com/Kris-camp/loft}.

\subsection{GLUE on DeBERTaV3-base}

\begin{table}[t]
\centering
\footnotesize
\setlength{\tabcolsep}{4pt}
\caption{
GLUE on DeBERTaV3-base. $^{\dagger}$: published values of \citet{wu2025memory}, same protocol; other rows: five-seed means in our codebase (SDs in Tables~\ref{tab:glue_rerun} and~\ref{tab:app_glue_loft_stds}). Trainable Params excludes the fixed support (Table~\ref{tab:per_task_storage}); Mem.(GB) is the training-loop peak; $^{\dagger}$ values are published, other rows are measured on one V100-SXM2. Settings: Appendix~\ref{app:glue_details}; more baselines: Table~\ref{tab:app_glue_full}.
}
\label{tab:glue_main}
\begin{tabular}{lccccccccc}
\toprule
Method & \makecell{Trainable\\Params} & Mem.(GB) & CoLA & STS-B & RTE & MRPC & SST-2 & QNLI & \cellcolor{avgblue}Avg. \\
\midrule
BOFT$_{m=2}^{b=8}$$^{\dagger}$ & 1.41M & 6.3 & 68.85 & 91.09 & 83.60 & 88.40 & 95.28 & 93.78 & \cellcolor{avgblue}86.83 \\
OFTv2$_{b=32}$$^{\dagger}$ & 1.29M & 4.5 & 66.79 & 91.22 & 84.03 & 89.61 & 93.72 & 92.64 & \cellcolor{avgblue}86.34 \\
LoRA$_{r=8}$$^{\dagger}$ & 1.33M & 4.5 & 67.98 & 91.60 & 84.87 & 90.20 & 95.28 & 93.89 & \cellcolor{avgblue}87.30 \\
DoRA$_{r=8}$$^{\dagger}$ & 1.41M & 5.8 & 67.06 & 91.60 & \textbf{87.19} & 90.49 & 95.23 & 94.09 & \cellcolor{avgblue}87.61 \\
PSOFT$_{r=46}$$^{\dagger}$ & 0.08M & 4.1 & 70.42 & 91.56 & 86.74 & 90.49 & 95.55 & 93.47 & \cellcolor{avgblue}88.04 \\
\midrule
PSOFT$_{r=46}$ & 0.08M & 4.1 & 71.57 & 91.45 & 85.88 & 90.50 & 95.24 & 93.57 & \cellcolor{avgblue}88.03 \\
LoRA$_{r=8}$ & 1.33M & 4.5 & \textbf{72.04} & 91.38 & 84.41 & 90.69 & 94.95 & 93.96 & \cellcolor{avgblue}87.91 \\
LoRA-GA$_{r=8}$ & 1.33M & 4.5 & 65.62 & 91.25 & 84.17 & 90.85 & 94.65 & 93.79 & \cellcolor{avgblue}86.72 \\
GiVA & 0.09M & 4.7 & 70.32 & 91.60 & 84.65 & 91.34 & 95.11 & 93.25 & \cellcolor{avgblue}87.71 \\
\midrule
\rowcolor{psoftgreen}
\loft$_{r=46}^{\scriptscriptstyle \pprin}$ 
& 0.07M & 4.1 & 70.52 & 91.37 & 86.19 & 90.90 & \textbf{95.73} & 93.63 & \cellcolor{avgblue}88.06 \\
\rowcolor{psoftgreen}
\loft$_{r=46}^{\scriptscriptstyle \pgrad}$
& 0.07M & 4.1 & 70.63 & 91.32 & 87.05 & 91.67 & 95.28 & \textbf{94.39} & \cellcolor{avgblue}88.39 \\
\rowcolor{psoftgreen}
\loft$_{r=46}^{\scriptscriptstyle \pskew}$ 
& 0.07M & 4.1 & 72.03 & \textbf{91.75} & 86.83 & \textbf{92.36} & 95.42 & 93.97 & \cellcolor{avgblue}\textbf{88.73} \\
\bottomrule
\end{tabular}
\end{table}

\paragraph{Experimental setting.}
We begin with a controlled encoder-only benchmark on six GLUE tasks (CoLA, STS-B, RTE, MRPC, SST-2, and QNLI), using DeBERTaV3-base \citep{he2021debertav3}. We follow the same hold-out validation protocol as PSOFT \citep{wu2025memory}: the original validation split is partitioned into validation/test subsets with a fixed seed, model selection is performed on the new validation split, and the reported score is the held-out test score of the best validation checkpoint. Our results are averaged over five random seeds. Following GLUE convention \citep{wang2018glue}, we report Matthews correlation on CoLA, Pearson correlation on STS-B, and accuracy on the remaining tasks.

In the main text, we report the default orthogonal LOFT transform \(T_r=\Torth\) and focus on \(P_r=\pprin\) and the proposed \(P_r=\pgrad\), \(P_r=\pskew\). We report additional LOFT variants in Appendix~\ref{app:support_transform_ablations}.

\paragraph{Main results.}
From Table~\ref{tab:glue_main}, principal-support LOFT matches the principal-subspace orthogonal regime of PSOFT, up to additional scaling and relaxation, with a comparable average score (\(88.06\) vs.\ \(88.04\) published and \(88.03\) reproduced), slightly fewer trainable parameters (\(0.07\)M vs.\ \(0.08\)M), and the same \(4.1\)GB peak memory.
The main improvement comes from support selection. Replacing the principal support with the proposed SkewGrad support improves the six-task average to \(88.73\) without increasing trainable parameters or training-loop peak memory.
Under the same protocol, memory instrument and five seeds, our reproductions of the closest baselines reach \(88.03\) (PSOFT), \(87.91\) (LoRA), \(86.72\) (LoRA-GA \citep{wang2024loraga}) and \(87.71\) (GiVA \citep{gangwar2026giva}); given the seed variability, we treat these margins as descriptive rather than as formal significance claims (per-task standard deviations in Table~\ref{tab:glue_rerun}).
\(\loft_{r=46}^{\pskew}\) uses \(17.8\times\) fewer trainable parameters than LoRA\(_{r=8}\) (\(0.075\)M vs.\ \(1.327\)M over the 72 adapted matrices). The fixed support, however, is task-specific state that must be stored: counting it, LOFT stores \(3.890\)M scalars per task against \(1.327\)M for LoRA\(_{r=8}\), about half of LoRA's \(7.631\)M at matched rank \(r=46\), and \(49.6\%\) fewer materialized adapter scalars than same-rank PSOFT (Table~\ref{tab:per_task_storage}).
This supports the view that \(P_r\) is a meaningful design axis: under the same rank, transform class, trainable-parameter count and training-loop peak memory, replacing the pretrained-weight principal support with a loss-informed support improves performance.

\subsection{Mathematical Question Answering on MetaMathQA-40K}

\begin{wraptable}{r}{0.55\columnwidth}
\vspace{-1.1\baselineskip}
\centering
\scriptsize
\renewcommand{\arraystretch}{1.05}
\setlength{\tabcolsep}{3.2pt}
\caption{
Mathematical reasoning on GSM8K and MATH under a matched low-budget setting; every row is trained and evaluated in our pipeline. LOFT rows are means \(\pm\) sample standard deviation over three runs; baseline rows are single runs. All methods run on a single A100-SXM GPU, except HRA, which runs on a single H200-SXM GPU due to OOM on A100-SXM. Mem.(GB) is the training-loop peak GPU memory, measured with the same instrumentation; support construction is reported separately (Section~\ref{sect:add}). Details are in Appendix~\ref{app:math_details}.
}
\label{tab:math_main}
\resizebox{\linewidth}{!}{
\begin{tabular}{lcc|cc}
\toprule
Method & \makecell{Trainable\\Params} & Mem.(GB) & GSM8K & MATH \\
\midrule
BOFT$_{m=1}^{b=2}$ & 0.786M & 63.24 & 37.33 & 4.38 \\
LoRA$_{r=1}$ & 0.524M & 49.50 & 36.32 & 4.68 \\
HRA$_{r=2,\lambda=\infty}$ & 0.524M & 124.92 & 36.85 & 5.02 \\
PSOFT$_{r=128}$ & 0.536M & 28.67 & 35.94 & 4.58 \\
\midrule
\rowcolor{psoftgreen}
\loft$_{r=128}^{\scriptscriptstyle \pprin}$ 
& 0.520M & 25.68 & \(34.45\pm0.90\) & \(4.68\pm0.35\) \\
\rowcolor{psoftgreen}
\loft$_{r=128}^{\scriptscriptstyle \pgrad}$ 
& 0.520M & 25.68 & \(\mathbf{39.63}\pm1.36\) & \(5.24\pm0.10\) \\
\rowcolor{psoftgreen}
\loft$_{r=128}^{\scriptscriptstyle \pskew}$ 
& 0.520M & 25.68 & \(\mathbf{39.57}\pm0.80\) & \(\mathbf{5.48}\pm0.04\) \\
\bottomrule
\end{tabular}
}
\vspace{-2.\baselineskip}
\end{wraptable}

\textbf{Experiment Settings.}
We fine-tune LLaMA2-7B \citep{touvron2023llama2} on MetaMathQA-40K \citep{yu2024metamath} and evaluate on GSM8K \citep{cobbe2021gsm8k} and MATH \citep{hendrycks2021math} in a matched low-budget setting where all main methods use \(0.52\)--\(0.79\)M trainable parameters, and report the finalized evaluation result under per-method tuned learning rates and a shared evaluation pipeline. Implementation details, learning rates, and evaluation scripts are deferred to Appendix~\ref{app:math_details}.

\paragraph{Main Results.} Table~\ref{tab:math_main} shows that the support-selection effect observed on GLUE also transfers to decoder-side mathematical reasoning. Principal-support LOFT remains close to PSOFT, while the task-aware supports improve performance under the same rank and trainable-parameter count: over three runs, GradSVD reaches \(39.63\pm1.36\) on GSM8K and SkewGrad \(39.57\pm0.80\), against \(34.45\pm0.90\) for the principal support, and every gradient-informed run exceeds every principal run on both tasks (lowest \(38.23\) vs.\ highest \(34.99\) on GSM8K; \(5.13\) vs.\ \(4.98\) on MATH). On GSM8K the two gradient-informed supports are within run-to-run variation of each other, so we do not rank them, consistent with Proposition~\ref{prop:signal-strength} being initialization-local rather than a guarantee on final rankings.

In terms of memory on LLaMA2-7B, all LOFT variants have a training-loop peak of \(25.68\)GB,
which is the lowest among the compared methods. 
This is \(10.4\%\) below PSOFT and \(48.1\%\) below LoRA (Appendix~\ref{app:math_details}).
Thus, gradient-informed LOFT improves over principal-support LOFT without increasing training-loop peak memory; the one-off support construction has its own transient peak, reported separately in Section~\ref{sect:add}.

\subsection{Visual Transfer on VTAB-1K}

\textbf{Experimental setting.}
Additionally, we evaluate visual transfer on VTAB-1K \citep{zhai2019vtab} by fine-tuning ViT-B/16 \citep{dosovitskiy2021vit} under the standard VTAB-1K transfer protocol. 
Each task uses 800 labeled training examples for adaptation and 200 validation examples for hyperparameter tuning, and performance is reported as top-1 test accuracy. 
In the main table, we report the proposed GradSVD and SkewGrad supports under the default orthogonal LOFT transform; principal-support controls and additional support/transform variants are reported in Appendix~\ref{app:support_transform_ablations}.

\paragraph{Main results.}
Table~\ref{tab:vtab_main} provides further evidence that support selection is a useful design axis beyond language tasks. 
SkewGrad gives the strongest category-balanced VTAB average, improving over the published PSOFT result from \(73.4\) to \(73.8\) (\(10\) wins, \(3\) ties and \(6\) losses over the 19 tasks), while using fewer trainable parameters (\(0.07\)M vs.\ \(0.08\)M), lower training-loop peak memory (\(5.70\)GB vs.\ \(6.65\)GB when PSOFT is measured in our codebase, \(14.3\%\) lower; \(8.1\%\) lower than the published \(6.2\)GB), and \(10.8\%\) less wall-clock time per task including support construction (Section~\ref{sect:add}).
Compared with standard low-rank baselines such as LoRA and DoRA, \(\loft_{r=46}^{\pskew}\) achieves the best average score with roughly an order of magnitude fewer trainable parameters. 
Since the \(r=46\) principal-support variant in Appendix~\ref{app:support_transform_ablations} does not improve the average, the gain is better explained by support selection than by rank alone.

\begin{table}[!tbp]
\vspace{-0.7em}
\centering
\captionsetup{skip=2pt}
\tiny
\renewcommand{\arraystretch}{1.05}
\setlength{\tabcolsep}{1.55pt}
\caption{
Experimental results of fine-tuned ViT-B/16 on VTAB-1K: top-1 accuracy (\%), with Avg.\ the category-balanced mean of the three group means (FFT's average is recomputed on this basis). Rows marked $^{\dagger}$ are published values of \citet{wu2025memory}; LOFT rows are five-seed means measured in our codebase. Mem.(GB) is training-loop peak GPU memory, measured for LOFT rows on a single NVIDIA Tesla V100-SXM2 GPU ($^{\dagger}$ rows report published values); PSOFT's \(6.2\)GB is the published value, and measured in our codebase with the same instrument as LOFT it is \(6.65\)GB. Full settings are in Appendix~\ref{app:vtab_details}.
}
\label{tab:vtab_main}
\vspace{-0.25em}
\begin{adjustbox}{width=\linewidth}
\begin{tabular}{lcc|ccccccc|cccc|cccccccc|c}
\toprule
& & & \multicolumn{7}{c|}{Natural} & \multicolumn{4}{c|}{Specialized} & \multicolumn{8}{c|}{Structured} & \\
\cmidrule(lr){4-10} \cmidrule(lr){11-14} \cmidrule(lr){15-22}
Method & \vthead{Trainable Params} & \vthead{Mem.(GB)}
& \vthead{Cifar100}
& \vthead{Caltech101}
& \vthead{DTD102}
& \vthead{Flower102}
& \vthead{Pets}
& \vthead{SVHN}
& \vthead{Sun397}
& \vthead{Camelyon}
& \vthead{EuroSAT}
& \vthead{Resisc45}
& \vthead{Retinopathy}
& \vthead{Clevr-Count}
& \vthead{Clevr-Dist}
& \vthead{DMLab}
& \vthead{KITTI-Dist}
& \vthead{dSpr-Loc}
& \vthead{dSpr-Ori}
& \vthead{sNORB-Azi}
& \vthead{sNORB-Ele}
& \cellcolor{avgblue}\vthead{Avg.} \\
\midrule
FFT$^{\dagger}$ & 85.9M & 8.2
& 70.7 & 89.3 & 69.5 & 99.0 & 90.4 & 81.7 & 54.9
& 85.4 & 93.6 & 83.8 & 74.5
& 58.3 & 51.5 & 43.2 & 75.0 & 73.1 & 48.7 & 16.4 & 30.0
& \cellcolor{avgblue}71.1 \\

GOFTv2$^{\dagger}$ & 0.08M & OOM
& \multicolumn{20}{c}{N/A} \\

qGOFTv2$^{\dagger}$ & 0.33M & OOM
& \multicolumn{20}{c}{N/A} \\

BOFT$_{m=2}^{b=8}$$^{\dagger}$ & 1.41M & 10.9
& 70.6 & 88.2 & 69.8 & 99.0 & 91.4 & 77.4 & 55.1
& 85.1 & 93.6 & 82.3 & 74.9
& 61.8 & 50.4 & 42.9 & 76.1 & 73.7 & 48.8 & 15.7 & 30.8
& \cellcolor{avgblue}70.9 \\

OFTv2$_{b=32}$$^{\dagger}$ & 1.29M & 7.7
& 68.5 & 88.9 & 67.5 & 98.4 & 89.5 & \textbf{86.9} & 53.6
& 86.0 & 94.1 & \textbf{84.2} & 74.6
& 58.7 & 56.4 & \textbf{46.7} & 78.5 & 81.1 & 48.1 & 17.3 & 32.5
& \cellcolor{avgblue}72.1 \\

LoRA$_{r=8}$$^{\dagger}$ & 1.33M & 9.9
& 71.4 & 88.4 & 70.1 & 99.0 & 91.4 & 76.6 & 55.7
& 85.9 & 94.2 & 83.3 & 74.1
& \textbf{72.0} & 54.3 & 43.0 & 76.6 & 74.8 & 48.6 & 16.4 & 31.8
& \cellcolor{avgblue}71.8 \\

PiSSA$_{r=8}$$^{\dagger}$ & 1.33M & 9.9
& 70.7 & 88.7 & 68.9 & 99.2 & 91.0 & 81.9 & 53.3
& 82.6 & 93.4 & 83.0 & 74.0
& 71.0 & \textbf{60.2} & 44.0 & 77.1 & 81.9 & 51.8 & 18.1 & 33.1
& \cellcolor{avgblue}72.3 \\

DoRA$_{r=8}$$^{\dagger}$ & 1.41M & 17.8
& 70.7 & 89.0 & 69.8 & 98.9 & 91.0 & 81.7 & 55.5
& 85.7 & 94.2 & 83.5 & 74.8
& 67.3 & 54.2 & 45.1 & 77.4 & \textbf{82.0} & 48.5 & 16.9 & 31.5
& \cellcolor{avgblue}72.3 \\

LoRA-XS$_{r=136}$$^{\dagger}$ & 1.33M & 6.6
& 68.5 & 89.4 & 68.4 & 98.7 & 90.9 & 84.5 & 54.1
& 84.0 & 94.3 & 80.8 & 73.6
& 60.0 & 57.7 & 45.8 & 79.6 & 80.6 & 48.1 & 17.4 & 30.8
& \cellcolor{avgblue}71.6 \\

PSOFT$_{r=46}$$^{\dagger}$ & 0.08M & 6.2
& \textbf{71.9} & 89.6 & \textbf{70.3} & 99.1 & 91.8 & \textbf{86.9} & \textbf{55.9}
& 84.6 & 94.2 & 82.4 & 75.2
& 71.2 & 59.9 & 45.7 & 79.6 & 80.9 & \textbf{52.9} & \textbf{20.0} & 32.9
& \cellcolor{avgblue}73.4 \\

\midrule

\rowcolor{psoftgreen}
\loft$_{r=46}^{\scriptscriptstyle \pgrad}$ & 0.07M & 5.7
& 69.9 & 89.6 & 69.4 & 99.2 & 92.0 & 85.2 & 55.2
& 86.4 & \textbf{95.7} & 82.7 & 75.3
& 68.9 & 59.8 & 46.2 & 79.8 & 79.8 & 50.2 & 19.8 & \textbf{33.2}
& \cellcolor{avgblue}73.3 \\
\rowcolor{psoftgreen}
\loft$_{r=46}^{\scriptscriptstyle \pskew}$ & 0.07M & 5.7
& 71.8 & \textbf{90.1} & 70.2 & \textbf{99.3} & \textbf{92.1} & 86.2 & 55.7
& \textbf{86.7} & 95.0 & 83.7 & \textbf{75.5}
& 71.2 & 59.9 & 46.1 & \textbf{80.1} & 80.1 & 52.5 & \textbf{20.0} & 33.1
& \cellcolor{avgblue}\textbf{73.8} \\
\bottomrule
\end{tabular}
\end{adjustbox}
\end{table}

Together with the GLUE results, the support effect is not specific to one benchmark, although SkewGrad is not uniformly the best construction (Section~\ref{sect:add}), and the VTAB gain over PSOFT is modest and not uniform across tasks.

\subsection{Training Dynamics under the Matched GLUE Protocol}
\label{sec:training_dynamics}

We next ask whether the support rule also changes how quickly the model reaches a good solution (Appendix~\ref{app:short_horizon_diag} isolates the initialization-local signal with probe and early-validation diagnostics). At full-training scale, SkewGrad separates from Principal and Random along the training path, reaching lower loss faster and stronger held-out metrics at the endpoint on both CoLA and STS-B. Random denotes a fixed random orthonormal support (Appendix~\ref{sec:first_order_diag}).

\begin{figure*}[t]
    \centering
    \includegraphics[width=\textwidth]{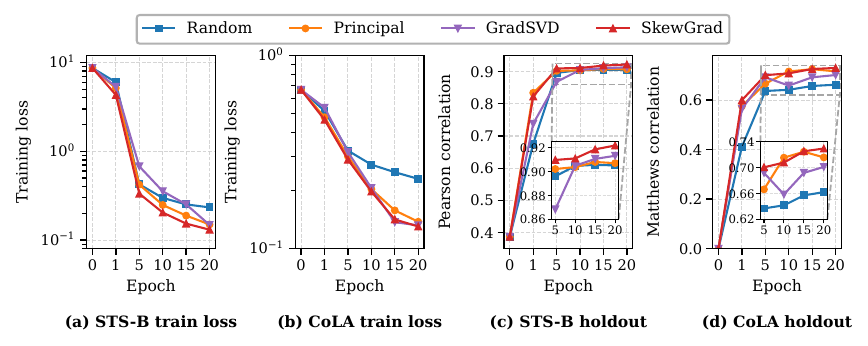}
    \caption{
    Training dynamics under the matched GLUE protocol.
    We compare Random, Principal, GradSVD, and SkewGrad supports at epochs \(0,1,5,10,15,20\).
    Panels (a,b) show training loss on STS-B and CoLA; panels (c,d) show held-out Pearson on STS-B and Matthews on CoLA.
    All settings are matched, including rank, optimizer, schedule, split, and the orthogonal LOFT transform; the only change is the support \(P_r\).
    }
    \label{fig:glue_training_dynamics}
    \vspace{-1em}
\end{figure*}

On CoLA (Figure~\ref{fig:glue_training_dynamics}), the support effect appears early and remains visible throughout training. 
SkewGrad reaches the low-loss regime faster: by epoch 10, its training loss is already below \(0.2\), whereas Principal only enters this regime around epoch 15 and Random remains above it over the reported horizon. 
The held-out trajectory shows the same time-to-performance pattern: SkewGrad crosses \(0.70\) Matthews correlation by epoch 5, while Principal reaches comparable performance later and Random never reaches this level within 20 epochs. 
Although Principal is competitive at some intermediate points, SkewGrad gives the best final Matthews correlation, reaching \(0.7297\) at epoch 20.

On STS-B, Principal and SkewGrad both improve quickly over Random by epoch 1, and SkewGrad is the strongest support from epoch 5 onward, ending at \(0.9220\) Pearson against \(0.9071\) for Principal and \(0.9052\) for Random; it also has the lowest training loss at every measured epoch, so its held-out advantage does not come from fitting the training data less.

Adding GradSVD separates generic gradient information from the skew-specific construction: GradSVD ends above Principal on three of the four endpoints, but SkewGrad has the strongest endpoint on all four curves, and GradSVD is less consistent mid-training (its STS-B loss is above Random at epochs 5 and 10). This trajectory uses one fixed training seed, and the endpoint gap between GradSVD and SkewGrad is small (Appendix~\ref{app:gradsvd_trajectory}).

\subsection{Additional Experiment Results}
\label{sect:add}

\textbf{First-order support diagnostics.}
Appendix~\ref{app:short_horizon_diag} tests Proposition~\ref{prop:signal-strength} with a controlled probe that updates only the LOFT adapter and with an early-validation diagnostic in the standard GLUE pipeline. The probe confirms the predicted ordering at its extremes but not between the principal and GradSVD supports: their signal shares \(\rho(P)\) are \(0.036\) and \(0.049\) on STS-B (\(0.002\) for Random and \(1.000\) for SkewGrad), yet GradSVD reduces the probe loss less than Principal at every measured horizon on both tasks, while SkewGrad gives the largest reduction from the fifth step onward. SkewGrad also gives the best early validation in the first 25 updates, and a few training-only calibration mini-batches suffice to construct it.

\textbf{Support construction cost.}
Gradient-informed supports are constructed once before training, one target matrix at a time, and then fixed. Gradient collection and spectral construction take \(6.45\)s and \(3.63\)s on ViT-B/16 (72 targets, \(10.08\)s in total) and \(29.21\)s and \(7.75\)s on LLaMA2-7B (64 targets, \(36.96\)s). Including construction, SkewGrad LOFT takes \(573.65\)s per VTAB task against \(643.06\)s for PSOFT, \(10.8\%\) less. Construction has its own transient memory peak, dominated by gradient collection (\(+3.71\)GiB on ViT-B/16; \(44.66\)GiB absolute on LLaMA2-7B, above the LOFT and PSOFT training-loop peaks in Table~\ref{tab:math_main}), which we therefore budget separately from the training loop (Appendix~\ref{app:construction_cost}).

\textbf{Boundary case: low-resource language adaptation.}
The preferred gradient-informed construction is task-dependent. On Llama-3.2-3B \citep{meta2024llama32} fine-tuned on five Bactrian-X \citep{li2023bactrianx} languages over three seeds (Appendix~\ref{app:lowres-lang}), the average held-out NLL is \(1.4037\) for Principal, \(1.3972\) for GradSVD and \(1.4112\) for SkewGrad, so SkewGrad trails the principal support and GradSVD is the better construction here. SkewGrad trails Principal on Indonesian, English and Afrikaans, matches it on Tagalog, and beats it on Swahili, the strongest OOD language in the set (\(1.3503\) vs.\ \(1.3631\); GradSVD \(1.3480\)). With per-method tuned learning rates on Swahili, GradSVD reaches \(1.3170\) at the 12M tier against \(1.3300\) for PSOFT and \(1.3362\) for parameter-matched LoRA.

\textbf{Further results.} Under the development-set protocol on MNLI and QQP (Appendix~\ref{app:mnli_qqp_details}), SkewGrad improves over the principal orthogonal support and the free transform on the principal support gives the strongest two-task average. Single-seed GLUE runtimes (Appendix~\ref{app:glue_efficiency_grad}) show principal-support orthogonal LOFT faster than PSOFT on every task and SkewGrad \(6.1\%\) faster on average with calibration included. The full support/transform grids and the data-fraction and rank sweeps (Appendices~\ref{app:support_transform_ablations} and~\ref{app:glue_ablation_skew}) show task-aware supports, especially SkewGrad, giving the stronger matched-rank variants.

\section{Concluding Remarks}

We presented LOFT, a parameter-efficient orthogonal fine-tuning framework that decomposes the update rotation into two explicit choices: the adaptation support \(P_r\) and the in-subspace transform \(T_r\). This view recovers the multiplicative core of representative orthogonal PEFT methods 
and, more importantly, exposes support selection as an important design problem. Our first-order analysis shows that, at initialization, the downstream task signal for orthogonal adaptation is governed by  \(\mathrm{skew}(W_0^\top G)\), rather than  by the principal support of the pretrained weight. This motivates task-aware supports such as \textsc{GradSVD} and \textsc{SkewGrad}. Across language understanding, visual transfer, mathematical reasoning and multilingual adaptation, support choice affects performance under a fixed transform and budget: in each setting at least one gradient-informed support improves on the principal support at matched trainable-parameter and transform settings.
Overall, LOFT suggests that future progress in orthogonal PEFT should focus not only on better orthogonal parameterizations, but also on systematic selection of the support where adaptation occurs.

\textbf{Limitations.}
Our empirical claim is the controlled one that support choice affects performance under a fixed transform and budget. We report effect sizes with their variability; on GLUE and VTAB the margins lie in the range where seed variability matters, so we treat task-level gaps as descriptive. Proposition~\ref{prop:signal-strength} is initialization-local and first-order: it motivates the gradient-informed supports but is not a guarantee about their final ranking after nonlinear training, and it covers only the orthogonal, skew-parameterized transform, so the free transform \(\Tfree\) lies outside it. We use the Cayley transform on skew-symmetric matrices as the default orthogonal parameterization and do not explore other transform classes. The preferred gradient-informed construction is task-dependent: SkewGrad has the highest means in our encoder and visual-transfer studies, but on low-resource multilingual adaptation it trails the principal support and GradSVD is the better construction. The fixed support is task-specific state: counting it, gradient-informed LOFT at \(r=46\) stores more per task than LoRA\(_{r=8}\) (\(3.890\)M vs.\ \(1.327\)M scalars). Gradient-informed supports also require a one-off calibration step, whose time and transient memory we report separately from the training loop (Appendix~\ref{app:construction_cost}) and which slightly increases end-to-end runtime relative to the principal-support variant (Appendix~\ref{app:glue_efficiency_grad}). Finally, our decoder-side reasoning study uses one matched low-budget setting with single-run baselines, and no model larger than LLaMA2-7B; we could not complete a reliable evaluation of a larger instruction-tuned model.

\begin{ack}
The authors received no specific funding or external
computational support for this work and declare no
competing interests.
\end{ack}

\bibliographystyle{plainnat}
\bibliography{references}

\appendix
\raggedbottom
\clearpage

\startcontents[appendices]

\begingroup
\definecolor{appred}{RGB}{170,55,55}

\titlecontents{section}
  [1.8em]
  {\addvspace{0.45em}\bfseries\color{appred}}
  {\contentslabel{1.8em}}
  {}
  {\titlerule*[0.55pc]{.}\contentspage}

\titlecontents{subsection}
  [4.6em]
  {\small\color{appred}}
  {\contentslabel{2.8em}}
  {}
  {\titlerule*[0.55pc]{.}\contentspage}

\titlecontents{subsubsection}
  [5.4em]
  {\small\color{appred}}
  {\contentslabel{3.3em}}
  {}
  {\titlerule*[0.55pc]{.}\contentspage}

\clearpage
\section*{Appendices Contents}
\setcounter{tocdepth}{2}
\printcontents[appendices]{}{1}{}
\endgroup

\clearpage
\AppendixSectionBreakstrue

\section{Related Work}
\label{sec:related_work}

Parameter-efficient fine-tuning (PEFT) methods are commonly grouped into selection-based, prompt-based, adapter-based, and reparameterization-based families \citep{han2024parameter}. Representative examples include BitFit, prefix/prompt tuning, and compact adapter layers \citep{benzaken2022bitfit,li2021prefix,lester2021power,liu2022fewshot,mahabadi2021compacter}. Since LOFT belongs to the reparameterization-based line, we focus on methods that adapt pretrained weights directly and can be merged into the backbone without introducing additional inference-time modules.

\textbf{Low-rank additive adaptation.}
The dominant reparameterization paradigm models downstream change as an additive perturbation constrained to a low-dimensional space, consistent with evidence that neural-network adaptation often has low intrinsic dimension \citep{li2018intrinsic,aghajanyan2021intrinsic}. LoRA established the basic recipe by freezing the pretrained weight and learning a low-rank update that can be merged at inference time \citep{hu2021lora}. Subsequent work refined this view from several directions. PiSSA initializes adaptation from the principal singular components of the pretrained weight \citep{meng2024pissa}, while DoRA separates magnitude and direction so that the adaptation dynamics more closely match those of full fine-tuning \citep{liu2024dora}. Other variants improve rank allocation, quantized training, parameter sharing, compression, or singular-vector structure within the same low-rank paradigm \citep{zhang2023adalora,valipour2022dylora,dettmers2023qlora,kopiczko2023vera,balazy2024loraxs,lingam2024svft}, or the optimization of the low-rank factors themselves, for example through Riemannian preconditioning of the factorized update \citep{mishra2016riemannian,zhang2024riemannian} or by aligning the optimizer states with the low-rank subspace \citep{tastan2025loft}. Among them, QLoRA \citep{dettmers2023qlora} trains low-rank adapters on a quantized frozen backbone, DyLoRA \citep{valipour2022dylora} trains adapters that remain usable across a range of ranks, and QDyLoRA \citep{rajabzadeh2024qdylora} combines quantization with dynamic-rank training; all three change how an additive low-rank update is stored or sized, which is complementary to the choice of support studied here. This literature strongly supports low-dimensional adaptation, but mainly studies additive updates rather than geometry-preserving multiplicative transformations. 
Beyond fine-tuning, low-rank structures have also been used for efficient pretraining \citep{mao2025survey,lialin2024relora,zhao2024galore,han2024sltrain,liang2024memory,glentis2025scalable}.

\textbf{Orthogonal fine-tuning.}
A complementary line approaches PEFT from a multiplicative and geometric perspective by constraining adaptation to remain orthogonal. Orthogonal over-parameterized training and OFT introduced this view, with OFT using block-diagonal orthogonal transforms to make optimization tractable \citep{liu2021orthogonal,qiu2023controlling}. BOFT replaces the block-diagonal construction with butterfly factorization, enabling denser information mixing under a parameter-efficient orthogonal parameterization \citep{liu2024parameterefficient}. qGOFT uses Givens rotations and relaxed orthogonality to improve flexibility \citep{ma2024qgoft}, while OFTv2 revisits the same line through an input-centric implementation with a Cayley--Neumann approximation \citep{qiu2025scalable}. Broader orthogonal PEFT variants also include Householder-based methods such as HOFT and SHOFT \citep{morenoarcas2025hoft}. For LOFT, the key message is that orthogonal PEFT has largely optimized how to parameterize the transform, while leaving the choice of adaptation support comparatively under-theorized.

\textbf{Bridging low-rank and orthogonal adaptation.}
Recent work makes the relation between low-rank and orthogonal PEFT explicit. HRA constructs adapters as products of learnable Householder reflections and interprets the resulting multiplicative update as adaptive low-rank adaptation, with reflection-plane orthogonality controlling the capacity--regularity trade-off \citep{yuan2024bridging}. Spectral Adapter and related singular-vector methods shift attention from parameterization to spectral placement, studying how subspace restriction changes effective adaptation capacity \citep{zhang2024spectral,lingam2024svft}. LoRA-XS is also related in form, since it freezes SVD-derived outer factors and learns a small inner matrix, but it remains an additive low-rank update rather than a multiplicative subspace transform \citep{balazy2024loraxs}. Such subspace--core separations have a long history in Riemannian low-rank optimization, where the polar factorization $W=UBV^\top$ represents a rank-$r$ matrix by two orthonormal bases $U,V$ and an $r\times r$ symmetric positive-definite core $B$ \citep{mishra2014fixed}. LOFT applies the same separation to a multiplicative update, with the support $P_r$ playing the role of the basis and $T_r$ that of the core. PSOFT is the closest direct antecedent to our formulation: it restricts orthogonal adaptation to the principal right-singular subspace of the pretrained weight and adds lightweight relaxation during training \citep{wu2025memory}. This makes PSOFT an important spectral instance of subspace-restricted orthogonal adaptation, but its support is chosen by pretrained-weight energy rather than by downstream loss. LOFT generalizes this regime by treating the support \(P_r\) as an explicit design variable. Our first-order analysis selects this support through the loss-coupled signal \(\mathrm{skew}(W_0^\top G)\), yielding gradient-informed supports that can be compared directly with principal supports at matched trainable-parameter and transform settings.

\textbf{Gradient-informed adaptation.}
Several methods use downstream gradients to build additive or vector-based adapters: LoRA-GA \citep{wang2024loraga} initializes an additive low-rank update from the gradient, GoRA \citep{he2025gora} uses gradients for rank allocation and initialization, and GiVA \citep{gangwar2026giva} constructs gradient-informed bases. LOFT uses the gradient for a different role, selecting the support on which a multiplicative orthogonal transform acts through \(\mathrm{skew}(W_0^\top G)\).
We reproduce LoRA-GA and GiVA under our GLUE protocol in Table~\ref{tab:glue_rerun} as positioning evidence rather than an exhaustive ranking.

\paragraph{Broader impacts.}
LOFT reduces the trainable parameters and training-loop memory needed to adapt a pretrained model, which can lower the cost of adaptation and make it more accessible. The same property also makes it easier to adapt models for harmful uses or on biased data; this risk is shared with other PEFT methods, and we do not identify a risk specific to LOFT beyond it.

\section{Additional Theory and Proofs}

\subsection{Proof of Proposition~\ref{prop:loft_orthogonality}}
\label{app:orthogonality_proof}

\begin{proof}
Let
\[
U := S(P_r,T_r)=I_{d_{\mathrm{in}}}+P_r^\top(T_r-I_r)P_r .
\]
Extend \(P_r\) to a full orthogonal basis
\[
P=
\begin{bmatrix}
P_r\\
P_\perp
\end{bmatrix}
\in O(d_{\mathrm{in}}).
\]
Then
\[
U
=
P^\top
\begin{bmatrix}
T_r & 0\\
0 & I_{d_{\mathrm{in}}-r}
\end{bmatrix}
P .
\]
Since \(T_r\in O(r)\), the block-diagonal matrix is orthogonal. Hence \(U\in O(d_{\mathrm{in}})\).

Now let \(W^+=WU\). Since \(U\) is orthogonal,
\[
W^+(W^+)^\top
=
(WU)(WU)^\top
=
WUU^\top W^\top
=
WW^\top .
\]
This proves exact row-Gram preservation.

Because \(U\) is orthogonal, it is invertible. Therefore,
\[
\operatorname{rank}(W^+)
=
\operatorname{rank}(WU)
=
\operatorname{rank}(W).
\]
Moreover,
\[
(W^+)^\top W^+
=
(WU)^\top(WU)
=
U^\top W^\top WU .
\]
Thus \((W^+)^\top W^+\) is orthogonally similar to \(W^\top W\), so the two matrices have the same eigenvalues. Hence \(W^+\) and \(W\) have the same singular values:
\[
\sigma_i(W^+)=\sigma_i(W)\qquad \forall i .
\]
Since the Frobenius norm and spectral norm are determined by the singular values, we also have
\[
\|W^+\|_F=\|W\|_F,
\qquad
\|W^+\|_2=\|W\|_2 .
\]
This proves all claims in Proposition~\ref{prop:loft_orthogonality}.
\end{proof}

\subsection{Why LOFT is Right-Multiplicative Rather than Double-Sided}
\label{app:right_vs_double}

A natural generalization of orthogonal adaptation is the double-sided update
\[
W^+ = QWR,
\qquad
Q \in O(d_{\mathrm{out}}),\; R \in O(d_{\mathrm{in}}).
\]
Although this form preserves singular values, it does not preserve the same row-neuron relational
structure as right-multiplicative LOFT.

\begin{remark}
For right-multiplicative orthogonal LOFT, $W^+=WU$ with $U\in O(d_{\mathrm{in}})$ implies
\[
W^+(W^+)^\top = WW^\top.
\]
Hence the row Gram matrix is preserved exactly.

In contrast, for a double-sided update $W^+=QWR$,
\[
W^+(W^+)^\top = QWW^\top Q^\top.
\]
Therefore the row Gram matrix is generally rotated rather than preserved entrywise. This is the
reason we restrict LOFT to the right-multiplicative setting.
\end{remark}

\subsection{Proof of Proposition \ref{prop:signal-strength}}
\label{app:signal_strength}

\begin{proof}
Let
\[
F:=\skewsym(W_0^\top G)
=
\frac{W_0^\top G-G^\top W_0}{2}.
\]
For any perturbation \(E\in\mathrm{Skew}(r)\), the local condition on \(Q\) gives
\[
Q(tE)=I_r+tE+o(t).
\]
Therefore,
\[
W_0\!\left(I_d+P_r^\top(Q(tE)-I_r)P_r\right)
=
W_0+tW_0P_r^\top E P_r+o(t).
\]
By the first-order expansion of \(L\) at \(W_0\),
\[
\frac{d}{dt}
L\!\left(
W_0\!\left(I_d+P_r^\top(Q(tE)-I_r)P_r\right)
\right)\Big|_{t=0}
=
\left\langle G,\,W_0P_r^\top E P_r\right\rangle_F .
\]
Using cyclicity of the trace,
\[
\left\langle G,\,W_0P_r^\top E P_r\right\rangle_F
=
\mathrm{tr}\!\left(G^\top W_0P_r^\top E P_r\right)
=
\mathrm{tr}\!\left(P_rG^\top W_0P_r^\top E\right).
\]
Since \(E\in\mathrm{Skew}(r)\), only the skew-symmetric part of
\(P_rW_0^\top G P_r^\top\) contributes to the inner product. Hence
\[
\mathrm{tr}\!\left(P_rG^\top W_0P_r^\top E\right)
=
\left\langle
\skewsym(P_rW_0^\top G P_r^\top),\,E
\right\rangle_F .
\]
Because \(P_rP_r^\top=I_r\), the skew operator commutes with this orthogonal compression:
\[
\skewsym(P_rW_0^\top G P_r^\top)
=
P_r\skewsym(W_0^\top G)P_r^\top
=
P_rFP_r^\top .
\]
Thus
\[
\frac{d}{dt}
L\!\left(
W_0\!\left(I_d+P_r^\top(Q(tE)-I_r)P_r\right)
\right)\Big|_{t=0}
=
\left\langle P_rFP_r^\top,E\right\rangle_F .
\]
Equivalently, under the Frobenius inner product on \(\mathrm{Skew}(r)\),
\[
\nabla_S L\big|_{S=0}=P_rFP_r^\top,
\qquad
\left\|\nabla_S L\big|_{S=0}\right\|_F^2
=
\|P_rFP_r^\top\|_F^2 .
\]

It remains to prove the upper bound. Since \(F\in\mathrm{Skew}(d)\), its nonzero eigenvalues occur in conjugate pairs
\(\{\pm i\mu_k\}\), with \(\mu_1\ge \mu_2\ge\cdots\ge 0\). The singular values of \(F\) are therefore the values \(\mu_k\), each repeated twice. For any row-orthonormal \(P_r\), the compressed matrix \(P_rFP_r^\top\) is also skew-symmetric. Let its nonzero eigenvalues be \(\{\pm i\nu_j\}_{j=1}^{\lfloor r/2\rfloor}\), with \(\nu_1\ge\nu_2\ge\cdots\ge0\). By Cauchy interlacing for the Hermitian compression \(P_r(iF)P_r^\top\), whose eigenvalues are \(\pm\nu_j\) while those of \(iF\) are \(\pm\mu_k\),
\[
\nu_j\le \mu_j
\qquad
\text{for } j=1,\ldots,\lfloor r/2\rfloor .
\]
Therefore, for even \(r\),
\[
\|P_rFP_r^\top\|_F^2
=
2\sum_{j=1}^{r/2}\nu_j^2
\le
2\sum_{j=1}^{r/2}\mu_j^2 .
\]
For odd \(r\), the compression is a skew-symmetric matrix of odd order and so has a zero eigenvalue; its remaining \(\lfloor r/2\rfloor\) pairs obey the same interlacing, which gives the bound with \(r/2\) replaced by \(\lfloor r/2\rfloor\).

Equality is achieved when the row space of \(P_r\) contains an invariant subspace \(V\) of \(F\) associated with the largest \(\lfloor r/2\rfloor\) skew-eigenvalue pairs. For even \(r\) the two coincide, and the compression \(P_rFP_r^\top\) is exactly those largest skew blocks. For odd \(r\) the row space is \(V\oplus\mathrm{span}(v)\) with \(v\perp V\); since \(V\) is \(F\)-invariant, \(v^\top Fu=0\) for every \(u\in V\), and \(v^\top Fv=0\) by skew-symmetry, so the compression is those same blocks together with one zero. In either case its Frobenius norm attains
\[
2\sum_{j=1}^{\lfloor r/2\rfloor}\mu_j^2 .
\]
Conversely, if equality holds then \(\nu_j=\mu_j\) for every \(j\le\lfloor r/2\rfloor\), since each term is bounded by its counterpart. The Rayleigh quotient of \(iF\) restricted to the complexified row space of \(P_r\) then attains \(\mu_1=\lambda_{\max}(iF)\), so its maximizer is an eigenvector of \(iF\) lying in that space; deflating it and repeating \(\lfloor r/2\rfloor\) times yields eigenvectors for \(\mu_1,\ldots,\mu_{\lfloor r/2\rfloor}\) inside the space, and since the space is the complexification of a real subspace it also contains their conjugates, the eigenvectors for \(-\mu_j\). The real subspace they span is an invariant subspace of \(F\) for the largest \(\lfloor r/2\rfloor\) pairs and lies in the row space of \(P_r\). If some \(\mu_j=0\) with \(j\le\lfloor r/2\rfloor\), all nonzero pairs are among those found, their span is the range of \(F\), and the remaining directions are taken from the row space intersected with \(\ker F\).
This proves Proposition~\ref{prop:signal-strength}.
\end{proof}
\subsection{Why PSOFT is Generically Suboptimal}
\label{app:psoft_suboptimal}

The training-signal analysis above yields a precise condition under which the PSOFT choice
$P_r = V_r^\top$ is optimal.

\paragraph{Corollary A.8.}
Let \(V_r\) denote the top-\(r\) right singular vectors of \(W_0\), and set \(P_r=V_r^\top\). The resulting first-order signal strength is
\[
\|V_r^\top F V_r\|_F^2.
\]
This achieves the maximal value
\[
2\sum_{k=1}^{\lfloor r/2\rfloor}\mu_k^2
\]
if and only if \(\mathrm{col}(V_r)\) contains an invariant subspace of \(F\) associated with the largest \(\lfloor r/2\rfloor\) skew-eigenvalue pairs; for even \(r\) that subspace is \(\mathrm{col}(V_r)\) itself.

Equivalently, for even \(r\), writing \(V=[V_r\;V_\perp]\) and
\[
V^\top F V
=
\begin{pmatrix}
F_{rr} & F_{r\perp}\\
-F_{r\perp}^\top & F_{\perp\perp}
\end{pmatrix},
\]
a necessary condition is
\[
F_{r\perp}=0,
\]
and maximality further requires that the skew-eigenvalue pairs contained in \(F_{rr}\) are the largest ones. Generically, even the invariance condition \(F_{r\perp}=0\) does not hold: for fixed \(W_0\neq0\) the map \(G\mapsto F_{r\perp}\) is linear and not identically zero (for \(G=W_0v_1w^\top\), with \(v_1\) the first column of \(V_r\), \(\sigma_1\) its singular value and \(w\) the first column of \(V_\perp\), \(F_{r\perp}=\tfrac{\sigma_1^2}{2}e_1e_1^\top\neq0\)), so the gradients that satisfy it form a proper linear subspace, a set of measure zero. Hence, for even \(r\) (all ranks used in this paper), principal-subspace orthogonal adaptation is not generically optimal for maximizing the first-order training signal.

\section{LOFT Design Space and Recovered Orthogonal PEFT Branches}
\label{app:design_recoveries}

\subsection{Recovering Right-Multiplicative Orthogonal PEFT Branches}
\label{app:recoveries}

We now clarify which existing orthogonal PEFT mechanisms are recovered by the right-multiplicative LOFT formulation
\[
W^{+}
=
W\prod_{\ell=1}^{L} S\!\left(P_{r_\ell}^{(\ell)}, R_{r_\ell}^{(\ell)}\right),
\qquad
S(P_r,R_r):=I + P_r^\top(R_r-I_r)P_r,
\]
where each $P_r \in \mathbb{R}^{r\times d_{\mathrm{in}}}$ satisfies $P_rP_r^\top=I_r$ and each
$R_r \in O(r)$. The statements below are made at the level of the core right-side subspace-rotation
mechanism rather than every implementation detail of a given method.

\paragraph{Full-space OFT.}
Setting $L=1$, $r=d_{\mathrm{in}}$, and $P_r=I_{d_{\mathrm{in}}}$ gives
\[
S(I,R)=R,
\qquad
W^{+}=WR.
\]
Therefore, the original full-space orthogonal fine-tuning update is recovered as the width-$d_{\mathrm{in}}$
special case of LOFT.

\paragraph{Block-diagonal OFT.}
Let $\{1,\dots,d_{\mathrm{in}}\}$ be partitioned into $k$ disjoint coordinate blocks of size $b$, and let
$P_b^{(i)}$ select the coordinates of the $i$-th block. Then
\[
W^{+}
=
W\prod_{i=1}^{k} S\!\left(P_b^{(i)}, R_b^{(i)}\right).
\]
Because the supports are disjoint, the factors commute and reduce to a block-diagonal orthogonal
transform,
\[
\prod_{i=1}^{k} S\!\left(P_b^{(i)}, R_b^{(i)}\right)
=
\operatorname{diag}\!\left(R_b^{(1)},\dots,R_b^{(k)}\right).
\]
Hence block-diagonal OFT is recovered by a single right-multiplicative LOFT layer with disjoint
coordinate supports.

\paragraph{GOFT / Givens-style coordinate rotations.}
Suppose each factor has width $2$ and selects a coordinate pair:
\[
W^{+}
=
W\prod_{\ell=1}^{L} S\!\left(P_2^{(\ell)}, R_2^{(\ell)}\right),
\qquad
R_2^{(\ell)} \in O(2).
\]
Each factor is then a Givens rotation acting on a selected coordinate plane. This recovers the core
coordinate-pair update mechanism underlying GOFT-style orthogonal fine-tuning.

\paragraph{BOFT.}
BOFT is obtained by stacking multiple right-multiplicative LOFT layers whose supports are coordinate
blocks chosen according to the butterfly permutation pattern. Concretely, with block width $b$ and
$m$ butterfly stages,
\[
W^{+}
=
W\prod_{\ell=1}^{m}\prod_{i=1}^{d_{\mathrm{in}}/b}
S\!\left(P_b^{(\ell,i)}, R_b^{(\ell,i)}\right),
\]
where within each stage the supports are disjoint, and across stages the supports are permuted in the
butterfly pattern. The stage count $m$ is a BOFT hyperparameter rather than a quantity fixed by
$d_{\mathrm{in}}$. BOFT is thus represented as a multi-layer structured sparse special case of the same
right-side primitive.

\paragraph{HRA.}
HRA is recovered by width-one factors with learnable subspace bases. Let $u_i \in \mathbb{R}^{d_{\mathrm{in}}}$
be a unit vector and set
\[
P_1^{(i)} = u_i^\top,
\qquad
R_1^{(i)} = -1.
\]
Then
\[
S(u_i^\top,-1)
=
I + u_i(-1-1)u_i^\top
=
I - 2u_i u_i^\top,
\]
which is exactly a Householder reflection. Therefore,
\[
W^{+}
=
W\prod_{i=1}^{r}\left(I-2u_i u_i^\top\right)
\]
recovers HRA as a product of right-multiplicative width-one reflections. In this case, the trainable
degrees of freedom lie in the learnable subspace bases $\{u_i\}$ rather than in $R_1^{(i)}$, so HRA
corresponds to LOFT with a \emph{learnable} width-one support, unlike the supports fixed before training
in our experiments.

\paragraph{PSOFT / principal-subspace OFT.}
Let the singular value decomposition of the pretrained weight be
\[
W = U\Sigma V^\top,
\qquad
V=[V_r\;V_\perp],
\qquad
U=[U_r\;U_\perp],
\]
and choose
\[
P_r = V_r^\top.
\]
Then
\[
W^{+}
=
W\bigl(I + V_r(R_r-I_r)V_r^\top\bigr)
=
WV_\perp V_\perp^\top + WV_r R_r V_r^\top
=
U_\perp\Sigma_\perp V_\perp^\top + U_r\Sigma_r R_r V_r^\top.
\]
Hence the residual right-singular subspace remains unchanged, while the principal right-singular
subspace undergoes an orthogonal transform. This recovers the core principal-subspace update
underlying PSOFT only up to PSOFT's additional scaling and relaxation terms, which lie outside the
orthogonal subspace transform represented here.

\paragraph{Scope of the recovery claims.}
The recoveries above should be interpreted as statements about the underlying right-multiplicative
subspace-rotation mechanism. In particular, LOFT represents the multiplicative core of full-space, block-diagonal,
coordinate-pair, butterfly, Householder, and principal-subspace orthogonal PEFT branches by varying
only the depth, width, and subspace basis of the same primitive update. The HRA correspondence uses a
learnable width-one support, and the PSOFT correspondence holds only up to PSOFT's scaling and
relaxation terms.

\paragraph{Additional support choices.} Supports may also be dynamic: a dynamic random support corresponds to randomized submanifold gradient descent on the orthogonal manifold \citep{han2025efficient}, which enjoys a global convergence guarantee. 
Relatedly, Riemannian coordinate descent methods select tangent-space coordinate directions to construct  efficient updates while maintaining manifold constraints \citep{han2024riemannian}. These methods provide a related perspective on restricting updates to low-dimensional directions. One can adopt a dynamic gradient-informed support by periodically updating the support using gradient information, in a similar spirit to GaLore \citep{zhao2024galore}. Alternatively, the support itself can be learned as a point on the Stiefel or Grassmann manifold, generalizing the learnable width-one supports of HRA, with adaptive Riemannian optimizers \citep{kasai2019riemannian}.

\subsection{Additional Remarks on the Free Variant}
\label{app:free_variant}

The free variant should be interpreted as replacing the orthogonal subspace transform with a fully unconstrained dense linear map
\[
T = S, \qquad S \in \mathbb{R}^{r \times r},
\]
under the same fixed subspace basis $P_r$. Accordingly, free LOFT is not a mild relaxation of orthogonality, nor a diagonal $\alpha,\beta$-style scaling extension. Instead, it is an identity-initialized dense subspace transform with no internal geometric constraint.

This distinction is useful when interpreting experiments: orthogonal and free LOFT differ only in the transformation class within the same selected subspace. Their comparison therefore isolates the effect of enforcing orthogonality once the subspace support is held fixed.

\section{Construction of Gradient-Based Supports}
\label{app:grad_support}

\subsection{Gradient-SVD support (GradSVD)}
\label{app:gradsvd_impl}
For \textbf{LOFT-GradSVD}, we construct the fixed support before inserting the fixed-\(P_r\) LOFT module.
We first identify all target linear layers.
We then freeze the model except for the target-layer weights, enable gradients only on those weights, and run forward and backward passes on 4 calibration mini-batches.
The weight gradients are accumulated across these calibration batches, producing a gradient matrix \(G_\ell\) for each target layer \(\ell\).
We then perform singular value decomposition on \(G_\ell\) and use its top-\(r\) right singular subspace as the fixed support basis \(P_{r,\ell}\).

\subsection{Skew-invariant gradient support (SkewGrad)}
\label{app:skewgrad_impl}
For each target matrix independently, we accumulate its calibration gradient \(G\) at the pretrained weight \(W_0\) over the same calibration mini-batches as for GradSVD and construct \(\pskew\) as follows:
\begin{verbatim}
Input: pretrained weight W0, accumulated calibration gradient G, rank r
W, G = float32(W0), float32(G)
S = 0.5 * (W.T @ G - G.T @ W)      # explicit FP32 tensor
C = S.T @ S                        # real symmetric PSD
eigvals, eigvecs = torch.linalg.eigh(C)
P_skew = eigvecs[:, argsort(eigvals, descending=True)[:r]].T
return P_skew cast to the model weight dtype
\end{verbatim}
Both \(S\) and \(C\) are materialized explicitly in FP32, once per target. We use exact dense \texttt{torch.linalg.eigh} with its default \texttt{UPLO="L"}, so no Schur decomposition, complex eigenvectors or approximate eigensolver is involved; \(\pskew\) is cast to the model dtype and fixed thereafter. Because \(S\) is real skew-symmetric, the eigenvalues of \(C=S^\top S\) are \(\mu_k^2\), each carried by the two-dimensional real invariant block of one skew pair, so for even \(r\) the top-\(r\) eigenspace collects the leading complete pairs, and sign and within-pair rotation ambiguities leave the projector \(\pskew^\top\pskew\) unchanged. For odd \(r\), the bound in Proposition~\ref{prop:signal-strength} counts \(\lfloor r/2\rfloor\) complete skew pairs; the implementation keeps an \(r\)-dimensional support, the unmatched direction adding no pair. Every reported \textsc{SkewGrad} rank is even, so no result splits a pair.

\subsection{Orthogonal and Free LOFT}
\label{sec:variants}

\begin{figure}[H]
\centering
\includegraphics[width=0.62\textwidth]{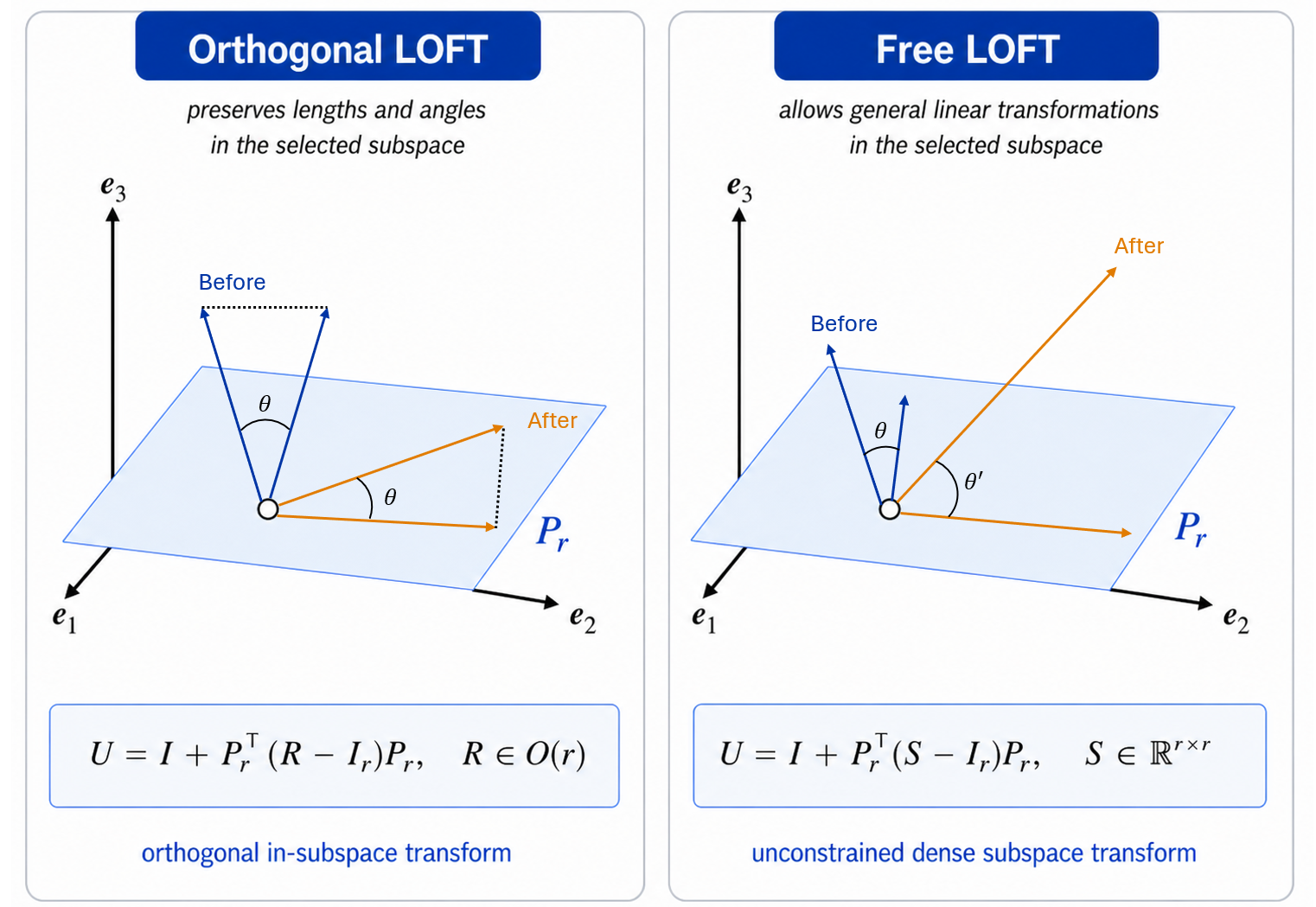}
\vspace{-0.4em}
\caption{Orthogonal and free LOFT under fixed support \(P_r\).}
\label{fig:ortho_vs_free}
\vspace{-0.6em}
\end{figure}

Once the support  basis \(P_r\) is fixed, LOFT enables a controlled comparison between two in-subspace transformation classes. In \emph{orthogonal LOFT}, we set \(T_r=R_r\in O(r)\), so the ambient transform is orthogonal and Proposition~\ref{prop:loft_orthogonality} applies directly. In \emph{free LOFT}, we use \(T_r=S_r\in\mathbb{R}^{r\times r}\), initialized at \(I_r\) and trained without an internal geometric constraint. Under a fixed support, the difference between orthogonal and free LOFT is entirely attributable to the transformation class acting inside the same selected subspace. Figure~\ref{fig:ortho_vs_free} visualizes this matched-support comparison: both methods share \(P_r\), but differ in whether the in-subspace transform is constrained to be orthogonal or allowed to be fully free. This distinction is important for the experimental analysis because, once the support is held fixed, comparisons between the two variants isolate the role of the transformation class rather than conflating it with support selection. In later experiments, we therefore use free LOFT only as a controlled ablation of \(T_r\), while keeping the main tables focused on the default orthogonal transform.

\section{Additional Experimental Details}
\label{app:exp_details}

\subsection{GLUE on DeBERTaV3-base}
\label{app:glue_details}

\paragraph{Datasets and metrics.}
We evaluate encoder-only adaptation on six GLUE tasks: CoLA, STS-B, RTE, MRPC, SST-2, and QNLI. For CoLA we report Matthews correlation, for STS-B we report Pearson correlation, and for RTE, MRPC, SST-2, and QNLI we report accuracy.

\paragraph{Hold-out validation protocol.}
Following PSOFT, we adopt the same hold-out validation protocol for these six tasks. Specifically, the original validation set of each task is split into new validation and test splits with a fixed random seed. Model selection is performed on the new validation split, and the final reported result is the held-out test score of the checkpoint with the best validation performance. All LOFT results and our same-codebase baseline reproductions are averaged over five random seeds; rows marked $^{\dagger}$ in Table~\ref{tab:glue_main} are published values of \citet{wu2025memory}. Peak memory is the training-loop peak measured on a single NVIDIA Tesla V100-SXM2 GPU.

\paragraph{Learning-rate selection.}
Task-wise learning rates follow the PSOFT protocol and its grid and are selected on the new validation split, the split also used for model selection; where an arm's optimum fell between two grid points we refined once between them, with the same procedure for every arm. Every value off the shared grid in Table~\ref{tab:glue_method_lrs} (\(9.5\mathrm{E}{-5}\), \(4.5\mathrm{E}{-4}\) and \(7.5\mathrm{E}{-4}\)) falls on a principal-support row, so the additional search went to the baseline support rather than to the proposed ones.

\paragraph{Model and training protocol.}
We use DeBERTaV3-base as the encoder backbone. Unless otherwise stated, we use AdamW, warmup ratio \(0.1\), a linear learning-rate schedule, classifier-head learning rate \(5\mathrm{E}{-4}\), and batch size \(32\). Task-specific maximum sequence lengths and training epochs follow the PSOFT protocol.

\begin{center}
\captionsetup{type=table}
\captionof{table}{Shared hyperparameter settings for fine-tuning DeBERTaV3-base on the six-task GLUE setup.}
\label{tab:glue_shared_setup}
\small
\begin{tabular}{lcccccc}
\toprule
Hyperparameter & CoLA & STS-B & MRPC & RTE & SST-2 & QNLI \\
\midrule
Optimizer & \multicolumn{6}{c}{AdamW} \\
Warmup ratio & \multicolumn{6}{c}{0.1} \\
LR schedule & \multicolumn{6}{c}{Linear} \\
Learning rate (head) & \multicolumn{6}{c}{\(5\mathrm{E}{-4}\)} \\
Batch size & \multicolumn{6}{c}{32} \\
\midrule
Max seq. len. & 64 & 128 & 256 & 256 & 128 & 256 \\
\#Epochs & 20 & 20 & 30 & 30 & 10 & 5 \\
\bottomrule
\end{tabular}
\end{center}

\paragraph{LOFT variants.}
We evaluate three support families. \(\loft^{\pprin}\) uses the same principal right-singular support as PSOFT. \(\loft^{\pgrad}\) uses a gradient-derived support obtained by collecting gradients on a short calibration set and taking the top-\(r\) right singular subspace of the resulting gradient matrix. \(\loft^{\pskew}\) uses the theory-guided support induced by \(F=\mathrm{skew}(W_0^\top G)=(W_0^\top G-G^\top W_0)/2\). For each support, we compare the default orthogonal transform \(T_r=\Torth\) with a free variant \(T_r=\Tfree\), where ``free'' denotes a fully unconstrained dense in-subspace transform initialized at the identity rather than a PSOFT-style easing factor.

\begin{center}
\captionsetup{type=table}
\captionof{table}{Task-specific learning rates for PSOFT and LOFT variants on the six-task GLUE setup.}
\label{tab:glue_method_lrs}
\small
\begin{adjustbox}{max width=\linewidth}
\renewcommand{\arraystretch}{1.25}
\begin{tabular}{lcccccc}
\toprule
Method & CoLA & STS-B & RTE & MRPC & SST-2 & QNLI \\
\midrule
PSOFT$_{r=46}$ & \(6\mathrm{E}{-4}\) & \(4\mathrm{E}{-4}\) & \(4\mathrm{E}{-4}\) & \(4\mathrm{E}{-4}\) & \(2\mathrm{E}{-4}\) & \(4\mathrm{E}{-4}\) \\
\loft$_{r=46}^{\scriptscriptstyle \pprin}$ & \(6\mathrm{E}{-4}\) & \(7\mathrm{E}{-4}\) & \(9\mathrm{E}{-4}\) & \(5\mathrm{E}{-4}\) & \(9.5\mathrm{E}{-5}\) & \(4.5\mathrm{E}{-4}\) \\
\loft$_{r=46}^{\scriptscriptstyle \pprin,\,\Tfree}$ & \(1\mathrm{E}{-3}\) & \(6\mathrm{E}{-4}\) & \(7.5\mathrm{E}{-4}\) & \(7\mathrm{E}{-4}\) & \(6\mathrm{E}{-4}\) & \(3\mathrm{E}{-4}\) \\
\loft$_{r=46}^{\scriptscriptstyle \pgrad}$ & \(5\mathrm{E}{-4}\) & \(4\mathrm{E}{-4}\) & \(1\mathrm{E}{-4}\) & \(2\mathrm{E}{-4}\) & \(4\mathrm{E}{-4}\) & \(1\mathrm{E}{-4}\) \\
\loft$_{r=46}^{\scriptscriptstyle \pgrad,\,\Tfree}$ & \(5\mathrm{E}{-4}\) & \(6\mathrm{E}{-4}\) & \(1\mathrm{E}{-4}\) & \(5\mathrm{E}{-4}\) & \(6\mathrm{E}{-4}\) & \(4\mathrm{E}{-4}\) \\
\loft$_{r=46}^{\scriptscriptstyle \pskew}$ & \(5\mathrm{E}{-4}\) & \(1\mathrm{E}{-4}\) & \(5\mathrm{E}{-4}\) & \(4\mathrm{E}{-4}\) & \(2\mathrm{E}{-4}\) & \(1\mathrm{E}{-4}\) \\
\loft$_{r=46}^{\scriptscriptstyle \pskew,\,\Tfree}$ & \(5\mathrm{E}{-4}\) & \(5\mathrm{E}{-4}\) & \(1\mathrm{E}{-4}\) & \(4\mathrm{E}{-4}\) & \(2\mathrm{E}{-4}\) & \(5\mathrm{E}{-4}\) \\
\bottomrule
\end{tabular}
\end{adjustbox}
\end{center}

\paragraph{Same-codebase baselines.}
Table~\ref{tab:glue_rerun} gives the task-wise results of the baselines we reproduced in our codebase under the protocol of Table~\ref{tab:glue_main}. Our PSOFT reproduction obtains \(88.03\), matching the published \(88.04\).
\begin{center}
\captionsetup{type=table}
\captionof{table}{Closest baselines reproduced in our codebase on the six-task GLUE hold-out protocol: one machine, one memory instrument (\texttt{max\_memory\_allocated}), five seeds per task, mean {\scriptsize$\pm$} sample standard deviation. Trainable parameters exclude the fixed support (Table~\ref{tab:per_task_storage}). Standard deviations of the other LOFT variants are in Table~\ref{tab:app_glue_loft_stds}.}
\label{tab:glue_rerun}
\small
\begin{adjustbox}{max width=\linewidth}
\begin{tabular}{lccccccccc}
\toprule
Method & Trainable Params & Mem.\,(GB) & CoLA & STS-B & RTE & MRPC & SST-2 & QNLI & Avg. \\
\midrule
PSOFT$_{r=46}$ & 0.081M & 4.10 & 71.57{\scriptsize$\pm$1.94} & 91.45{\scriptsize$\pm$0.34} & 85.88{\scriptsize$\pm$3.86} & 90.50{\scriptsize$\pm$1.24} & 95.24{\scriptsize$\pm$0.69} & 93.57{\scriptsize$\pm$1.29} & 88.03 \\
LoRA$_{r=8}$ & 1.327M & 4.49 & 72.04{\scriptsize$\pm$3.57} & 91.38{\scriptsize$\pm$0.55} & 84.41{\scriptsize$\pm$1.81} & 90.69{\scriptsize$\pm$0.49} & 94.95{\scriptsize$\pm$1.59} & 93.96{\scriptsize$\pm$0.60} & 87.91 \\
LoRA-GA$_{r=8}$ & 1.327M & 4.49 & 65.62{\scriptsize$\pm$3.56} & 91.25{\scriptsize$\pm$0.72} & 84.17{\scriptsize$\pm$3.30} & 90.85{\scriptsize$\pm$2.42} & 94.65{\scriptsize$\pm$0.70} & 93.79{\scriptsize$\pm$0.63} & 86.72 \\
GiVA & 0.090M & 4.65 & 70.32{\scriptsize$\pm$0.28} & 91.60{\scriptsize$\pm$0.69} & 84.65{\scriptsize$\pm$1.50} & 91.34{\scriptsize$\pm$1.41} & 95.11{\scriptsize$\pm$0.57} & 93.25{\scriptsize$\pm$0.72} & 87.71 \\
\midrule
\loft$_{r=46}^{\scriptscriptstyle \pskew}$ & 0.075M & 4.05 & 72.03{\scriptsize$\pm$1.85} & 91.75{\scriptsize$\pm$0.38} & 86.83{\scriptsize$\pm$3.43} & 92.36{\scriptsize$\pm$1.50} & 95.42{\scriptsize$\pm$0.38} & 93.97{\scriptsize$\pm$0.46} & 88.73 \\
\bottomrule
\end{tabular}
\end{adjustbox}
\end{center}

\paragraph{Per-task storage.}
Trainable Params counts trainable parameters only, following the convention of the compared PEFT literature. For gradient-informed LOFT the fixed support is also task-specific and must be stored, so Table~\ref{tab:per_task_storage} reports both for the ranks used on GLUE. Counted this way, gradient-informed LOFT stores more per task than LoRA$_{r=8}$ (\(3.890\)M vs.\ \(1.327\)M), and about half of LoRA's \(7.631\)M at matched rank \(r=46\). Against same-rank PSOFT, LOFT uses \(49.6\%\) fewer materialized scalars and \(8.2\%\) fewer trainable parameters (\(74{,}520\) vs.\ \(81{,}144\)), because it materializes one support basis per target where PSOFT materializes both, and PSOFT carries \(2r\) extra scaling parameters per matrix. Only the GradSVD/SkewGrad support is independently task-specific: Random supports regenerate from a seed, and Principal, PSOFT and LoRA-XS factors derive from the shared \(W_0\). The table reports scalar counts, so byte storage depends on the stored dtype.
\begin{center}
\captionsetup{type=table}
\captionof{table}{Per-task scalar counts on GLUE over the 72 adapted matrices (backbone and classifier head excluded), at the ranks used in Table~\ref{tab:glue_main}.}
\label{tab:per_task_storage}
\small
\begin{tabular}{lcccc}
\toprule
Quantity & \loft$_{r=46}$ & PSOFT$_{r=46}$ & LoRA$_{r=8}$ & DoRA$_{r=8}$ \\
\midrule
Trainable & 0.075M & 0.081M & 1.327M & 1.410M \\
Fixed support & 3.815M & 7.631M & 0 & 0 \\
Total & 3.890M & 7.712M & 1.327M & 1.410M \\
\bottomrule
\end{tabular}
\end{center}

\subsection{Additional Experiments on MNLI and QQP}
\label{app:mnli_qqp_details}

MNLI-matched and QQP are reported as a secondary experiment under the original GLUE development-set protocol, rather than as part of the primary six-task hold-out GLUE comparison. This distinction is important. The main GLUE table follows the PSOFT-style protocol on CoLA, STS-B, RTE, MRPC, SST-2, and QNLI: the original validation split is partitioned into a new validation split for model selection and a held-out test split for final reporting. By contrast, the MNLI/QQP experiment reports performance directly on the standard GLUE development sets, matching the reporting convention used by the published DeBERTaV3-base baselines.

We use this setting because PSOFT does not report MNLI or QQP under its hold-out protocol, and reproducing a comparably controlled large-task hold-out benchmark would substantially increase computational cost. We therefore compare our finalized MNLI/QQP runs against the published DeBERTaV3-base development-set baselines compiled in the comparison of \citet{wu2025memory}, which covers LoRA, PiSSA, DoRA, LoRA-XS, BOFT and the OFT variants \citep{hu2021lora,meng2024pissa,liu2024dora,balazy2024loraxs,liu2024parameterefficient,qiu2025scalable,ma2024qgoft}; those values aggregate results from several prior sources and serve as secondary-source context rather than as a matched benchmark. This experiment is consequently less controlled than the six-task hold-out benchmark, but it is useful for checking whether LOFT remains competitive on larger sentence-pair tasks under the widely used GLUE dev-set reporting protocol.

\paragraph{Protocol.}
Our runs use \texttt{microsoft/deberta-v3-base} with rank \(r=46\). We train with batch size \(32\), evaluate with batch size \(64\), use maximum sequence length \(128\), a linear learning-rate schedule, warmup ratio \(0.1\), and \(10\) training epochs. Evaluation is performed once per epoch, model selection uses development-set accuracy, and the best development checkpoint is loaded at the end. For MNLI, we report matched development accuracy only. Although the wrapper can evaluate both matched and mismatched splits, the finalized table uses MNLI-matched as the main MNLI metric and does not report mismatched accuracy.

\begin{center}
\captionsetup{type=table}
\captionof{table}{Shared hyperparameter settings for our MNLI-matched and QQP runs.}
\label{tab:mnli_qqp_shared_setup}
\small
\begin{tabular}{lcc}
\toprule
Hyperparameter & MNLI-matched & QQP \\
\midrule
Backbone & \multicolumn{2}{c}{DeBERTaV3-base} \\
Optimizer & \multicolumn{2}{c}{AdamW} \\
Warmup ratio & \multicolumn{2}{c}{0.1} \\
LR schedule & \multicolumn{2}{c}{Linear} \\
Train batch size & \multicolumn{2}{c}{32} \\
Eval batch size & \multicolumn{2}{c}{64} \\
Max seq. len. & 128 & 128 \\
\#Epochs & 10 & 10 \\
Model selection & \multicolumn{2}{c}{Best validation accuracy} \\
Reported split & validation-matched & validation \\
\bottomrule
\end{tabular}
\end{center}

\par\noindent\begin{minipage}{\linewidth}
\begin{center}
\captionsetup{type=table}
\captionof{table}{
Results on MNLI-matched and QQP under the original GLUE development-set protocol. Baseline values are reproduced from the DeBERTaV3-base comparison reported by \citet{wu2025memory}; our LOFT and PSOFT runs use train batch size \(32\), evaluation batch size \(64\), maximum sequence length \(128\), and \(10\) epochs. Results are averaged over five seeds where rerun by us. Bold marks the best PEFT result (FFT excluded).
}
\label{tab:mnli_qqp_followup}
\small
\setlength{\tabcolsep}{4.2pt}
\renewcommand{\arraystretch}{1.04}
\begin{adjustbox}{max width=\linewidth}
\begin{tabular}{lccccc}
\toprule
Method & Trainable Params & Mem. & MNLI-m & QQP & \cellcolor{avgblue}Avg. \\
\midrule
FFT & 184M & 5.9 & 89.90 & 92.40 & \cellcolor{avgblue}91.15 \\
GOFTv2 & 0.08M & 18.5 & 90.10 & 90.85 & \cellcolor{avgblue}90.48 \\
qGOFTv2 & 0.33M & 18.5 & 90.17 & 91.34 & \cellcolor{avgblue}90.76 \\
BOFT$_{m=2}^{b=8}$ & 1.41M & 6.3 & 90.25 & 92.10 & \cellcolor{avgblue}91.18 \\
OFTv2$_{b=32}$ & 1.29M & 4.5 & 90.33 & 92.10 & \cellcolor{avgblue}91.22 \\
LoRA$_{r=8}$ & 1.33M & 4.5 & \textbf{90.65} & 91.99 & \cellcolor{avgblue}91.32 \\
PiSSA$_{r=8}$ & 1.33M & 4.5 & 90.37 & 92.33 & \cellcolor{avgblue}91.35 \\
DoRA$_{r=8}$ & 1.41M & 5.8 & 90.29 & 92.10 & \cellcolor{avgblue}91.20 \\
LoRA-XS$_{r=136}$ & 1.33M & 4.2 & \multicolumn{2}{c}{OOM} & \cellcolor{avgblue}N/A \\
PSOFT$_{r=46}$ & 0.08M & 4.1 & 89.83 & 91.00 & \cellcolor{avgblue}90.42 \\
\midrule
\rowcolor{psoftgreen}
\loft$_{r=46}^{\scriptscriptstyle \pprin}$ 
& 0.07M & 4.1 & 90.12 & 91.05 & \cellcolor{avgblue}90.58 \\
\rowcolor{psoftgreen}
\loft$_{r=46}^{\scriptscriptstyle \pprin,\,\Tfree}$ 
& 0.15M & 4.1 & 90.42 & \textbf{92.34} & \cellcolor{avgblue}\textbf{91.38} \\
\rowcolor{psoftgreen}
\loft$_{r=46}^{\scriptscriptstyle \pskew}$ 
& 0.07M & 4.1 & 90.35 & 92.13 & \cellcolor{avgblue}91.24 \\
\bottomrule
\end{tabular}
\end{adjustbox}
\end{center}
\end{minipage}\par\medskip

\paragraph{Learning rates.}
For our MNLI/QQP runs, we swept learning rates in \(\{1\mathrm{E}{-4},\allowbreak 2\mathrm{E}{-4},\allowbreak 3\mathrm{E}{-4},\allowbreak 5\mathrm{E}{-4},\allowbreak 7\mathrm{E}{-4},\allowbreak 1\mathrm{E}{-3}\}\), depending on the variant. Table~\ref{tab:mnli_qqp_lrs} lists the selected learning rates.

\begin{center}
\captionsetup{type=table}
\captionof{table}{Selected learning rates for our MNLI-matched and QQP runs.}
\label{tab:mnli_qqp_lrs}
\small
\renewcommand{\arraystretch}{1.25}
\begin{tabular}{lcc}
\toprule
Method & MNLI-matched & QQP \\
\midrule
PSOFT$_{r=46}$ & \(5\mathrm{E}{-4}\) & \(5\mathrm{E}{-4}\) \\
\loft$_{r=46}^{\scriptscriptstyle \pprin}$ & \(5\mathrm{E}{-4}\) & \(5\mathrm{E}{-4}\) \\
\loft$_{r=46}^{\scriptscriptstyle \pprin,\,\Tfree}$ & \(5\mathrm{E}{-4}\) & \(5\mathrm{E}{-4}\) \\
\loft$_{r=46}^{\scriptscriptstyle \pskew}$ & \(2\mathrm{E}{-4}\) & \(3\mathrm{E}{-4}\) \\
\bottomrule
\end{tabular}
\end{center}

\paragraph{Interpretation.}
This experiment should be read as a large-task sanity check rather than a fully controlled benchmark. The primary GLUE evidence in this paper remains the six-task hold-out comparison in Table~\ref{tab:glue_main}, where all reported tasks follow the same PSOFT-style validation/test split protocol. MNLI and QQP are different: published DeBERTaV3-base baselines for these two tasks are usually reported on the original GLUE development sets, and the corresponding papers do not use a single unified hyperparameter protocol. In particular, prior results differ in maximum sequence length, epoch count, dropout, batch size, and sometimes reuse numbers from earlier sources. Therefore, the MNLI/QQP table should not be interpreted as a strictly matched replacement for the main GLUE evaluation.

Despite this limitation, the experiment is still informative because it tests whether the same LOFT design choices remain viable on larger sentence-pair datasets. Under the GLUE development-set reporting convention, \(\loft_{r=46}^{\pprin}\) matches and slightly exceeds PSOFT (\(90.58\) vs.\ \(90.42\)), which adds scaling and relaxation terms to the same principal-subspace core, confirming that the principal-support LOFT implementation preserves the closest principal-subspace baseline behavior beyond the six-task hold-out setting. The free-transform variant \(\loft_{r=46}^{\pprin,\Tfree}\) achieves the strongest two-task average, showing that additional in-subspace flexibility can be helpful on large sentence-pair tasks. 
The SkewGrad variant \(\loft_{r=46}^{\pskew}\) improves over principal-support LOFT on both tasks, consistent with the six-task hold-out comparison, the cleaner setting for isolating support-selection effects.

It is also worth noting that our MNLI/QQP runs use maximum sequence length \(128\). Some published DeBERTaV3-base GLUE baselines use longer task-specific maximum sequence lengths or otherwise different training recipes, so this experiment is conservative with respect to input length and compute. The main takeaway is therefore not that the MNLI/QQP table establishes a fully controlled ranking across all methods, but that LOFT variants remain competitive under a widely used development-set protocol while using a relatively lightweight sequence-length configuration. We use these results as supportive evidence for scalability to larger GLUE tasks, while reserving Table~\ref{tab:glue_main} as the primary controlled GLUE conclusion.

\subsection{VTAB-1K on ViT-B/16}
\label{app:vtab_details}

\paragraph{Datasets and evaluation.}
We evaluate visual transfer on VTAB-1K, which consists of 19 tasks grouped into natural, specialized, and structured categories. Following the standard VTAB-1K protocol and PSOFT, each task uses 800 labeled training examples for adaptation and 200 validation examples for hyperparameter selection. Performance is reported as top-1 classification accuracy on the original test set. All LOFT numbers are averaged over five random seeds; baseline rows marked $^{\dagger}$ in Table~\ref{tab:vtab_main} are published values of \citet{wu2025memory}. Peak memory is the training-loop peak measured on a single NVIDIA Tesla V100-SXM2 GPU.

\paragraph{Model and training protocol.}
We use ViT-B/16 (\texttt{google/vit-base-patch16-224-in21k}) as the image backbone. Unless otherwise stated, we use AdamW, warmup ratio \(0.1\), cosine learning-rate schedule, classifier-head learning rate \(5\mathrm{E}{-3}\), batch size \(64\), weight decay \(1\mathrm{E}{-3}\), dropout \(1\mathrm{E}{-1}\), and \(50\) epochs.

\begin{center}
\captionsetup{type=table}
\captionof{table}{Shared hyperparameter settings for fine-tuning ViT-B/16 on VTAB-1K.}
\label{tab:vtab_shared_setup}
\small
\begin{tabular}{lc}
\toprule
Hyperparameter & ViT-B/16 \\
\midrule
Optimizer & AdamW \\
Warmup ratio & 0.1 \\
LR schedule & Cosine \\
Learning rate (head) & \(5\mathrm{E}{-3}\) \\
Batch size & 64 \\
Weight decay & \(1\mathrm{E}{-3}\) \\
Dropout & \(1\mathrm{E}{-1}\) \\
\#Epochs & 50 \\
\bottomrule
\end{tabular}
\end{center}

\paragraph{LOFT variants on VTAB.}
We evaluate principal-support LOFT with orthogonal and free transforms at ranks \(r=42\) and \(r=46\). We additionally evaluate two orthogonal gradient-based variants: \(\loft^{\pgrad}\) and \(\loft^{\pskew}\), both at rank \(r=46\). Table~\ref{tab:vtab_loft_lrs} reports the selected task-specific learning rates.

\begin{center}
\captionsetup{type=table}
\captionof{table}{
Task-specific learning rates for LOFT variants on VTAB-1K. Task order follows the 19 VTAB tasks used in the main table.
}
\label{tab:vtab_loft_lrs}
\small
\renewcommand{\arraystretch}{1.15}
\setlength{\tabcolsep}{3.4pt}
\begin{tabular}{lcccccc}
\toprule
Task & \loft$_{r=42}^{\scriptscriptstyle \pprin}$ & \loft$_{r=42}^{\scriptscriptstyle \pprin,\,\Tfree}$ & \loft$_{r=46}^{\scriptscriptstyle \pprin}$ & \loft$_{r=46}^{\scriptscriptstyle \pprin,\,\Tfree}$ & \loft$_{r=46}^{\scriptscriptstyle \pgrad}$ & \loft$_{r=46}^{\scriptscriptstyle \pskew}$ \\
\midrule
Cifar100 & \(5\mathrm{E}{-4}\) & \(1.5\mathrm{E}{-3}\) & \(5\mathrm{E}{-4}\) & \(1.5\mathrm{E}{-3}\) & \(2\mathrm{E}{-4}\) & \(1.5\mathrm{E}{-4}\) \\
Caltech101 & \(3\mathrm{E}{-3}\) & \(3.5\mathrm{E}{-3}\) & \(4\mathrm{E}{-4}\) & \(3.5\mathrm{E}{-3}\) & \(9\mathrm{E}{-4}\) & \(1.3\mathrm{E}{-3}\) \\
DTD102 & \(1\mathrm{E}{-3}\) & \(1\mathrm{E}{-3}\) & \(2\mathrm{E}{-4}\) & \(5\mathrm{E}{-4}\) & \(3.5\mathrm{E}{-4}\) & \(2\mathrm{E}{-4}\) \\
Flower102 & \(2\mathrm{E}{-3}\) & \(3.5\mathrm{E}{-3}\) & \(1\mathrm{E}{-4}\) & \(3.5\mathrm{E}{-3}\) & \(3.5\mathrm{E}{-4}\) & \(3\mathrm{E}{-4}\) \\
Pets & \(5\mathrm{E}{-4}\) & \(5\mathrm{E}{-4}\) & \(5\mathrm{E}{-4}\) & \(5\mathrm{E}{-4}\) & \(2.8\mathrm{E}{-4}\) & \(4\mathrm{E}{-5}\) \\
SVHN & \(3.5\mathrm{E}{-3}\) & \(5\mathrm{E}{-3}\) & \(3.5\mathrm{E}{-3}\) & \(5\mathrm{E}{-3}\) & \(9\mathrm{E}{-4}\) & \(9.5\mathrm{E}{-4}\) \\
Sun397 & \(1.5\mathrm{E}{-3}\) & \(1.5\mathrm{E}{-3}\) & \(1.5\mathrm{E}{-3}\) & \(1\mathrm{E}{-3}\) & \(1.3\mathrm{E}{-3}\) & \(2.9\mathrm{E}{-4}\) \\
\midrule
Camelyon & \(5\mathrm{E}{-4}\) & \(1\mathrm{E}{-3}\) & \(1\mathrm{E}{-3}\) & \(1\mathrm{E}{-3}\) & \(7\mathrm{E}{-4}\) & \(8\mathrm{E}{-4}\) \\
EuroSAT & \(1.5\mathrm{E}{-3}\) & \(3.5\mathrm{E}{-3}\) & \(2\mathrm{E}{-3}\) & \(3.5\mathrm{E}{-3}\) & \(1.78\mathrm{E}{-3}\) & \(8\mathrm{E}{-4}\) \\
Resisc45 & \(2\mathrm{E}{-3}\) & \(3.5\mathrm{E}{-3}\) & \(2\mathrm{E}{-3}\) & \(3.5\mathrm{E}{-3}\) & \(9\mathrm{E}{-4}\) & \(9\mathrm{E}{-4}\) \\
Retinopathy & \(5\mathrm{E}{-4}\) & \(1.5\mathrm{E}{-3}\) & \(3\mathrm{E}{-3}\) & \(1\mathrm{E}{-3}\) & \(1.6\mathrm{E}{-4}\) & \(7\mathrm{E}{-4}\) \\
\midrule
Clevr-Count & \(1.5\mathrm{E}{-3}\) & \(1.5\mathrm{E}{-3}\) & \(1\mathrm{E}{-3}\) & \(1\mathrm{E}{-3}\) & \(7\mathrm{E}{-4}\) & \(8\mathrm{E}{-4}\) \\
Clevr-Dist & \(1\mathrm{E}{-3}\) & \(5\mathrm{E}{-3}\) & \(3.5\mathrm{E}{-3}\) & \(5\mathrm{E}{-3}\) & \(1.36\mathrm{E}{-3}\) & \(8\mathrm{E}{-4}\) \\
DMLab & \(3.5\mathrm{E}{-3}\) & \(5\mathrm{E}{-3}\) & \(3.5\mathrm{E}{-3}\) & \(5\mathrm{E}{-3}\) & \(1.73\mathrm{E}{-3}\) & \(9\mathrm{E}{-4}\) \\
KITTI-Dist & \(2\mathrm{E}{-3}\) & \(3.5\mathrm{E}{-3}\) & \(2\mathrm{E}{-3}\) & \(3.5\mathrm{E}{-3}\) & \(1.6\mathrm{E}{-3}\) & \(8\mathrm{E}{-4}\) \\
dSpr-Loc & \(3\mathrm{E}{-3}\) & \(3.5\mathrm{E}{-3}\) & \(3\mathrm{E}{-3}\) & \(3.5\mathrm{E}{-3}\) & \(4\mathrm{E}{-4}\) & \(9\mathrm{E}{-4}\) \\
dSpr-Ori & \(3.5\mathrm{E}{-3}\) & \(1.5\mathrm{E}{-3}\) & \(2\mathrm{E}{-3}\) & \(1.5\mathrm{E}{-3}\) & \(1.3\mathrm{E}{-3}\) & \(9\mathrm{E}{-4}\) \\
sNORB-Azi & \(5\mathrm{E}{-3}\) & \(5\mathrm{E}{-3}\) & \(3\mathrm{E}{-3}\) & \(5\mathrm{E}{-3}\) & \(1.73\mathrm{E}{-3}\) & \(8\mathrm{E}{-4}\) \\
sNORB-Ele & \(3.5\mathrm{E}{-3}\) & \(3.5\mathrm{E}{-3}\) & \(2\mathrm{E}{-3}\) & \(1.5\mathrm{E}{-3}\) & \(1.4\mathrm{E}{-3}\) & \(9\mathrm{E}{-4}\) \\
\bottomrule
\end{tabular}
\end{center}

\subsection{Mathematical Question Answering on MetaMathQA-40K}
\label{app:math_details}

\paragraph{Datasets and evaluation.}
For mathematical reasoning, we fine-tune LLaMA2-7B \citep{touvron2023llama2} on MetaMathQA-40K \citep{yu2024metamath} and evaluate the resulting models on GSM8K \citep{cobbe2021gsm8k} and MATH \citep{hendrycks2021math}.
This follows the standard mathematical reasoning protocol used in recent PEFT and orthogonal fine-tuning evaluations, where MetaMathQA-40K is used as the instruction-tuning set and GSM8K/MATH are used as held-out reasoning benchmarks. 
All methods are evaluated under the same finalized pipeline to avoid method-dependent differences in answer parsing, adapter loading, or post-training deployment.

For GSM8K, we report exact-match accuracy after extracting the final numerical answer from model generations. 
For MATH, we use the same corrected answer-extraction routine across all methods and report the final accuracy on the evaluation set. 
In both benchmarks, all finalized PEFT adapters are trained first and merged only after training before evaluation. 
This avoids discrepancies caused by evaluating unmerged adapters or by merging adapters at different stages of training.

For evaluation throughput, we use batch size \(32\) on GSM8K and batch size \(1\) on MATH. 
Unless otherwise stated, all methods are evaluated on a single NVIDIA A100-SXM GPU. 
HRA is evaluated on a single NVIDIA H200-SXM GPU because its implementation incurs a substantially larger memory requirement and does not fit under the same A100 setting used by the other methods.

\paragraph{Low-budget adaptation setting.}
The main mathematical reasoning comparison is designed as a matched low-budget experiment. 
Instead of comparing methods at the larger parameter scale used in prior full-budget PEFT tables, we restrict all main methods to \(0.52\)--\(0.79\)M trainable parameters.
This setting makes the comparison more diagnostic: it tests whether the method can use a very small adaptation budget effectively, rather than simply benefiting from a larger number of trainable parameters.

For LoRA, we use rank \(r=1\), which gives \(0.524\)M trainable parameters in this setting. 
For PSOFT and LOFT, we use rank \(r=128\), resulting in approximately \(0.52\)--\(0.54\)M trainable parameters. 
This is possible because LOFT and PSOFT place the trainable transformation inside a compact subspace rather than training two full low-rank factors. 
For BOFT and HRA, we use their closest low-budget configurations that remain comparable in parameter count.

\paragraph{Training protocol.}
All memory-feasible methods are fine-tuned with the same general training protocol. 
We use AdamW, cosine learning-rate decay, warmup ratio \(0.1\), maximum sequence length \(512\), train batch size \(64\), and \(2\) training epochs. 
Method-specific differences are limited to the adapter parameterization, rank/configuration, learning rate, and support selection rule. 
The finalized learning rate for each method is reported in Table~\ref{tab:math_lrs_lowbudget}.

\begin{table}[H]
\centering
\small
\caption{Shared hyperparameter settings for low-budget MetaMathQA-40K fine-tuning.}
\label{tab:math_shared_setup_lowbudget}
\begin{tabular}{lc}
\toprule
Hyperparameter & Value \\
\midrule
Base model & LLaMA2-7B \\
Training data & MetaMathQA-40K \\
Evaluation datasets & GSM8K, MATH \\
Optimizer & AdamW \\
Warmup ratio & \(0.1\) \\
LR schedule & Cosine \\
Max sequence length & \(512\) \\
Train batch size & \(64\) \\
Training epochs & \(2\) \\
GSM8K eval batch size & \(32\) \\
MATH eval batch size & \(1\) \\
Default GPU & NVIDIA A100-SXM \\
HRA GPU & NVIDIA H200-SXM \\
\bottomrule
\end{tabular}
\end{table}

\paragraph{Compared methods.}
We compare LOFT against representative low-rank and orthogonal PEFT baselines under the matched low-budget setting.

\begin{itemize}
    \item \textbf{LoRA} uses a rank-\(1\) additive low-rank update.
    \item \textbf{BOFT} uses the low-budget butterfly orthogonal configuration \(b=2,m=1\).
    \item \textbf{HRA} uses Householder Reflection Adaptation with \(r=2\) and strict orthogonality, i.e., \(\lambda=\infty\).
    \item \textbf{PSOFT} uses principal-subspace orthogonal fine-tuning with rank \(r=128\).
    \item \textbf{\loft$^{\pprin}$} uses the same orthogonal LOFT parameterization with principal support.
    \item \textbf{\loft$^{\pgrad}$} selects the support using gradient singular directions.
    \item \textbf{\loft$^{\pskew}$} selects the support from skew-gradient information.
\end{itemize}

The principal LOFT variant is included to isolate the effect of the parameterization when the support is chosen from pretrained weight geometry. 
The GradSVD and SkewGrad variants are included to test whether task-informed support selection improves mathematical reasoning under the same parameter budget.

\paragraph{Learning rates and method configurations.}
Table~\ref{tab:math_lrs_lowbudget} reports the finalized learning rates and parameter budgets used in the low-budget mathematical reasoning experiments. 
The learning rate is tuned per method, while the trainable parameter count is kept close to \(0.5\)M whenever possible.

\begin{table}[H]
\centering
\small
\caption{Method-specific settings for low-budget MetaMathQA-40K fine-tuning.}
\label{tab:math_lrs_lowbudget}
\begin{tabular}{lccc}
\toprule
Method & Configuration & Trainable Params & Learning rate \\
\midrule
BOFT & \(b=2,m=1\) & 0.786M & \(8\mathrm{E}{-4}\) \\
LoRA & \(r=1\) & 0.524M & \(8\mathrm{E}{-4}\) \\
HRA & \(r=2,\lambda=\infty\) & 0.524M & \(1\mathrm{E}{-3}\) \\
PSOFT & \(r=128\) & 0.536M & \(5\mathrm{E}{-4}\) \\
\loft$^{\pprin}$ & \(r=128\) & 0.520M & \(5\mathrm{E}{-4}\) \\
\loft$^{\pgrad}$ & \(r=128\) & 0.520M & \(8\mathrm{E}{-4}\) \\
\loft$^{\pskew}$ & \(r=128\) & 0.520M & \(8\mathrm{E}{-4}\) \\
\bottomrule
\end{tabular}
\end{table}

\paragraph{Peak memory measurement.}
The memory column in Table~\ref{tab:math_main} reports peak GPU memory allocated during training. 
For each method, we reset the CUDA peak-memory counter before fine-tuning and record the maximum allocated memory reached under the shared training pipeline. 
All memory-feasible methods are run on a single NVIDIA A100-SXM GPU. 
HRA is run on a single NVIDIA H200-SXM GPU due to out-of-memory issues on A100-SXM.

Under this measurement, LoRA and PSOFT require \(49.50\)GB and \(28.67\)GB, respectively, while all LOFT variants require \(25.68\)GB. 
Thus, LOFT has the lowest peak memory among the compared methods, reducing memory by \(48.1\%\) relative to LoRA and \(10.4\%\) relative to PSOFT. 
BOFT and HRA require \(63.24\)GB and \(124.92\)GB, respectively, corresponding to \(59.4\%\) and \(79.4\%\) reductions for LOFT in our implementation. These are training-loop peaks; for the gradient-informed supports, the one-off support construction has its own transient peak (\(44.66\)GiB during gradient collection on LLaMA2-7B), which we report separately in Appendix~\ref{app:construction_cost} and which should be counted separately in an end-to-end memory budget.

\paragraph{Evaluation implementation details.}
We found mathematical reasoning evaluation to be sensitive to two implementation details. 
First, answer extraction can noticeably affect GSM8K and MATH scores if the parser only recognizes overly restrictive numerical formats. 
We therefore use a corrected answer-extraction routine that handles more general numeric patterns and apply it consistently to all finalized methods. 
Second, adapter deployment timing can affect the final reported score. 
In the finalized pipeline, each adapter is trained to completion, then merged into the backbone only after training, and the merged model is used for evaluation. 
This ensures that the reported differences reflect the adaptation method rather than inconsistencies in adapter loading or merging.

\paragraph{Result summary.}
Table~\ref{tab:math_main} reports the final GSM8K and MATH results. The principal-support LOFT variant performs similarly to PSOFT, which is expected because both rely on pretrained weight geometry to define the adaptation support. Replacing the principal support with a gradient-informed support substantially improves accuracy at the same rank, trainable-parameter count and training-loop peak memory: every gradient-informed run exceeds every principal run on both tasks (lowest \(38.23\) vs.\ highest \(34.99\) on GSM8K; \(5.13\) vs.\ \(4.98\) on MATH). The two gradient-informed supports are close: the GSM8K gap between GradSVD and SkewGrad, \(1.29\) on the first run, is \(0.06\) over three runs, well inside the run-to-run standard deviations (\(1.36\) and \(0.80\)), so we do not read it as a ranking. The LOFT rows are means over three runs of each configuration under the same fixed training and evaluation sets; the first run of each, reported in an earlier version of this paper, was not the best of the three (Principal \(34.95\) vs.\ \(34.99\); SkewGrad \(39.65\) vs.\ \(40.33\) on GSM8K). The baseline rows are single runs.

\section{Full Short-Horizon Support Diagnostics}
\label{app:short_horizon_diag}

This appendix provides the full numerical results for short-horizon diagnostics of the first-order support-selection mechanism. 
These diagnostics are not intended to replace final task evaluation. 
Instead, they test a narrower initialization-local question: whether the support score predicted by Proposition~\ref{prop:signal-strength} identifies subspaces that expose useful short-horizon descent directions. 
We report two complementary diagnostics. 
The first is a controlled probe diagnostic, where supports are constructed from training-only calibration batches, the task head and all non-LOFT parameters are frozen, and only the LOFT adapter is updated for a small number of steps. 
This setting isolates the effect of the selected support. 
The second is an early-validation diagnostic under the standard GLUE training pipeline, where the LOFT adapter and task head are updated while the pretrained backbone remains frozen, and validation loss is recorded during the first 25 update steps.

GradSVD is included here as an auxiliary gradient-informed baseline. It selects a support from the principal right singular subspace of the downstream gradient, but unlike SkewGrad it does not follow the skew-symmetric first-order generator \(F=\mathrm{skew}(W_0^\top G)\) that governs orthogonal subspace adaptation. Its role is therefore diagnostic: it tests whether generic gradient awareness is sufficient, or whether the orthogonal-specific skew signal is needed.

\subsection{Does the First-Order Signal Predict Support Quality?}
\label{sec:first_order_diag}

We first define the support score used in the diagnostics. 
The first-order analysis gives a criterion for choosing the support before training, and the experiments below test whether this criterion aligns with short-horizon optimization behavior.
For each adapted layer \(\ell\), let $F_\ell=\mathrm{skew}(W_{0,\ell}^{\top}G_\ell),$ where \(G_\ell\) is the calibration gradient at initialization. 
Given a chosen support \(P=\{P_{r,\ell}\}_{\ell\in\mathcal L}\), define its relative signal capture as
\begin{equation*}
    \rho(P)
=
\frac{
\sum_{\ell\in\mathcal L}
\|P_{r,\ell}F_\ell P_{r,\ell}^{\top}\|_F^2
}{
\sum_{\ell\in\mathcal L}
2\sum_{k=1}^{\lfloor r/2\rfloor}\mu_{\ell,k}^2
},
\end{equation*}
where \(\{\pm i\mu_{\ell,k}\}\) are the skew-eigenvalue pairs of \(F_\ell\), ordered by 
\(\mu_{\ell,1}\geq \mu_{\ell,2}\geq\cdots\). 
The denominator is the best rank-\(r\) skew-invariant signal predicted by Proposition~\ref{prop:signal-strength}. 
Thus, SkewGrad has \(\rho(P)\approx 1\) on the calibration batches by construction. 
The key question is whether this local score also predicts short-horizon loss decrease.

Random is a fixed row-orthonormal support \(P_r=Q^\top\), with \(Q\) the reduced QR factor of a seeded Gaussian matrix \(M\in\mathbb{R}^{d_{\mathrm{in}}\times r}\) drawn once per layer and registered as a buffer, so it depends on neither \(W_0\) nor the data nor the gradient. It serves as a task-agnostic control: if any low-dimensional support were sufficient, Random would be competitive.
Principal support tests whether pretrained-weight geometry alone is sufficient, while GradSVD and SkewGrad test two ways of using downstream gradient information.

\paragraph{Controlled probe diagnostic.}
Figure~\ref{fig:app_probe_all_supports} and Table~\ref{tab:app_diag_full_probe} report the full controlled probe diagnostic with Random, Principal, GradSVD, and SkewGrad supports. 
The score \(\rho(P)\) is the relative calibration skew-signal capture defined above. 
Because SkewGrad is constructed from the top skew-invariant subspace, it achieves \(\rho(P)\approx 1\) on the calibration gradients by construction. 
The important question is whether this local signal translates into short-horizon loss decrease on held-out probe batches.

\begin{center}
\captionsetup{type=figure}
\begin{minipage}[t]{0.48\textwidth}
    \centering
    \includegraphics[width=\linewidth]{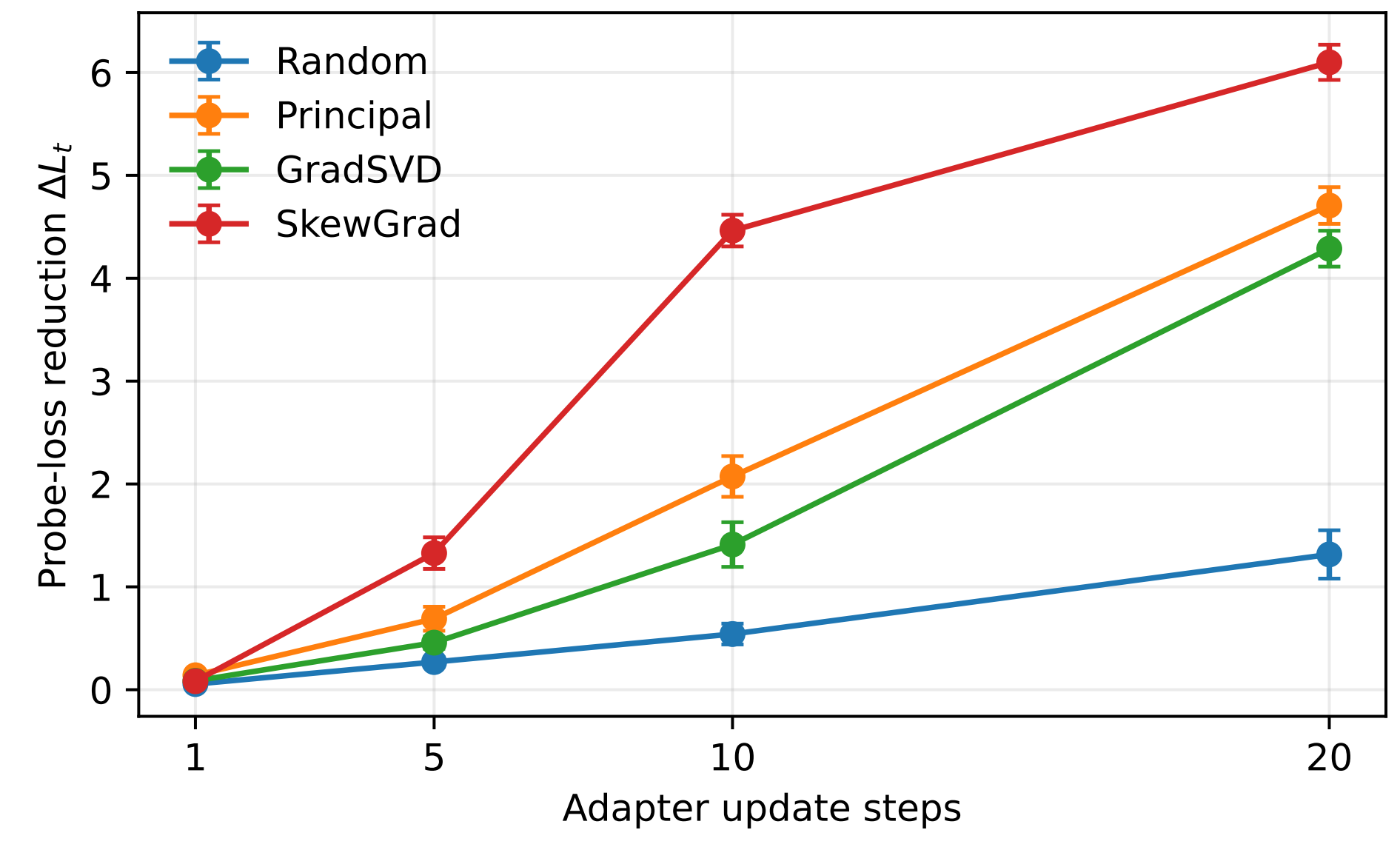}
    \par\vspace{2pt}{\small (a) STS-B.}\par
\end{minipage}
\hfill
\begin{minipage}[t]{0.48\textwidth}
    \centering
    \includegraphics[width=\linewidth]{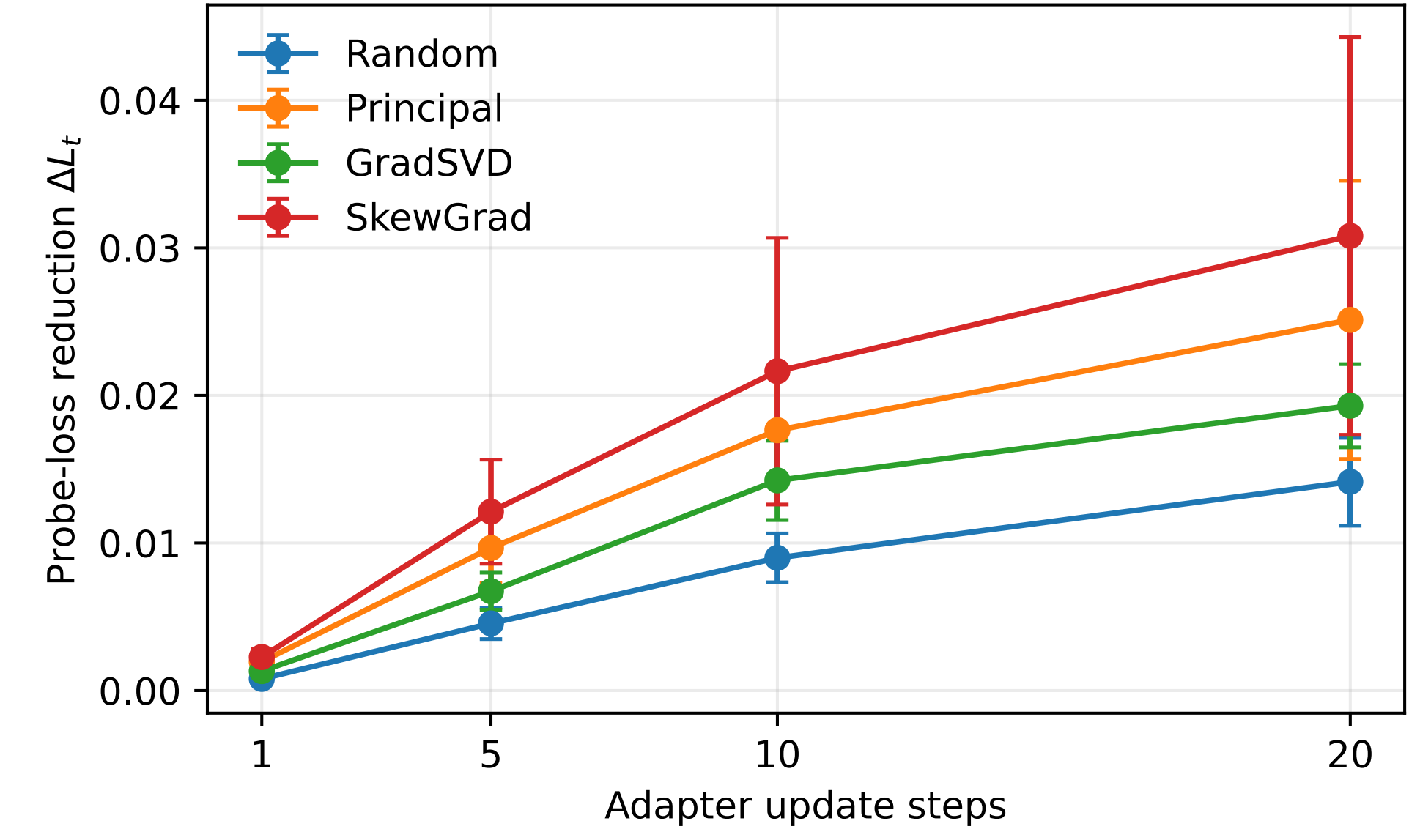}
    \par\vspace{2pt}{\small (b) CoLA.}\par
\end{minipage}
\par\vspace{4pt}
\captionof{figure}{
\textbf{All-support controlled probe diagnostic.}
The task head and all non-LOFT parameters are frozen, so the probe-loss reduction isolates the effect of support selection. GradSVD is included as a gradient-only auxiliary baseline. Compared with GradSVD, SkewGrad better matches the orthogonal first-order generator \(F=\mathrm{skew}(W_0^\top G)\), and yields the strongest probe-loss reduction from the fifth step onward on both tasks.
}
\label{fig:app_probe_all_supports}
\end{center}

\par\noindent\begin{minipage}{\linewidth}
\begin{center}
\captionsetup{type=table}
\captionof{table}{
\textbf{Full controlled probe diagnostics.}
Mean \(\pm\) standard deviation over five seeds. \(\rho(P)\) is relative calibration skew-signal capture, and \(\Delta L_t\) is probe-loss reduction after \(t\) adapter-update steps. The task head and all non-LOFT parameters are frozen.
}
\label{tab:app_diag_full_probe}
\small
\setlength{\tabcolsep}{4.5pt}
\renewcommand{\arraystretch}{1.05}

\textbf{(a) STS-B}
\vspace{0.3em}

\begin{adjustbox}{width=0.92\textwidth}
\begin{tabular}{lccccc}
\toprule
Support & \(\rho(P)\) & \(\Delta L_{1}\) & \(\Delta L_{5}\) & \(\Delta L_{10}\) & \(\Delta L_{20}\) \\
\midrule
Random
& \(0.0020 \pm 0.0003\)
& \(0.055 \pm 0.007\)
& \(0.269 \pm 0.106\)
& \(0.541 \pm 0.226\)
& \(1.315 \pm 0.526\) \\

Principal
& \(0.0357 \pm 0.0070\)
& \(\mathbf{0.140 \pm 0.026}\)
& \(0.690 \pm 0.261\)
& \(2.073 \pm 0.443\)
& \(4.707 \pm 0.400\) \\

GradSVD
& \(0.0491 \pm 0.0073\)
& \(0.087 \pm 0.008\)
& \(0.457 \pm 0.133\)
& \(1.411 \pm 0.485\)
& \(4.287 \pm 0.390\) \\

SkewGrad
& \(\mathbf{1.0000 \pm 0.0000}\)
& \(0.083 \pm 0.038\)
& \(\mathbf{1.328 \pm 0.343}\)
& \(\mathbf{4.463 \pm 0.345}\)
& \(\mathbf{6.099 \pm 0.383}\) \\
\bottomrule
\end{tabular}
\end{adjustbox}

\vspace{1.0em}

\textbf{(b) CoLA}
\vspace{0.3em}

\begin{adjustbox}{width=0.92\textwidth}
\begin{tabular}{lccccc}
\toprule
Support & \(\rho(P)\) & \(\Delta L_{1}\) & \(\Delta L_{5}\) & \(\Delta L_{10}\) & \(\Delta L_{20}\) \\
\midrule
Random
& \(0.0020 \pm 0.0004\)
& \(0.0008 \pm 0.0003\)
& \(0.0045 \pm 0.0024\)
& \(0.0090 \pm 0.0037\)
& \(0.0141 \pm 0.0067\) \\

Principal
& \(0.0353 \pm 0.0153\)
& \(0.0020 \pm 0.0014\)
& \(0.0097 \pm 0.0053\)
& \(0.0176 \pm 0.0082\)
& \(0.0251 \pm 0.0211\) \\

GradSVD
& \(0.0500 \pm 0.0052\)
& \(0.0013 \pm 0.0006\)
& \(0.0067 \pm 0.0028\)
& \(0.0142 \pm 0.0060\)
& \(0.0193 \pm 0.0063\) \\

SkewGrad
& \(\mathbf{0.9995 \pm 0.0001}\)
& \(\mathbf{0.0023 \pm 0.0012}\)
& \(\mathbf{0.0121 \pm 0.0079}\)
& \(\mathbf{0.0216 \pm 0.0202}\)
& \(\mathbf{0.0308 \pm 0.0301}\) \\
\bottomrule
\end{tabular}
\end{adjustbox}
\end{center}
\end{minipage}\par\medskip

The controlled probe results confirm the predicted ordering at its extremes but not between the two intermediate supports: GradSVD scores above the principal support on \(\rho(P)\) yet gives a smaller loss reduction at every measured horizon on both tasks.
On STS-B, SkewGrad gives the strongest reduction from step 5 onward and the largest reduction at step 20. 
On CoLA, SkewGrad gives the largest mean reduction across all measured steps, although the results have higher variance. 
GradSVD captures more skew-signal than Random but does not match SkewGrad, suggesting that generic gradient awareness is less effective than the orthogonal-specific skew signal.

\paragraph{Practical significance.}
The controlled probe diagnostic is not intended to replace final validation or test performance. Its practical role is to provide a cheap, calibration-stage test of whether a candidate support exposes useful local descent directions before running full fine-tuning. This is especially relevant in step-limited adaptation settings, such as rapid personalization, low-resource task adaptation, or federated/local fine-tuning, where the first few update steps can matter disproportionately. In this probe, aligning the support with the skew-symmetric generator specific to orthogonal adaptation gives a stronger short-horizon signal than generic gradient information.

\paragraph{Calibration-size robustness.}
We additionally vary the number \(K\) of training-only calibration mini-batches used to construct the SkewGrad support. This diagnostic tests whether SkewGrad depends on a large auxiliary calibration set, while keeping the rank, transform class, optimizer, and probe protocol fixed. Since \(\rho(P)\) is measured on the calibration skew signal used for support construction, it mainly serves as a construction-stability check. The more important quantity is the controlled probe loss reduction \(\Delta L_{20}\), which tests whether the constructed support yields useful short-horizon optimization. The results show that small calibration sizes are already sufficient: \(K=1\) or \(K=2\) achieves near-saturated \(\rho(P)\) on most tasks and competitive \(\Delta L_{20}\) across tasks. 
Increasing \(K\) does not systematically improve the probe objective, suggesting that SkewGrad does not rely on a large hidden calibration set. Larger \(K\) need not help monotonically: Proposition~\ref{prop:signal-strength} ranks supports at a fixed \(G\), and raising \(K\) changes \(G\) itself; heterogeneous batch gradients partly cancel, so a more stable average need not concentrate the finite-step descent direction.

\par\noindent\begin{minipage}{\linewidth}
\begin{center}
\captionsetup{type=table}
\captionof{table}{
\textbf{Calibration-size robustness for SkewGrad.}
\(K\) is the number of training-only calibration mini-batches used to construct the SkewGrad support. Values are mean \(\pm\) standard deviation over five seeds. Rank, transform, optimizer, and probe protocol are fixed. \(\Delta L_{20}\) is the controlled probe-loss reduction after 20 adapter-update steps.
}
\label{tab:calib_size_skewgrad}
\small
\setlength{\tabcolsep}{4.5pt}
\renewcommand{\arraystretch}{1.04}
\begin{tabular}{llcc}
\toprule
Task & \(K\) & \(\rho(P)\) & \(\Delta L_{20}\) \\
\midrule
CoLA & 1 & \(0.9991 \pm 0.0002\) & \(0.1133 \pm 0.0466\) \\
CoLA & 2 & \(0.9991 \pm 0.0001\) & \(0.0740 \pm 0.0415\) \\
CoLA & 4 & \(0.9995 \pm 0.0001\) & \(0.0308 \pm 0.0301\) \\
CoLA & 8 & \(0.9993 \pm 0.0002\) & \(0.0310 \pm 0.0195\) \\
\midrule
MRPC & 1 & \(0.9902 \pm 0.0034\) & \(0.1623 \pm 0.0154\) \\
MRPC & 2 & \(0.9978 \pm 0.0010\) & \(0.1845 \pm 0.0315\) \\
MRPC & 4 & \(0.9994 \pm 0.0002\) & \(0.1417 \pm 0.0164\) \\
MRPC & 8 & \(0.9998 \pm 0.0000\) & \(0.0870 \pm 0.0306\) \\
\midrule
STS-B & 1 & \(0.9999 \pm 0.0000\) & \(8.3965 \pm 0.3313\) \\
STS-B & 2 & \(0.9999 \pm 0.0000\) & \(7.4750 \pm 0.5630\) \\
STS-B & 4 & \(1.0000 \pm 0.0000\) & \(6.0991 \pm 0.3830\) \\
STS-B & 8 & \(1.0000 \pm 0.0000\) & \(3.4890 \pm 0.3343\) \\
\bottomrule
\end{tabular}
\end{center}
\end{minipage}\par\medskip

Table~\ref{tab:calib_size_skewgrad} shows that SkewGrad does not require a large calibration set to construct a useful support. The calibration skew-signal capture \(\rho(P)\) is already close to one with \(K=1\) on CoLA and STS-B, and reaches the same near-saturated regime on MRPC with \(K=2\) or \(K=4\). The probe-loss reduction also remains strong with small \(K\): on STS-B, SkewGrad gives the largest \(\Delta L_{20}\) among Random, Principal, GradSVD, and SkewGrad for \(K\leq 4\) (MRPC has no comparator probes); on CoLA, it is strongest for \(K\leq4\), while the \(K=8\) differences are small and within the high variance of the task. These results suggest that the proposed support construction is not dependent on a large auxiliary calibration set; a few training-only mini-batches are sufficient for the short-horizon optimization signal used by SkewGrad.

\paragraph{Sensitivity to the calibration draw.}
\label{app:calib_draw}
Table~\ref{tab:calib_size_skewgrad} varies the calibration \emph{size}. To isolate the calibration \emph{sample} for both gradient-informed supports, we fixed the training and probe randomness, model initialization, trajectory and all hyperparameters, and varied only the calibration-data shuffle across five draws (DeBERTaV3-base, \(r=46\), CoLA), a controlled test distinct from the five training seeds above. For two row-orthonormal supports \(P,Q\in\mathbb{R}^{r\times d}\) we report the mean principal angle \(\tfrac{1}{r}\sum_i\arccos\sigma_i(PQ^\top)\) and the normalized projector distance \(\|P^\top P-Q^\top Q\|_F/\sqrt{2r}=\sqrt{1-\|PQ^\top\|_F^2/r}\in[0,1]\) (\(0\) for identical subspaces, \(1\) for orthogonal ones), averaged over target matrices and draw pairs; the probe column is the 20-step loss reduction of each run relative to its own start (Table~\ref{tab:calib_draw}).
\begin{center}
\captionsetup{type=table}
\captionof{table}{Sensitivity to the calibration draw on CoLA (\(r=46\)): five calibration shuffles with everything else fixed. Probe and Matthews entries are mean \(\pm\) standard deviation over the five draws; the range of held-out Matthews is in parentheses.}
\label{tab:calib_draw}
\small
\setlength{\tabcolsep}{4pt}
\begin{adjustbox}{max width=\linewidth}
\begin{tabular}{lcccc}
\toprule
Support & Mean principal angle & Projector distance & 20-step probe-loss reduction & Held-out Matthews (20 epochs) \\
\midrule
GradSVD & \(58.80^{\circ}\) & \(0.829\) & \(0.1017\pm0.0039\) & \(69.82\pm1.48\) (68.62--71.63) \\
SkewGrad & \(57.10^{\circ}\) & \(0.814\) & \(0.0892\pm0.0080\) & \(71.85\pm1.29\) (70.30--73.12) \\
\bottomrule
\end{tabular}
\end{adjustbox}
\end{center}
The selected subspaces differ substantially across draws, yet all five draws reduce the probe loss for both supports, with the larger mean 20-step reduction for GradSVD in this experiment. Since a short-horizon probe does not settle final-score robustness, we also completed full 20-epoch CoLA training under the same protocol. The reference for this spread is ordinary run-to-run variation on the same task: across five training seeds, CoLA varies by \(\pm1.85\) for SkewGrad and \(\pm1.11\) for GradSVD (Table~\ref{tab:app_glue_loft_stds}), against calibration-draw standard deviations of \(\pm1.29\) and \(\pm1.48\). Calibration-sample variability is therefore comparable in magnitude to training-seed variability for SkewGrad, and somewhat larger for GradSVD. The per-draw margins (SkewGrad minus GradSVD) are \(+0.92\), \(+1.89\), \(+1.68\), \(+2.02\) and \(+3.66\), so SkewGrad is higher on each draw under this protocol; we report this as a within-protocol observation, not a general ranking.

\paragraph{Early validation diagnostic.}

The early-validation curves and summary statistics are reported in Figure~\ref{fig:first_order_diag} and Table~\ref{tab:valearly25_summary}. 
This diagnostic checks whether the short-horizon ordering observed in the controlled probe setting is also visible under the standard GLUE training pipeline. 
The LOFT adapter and task head are updated while the pretrained backbone remains frozen, and validation loss is recorded at fixed early update steps. 
We use these results only as a sanity check for early training behavior; they are not intended to replace final task performance.

\begin{center}
\captionsetup{type=table, skip=3pt}
\vspace{-0.1em}
\small
\setlength{\tabcolsep}{4pt}
\captionof{table}{
\textbf{Early validation loss.}
Mean \(\pm\) standard deviation over five seeds. 
Validation loss is recorded at fixed update steps in the first 25 training steps.
Average loss is computed over the listed step ranges. 
Wins count how many fixed steps give the lowest mean validation loss. 
Lower loss and more wins are better.
}
\begin{tabular}{llcccc}
\toprule
Task & Support & Avg. loss 5--20 & Avg. loss 5--25 & Wins 1--25 & Wins 5--25 \\
\midrule
STS-B
& Random    & \(4.5900 \pm 0.3400\) & \(4.1745 \pm 0.2927\) & 0 & 0 \\
& Principal & \(3.9261 \pm 0.3246\) & \(3.6557 \pm 0.2289\) & 2 & 1 \\
& SkewGrad  & \(\mathbf{3.4281 \pm 0.2808}\) & \(\mathbf{3.1687 \pm 0.2911}\) & \(\mathbf{4}\) & \(\mathbf{4}\) \\
\midrule
CoLA
& Random    & \(0.6235 \pm 0.0163\) & \(0.6200 \pm 0.0160\) & 0 & 0 \\
& Principal & \(0.6026 \pm 0.0206\) & \(0.5866 \pm 0.0218\) & 1 & 0 \\
& SkewGrad  & \(\mathbf{0.5673 \pm 0.0346}\) & \(\mathbf{0.5570 \pm 0.0419}\) & \(\mathbf{5}\) & \(\mathbf{5}\) \\
\bottomrule
\end{tabular}
\label{tab:valearly25_summary}
\vspace{-0.3em}
\end{center}

\begin{figure}[!tbp]
\centering
\includegraphics[width=0.62\linewidth]{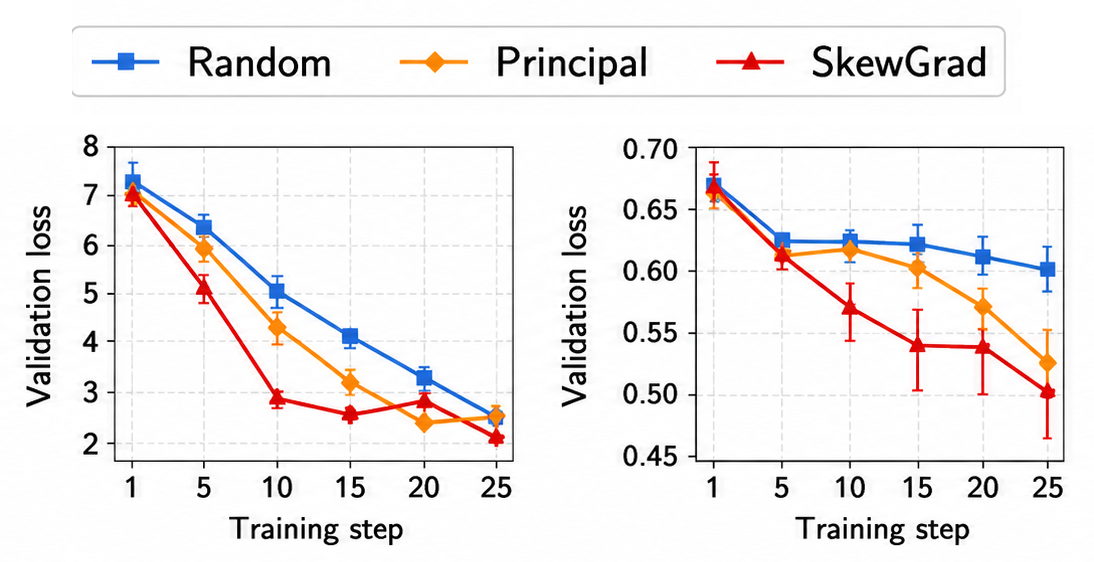}
\caption{
\textbf{Early-validation diagnostic.}
Validation loss during the first 25 steps of normal GLUE training on STS-B (left) and CoLA (right), comparing Random, Principal, and SkewGrad under matched rank, optimizer, learning rate, and orthogonal LOFT transform.
The controlled-probe results, including GradSVD, are in Figure~\ref{fig:app_probe_all_supports} and Table~\ref{tab:app_diag_full_probe}.
}
\label{fig:first_order_diag}
\vspace{-0.7em}
\end{figure}

The early-validation results show that the same ordering also appears in normal GLUE training. 
Within the first 25 updates, SkewGrad achieves the lowest average validation loss on both tasks and wins most fixed horizons. 
This supports the narrower claim that the score in Proposition~\ref{prop:signal-strength} identifies supports that help early optimization. 
Together with the controlled probe and calibration-size diagnostics, these results link the first-order support signal to short-horizon descent behavior under matched rank, optimizer, and transform settings.

\section{Full Support and Transform Ablations}
\label{app:support_transform_ablations}

This appendix reports LOFT variants omitted from the main GLUE and VTAB tables. 
For these two benchmarks, the main text fixes the default orthogonal transform \(T_r=\Torth\) and reports GradSVD support \(P_r=\pgrad\) and SkewGrad support \(P_r=\pskew\), together with principal support \(P_r=\pprin\) on GLUE. 
Here we additionally report principal support on VTAB, unconstrained transform variants \(T_r=\Tfree\), and auxiliary rank variants whenever available.
Some rows are repeated from the main text so that each benchmark has a complete ablation table in one place. 
After each table, we summarize which effects are attributable to support selection and which are attributable to the transform class.

\paragraph{GLUE ablations.}
Table~\ref{tab:app_glue_full} reports the full GLUE table, including OOM baselines and free in-subspace transforms. The main text reports only the default orthogonal transform.

\begin{table}[H]
\centering
\scriptsize
\renewcommand{\arraystretch}{1.06}
\setlength{\tabcolsep}{4pt}
\caption{
Full results on six controlled GLUE tasks using DeBERTaV3-base under the same hold-out validation protocol as PSOFT. We report Matthews correlation on CoLA, Pearson correlation on STS-B, and accuracy on RTE, MRPC, SST-2, and QNLI. LOFT rows are averaged over five random seeds measured in our codebase; baseline rows marked $^{\dagger}$ repeat the published values of \citet{wu2025memory} (our same-codebase reproductions are in Table~\ref{tab:glue_rerun}). Peak memory is the training-loop peak on a single NVIDIA Tesla V100-SXM2 GPU. For LOFT, the subscript denotes rank \(r\), the superscript denotes support \(P_r\), and \(T_r=\Torth\) is used by default unless \(T_r=\Tfree\) is shown.
}
\label{tab:app_glue_full}
\resizebox{\textwidth}{!}{
\renewcommand{\arraystretch}{1.25}
\begin{tabular}{lccccccccc}
\toprule
Method & Trainable Params & Mem. (GB) & CoLA & STS-B & RTE & MRPC & SST-2 & QNLI & \cellcolor{avgblue}Avg. \\
\midrule
FFT$^{\dagger}$ & 184M & 5.9 & 67.56 & 91.46 & 82.88 & 90.69 & 94.13 & 93.37 & \cellcolor{avgblue}86.68 \\
GOFTv2$^{\dagger}$ & 0.08M & 18.5 & 65.45 & \multicolumn{5}{c}{OOM} & \cellcolor{avgblue}N/A \\
qGOFTv2$^{\dagger}$ & 0.33M & 18.5 & 68.03 & \multicolumn{5}{c}{OOM} & \cellcolor{avgblue}N/A \\
BOFT$_{m=2}^{b=8}$$^{\dagger}$ & 1.41M & 6.3 & 68.85 & 91.09 & 83.60 & 88.40 & 95.28 & 93.78 & \cellcolor{avgblue}86.83 \\
OFTv2$_{b=32}$$^{\dagger}$ & 1.29M & 4.5 & 66.79 & 91.22 & 84.03 & 89.61 & 93.72 & 92.64 & \cellcolor{avgblue}86.34 \\
LoRA$_{r=8}$$^{\dagger}$ & 1.33M & 4.5 & 67.98 & 91.60 & 84.87 & 90.20 & 95.28 & 93.89 & \cellcolor{avgblue}87.30 \\
PiSSA$_{r=8}$$^{\dagger}$ & 1.33M & 4.5 & 66.50 & 91.40 & 83.77 & 89.90 & 93.17 & 92.72 & \cellcolor{avgblue}86.24 \\
DoRA$_{r=8}$$^{\dagger}$ & 1.41M & 5.8 & 67.06 & 91.60 & \textbf{87.19} & 90.49 & 95.23 & 94.09 & \cellcolor{avgblue}87.61 \\
LoRA-XS$_{r=136}$$^{\dagger}$ & 1.33M & 4.2 & 64.67 & 91.48 & 84.17 & 91.27 & 93.85 & 93.14 & \cellcolor{avgblue}86.43 \\
PSOFT$_{r=46}$$^{\dagger}$ & 0.08M & 4.1 & 70.42 & 91.56 & 86.74 & 90.49 & 95.55 & 93.47 & \cellcolor{avgblue}88.04 \\
\midrule
\rowcolor{psoftgreen}
\loft$_{r=46}^{\scriptscriptstyle \pprin}$
& 0.07M & 4.1 & 70.52 & 91.37 & 86.19 & 90.90 & \textbf{95.73} & 93.63 & \cellcolor{avgblue}88.06 \\
\rowcolor{psoftgreen}
\loft$_{r=46}^{\scriptscriptstyle \pprin,\,\Tfree}$
& 0.15M & 4.1 & 70.84 & 91.41 & 87.05 & 91.17 & 95.41 & 94.23 & \cellcolor{avgblue}88.35 \\
\rowcolor{psoftgreen}
\loft$_{r=46}^{\scriptscriptstyle \pgrad}$
& 0.07M & 4.1 & 70.63 & 91.32 & 87.05 & 91.67 & 95.28 & 94.39 & \cellcolor{avgblue}88.39 \\
\rowcolor{psoftgreen}
\loft$_{r=46}^{\scriptscriptstyle \pgrad,\,\Tfree}$
& 0.15M & 4.1 & \textbf{72.09} & 91.58 & 86.17 & 91.18 & 95.24 & 93.45 & \cellcolor{avgblue}88.29 \\
\rowcolor{psoftgreen}
\loft$_{r=46}^{\scriptscriptstyle \pskew}$
& 0.07M & 4.1 & 72.03 & \textbf{91.75} & 86.83 & 92.36 & 95.42 & 93.97 & \cellcolor{avgblue}\textbf{88.73} \\
\rowcolor{psoftgreen}
\loft$_{r=46}^{\scriptscriptstyle \pskew,\,\Tfree}$
& 0.15M & 4.1 & 71.04 & 91.50 & 86.77 & \textbf{92.65} & 94.72 & \textbf{94.43} & \cellcolor{avgblue}88.52 \\
\bottomrule
\end{tabular}}
\end{table}
\begin{center}
\captionsetup{type=table}
\captionof{table}{
Task-specific standard deviations for LOFT variants on the six-task GLUE setup. Values are standard deviations over five random seeds and use the same task order as Table~\ref{tab:app_glue_full}.
}
\label{tab:app_glue_loft_stds}
\small
\setlength{\tabcolsep}{4.5pt}
\renewcommand{\arraystretch}{1.04}
\begin{adjustbox}{max width=\linewidth}
\renewcommand{\arraystretch}{1.25}
\begin{tabular}{lcccccc}
\toprule
Method & CoLA & STS-B & RTE & MRPC & SST-2 & QNLI \\
\midrule
\loft$_{r=46}^{\scriptscriptstyle \pprin}$
& 1.13 & 0.32 & 1.58 & 1.71 & 0.31 & 0.25 \\
\loft$_{r=46}^{\scriptscriptstyle \pprin,\,\Tfree}$
& 1.12 & 0.26 & 2.83 & 1.22 & 0.35 & 0.31 \\
\loft$_{r=46}^{\scriptscriptstyle \pgrad}$
& 1.11 & 0.48 & 3.35 & 1.65 & 0.16 & 0.25 \\
\loft$_{r=46}^{\scriptscriptstyle \pgrad,\,\Tfree}$
& 1.81 & 0.38 & 2.74 & 2.97 & 0.82 & 0.62 \\
\loft$_{r=46}^{\scriptscriptstyle \pskew}$
& 1.85 & 0.38 & 3.43 & 1.50 & 0.38 & 0.46 \\
\loft$_{r=46}^{\scriptscriptstyle \pskew,\,\Tfree}$
& 2.53 & 0.39 & 2.62 & 2.06 & 0.91 & 0.37 \\
\bottomrule
\end{tabular}
\end{adjustbox}
\end{center}

\paragraph{GLUE ablation analysis.}
The full GLUE ablation separates the two LOFT design axes. Under the default orthogonal transform, the support means order as follows: \(\loft_{r=46}^{\pprin}\) reaches \(88.06\), \(\loft_{r=46}^{\pgrad}\) improves to \(88.39\), and \(\loft_{r=46}^{\pskew}\) further improves to \(88.73\). This is consistent with generic gradient awareness being useful; on GLUE the skew-specific signal gives the strongest support, an ordering the multilingual setting reverses (Appendix~\ref{app:lowres-lang}).

The transform-class comparison is more nuanced. The free transform helps in the principal-support regime, improving \(\loft_{r=46}^{\pprin}\) from \(88.06\) to \(88.35\), but it is not uniformly better than the orthogonal transform. For GradSVD, the orthogonal and free variants are close in average score, \(88.39\) versus \(88.29\). For SkewGrad, the orthogonal transform remains stronger than the free transform, \(88.73\) versus \(88.52\). These results suggest that, on GLUE, the primary gain comes from selecting a better task-aligned support. The OOM rows for GOFTv2/qGOFTv2 further indicate that the comparison is constrained by practical memory feasibility, not only by final accuracy.

Table~\ref{tab:app_glue_loft_stds} reports the corresponding across-seed standard deviations for the LOFT rows. The variance is largest on small or seed-sensitive tasks such as RTE and CoLA, while STS-B, SST-2, and QNLI are comparatively stable across seeds.

\paragraph{GLUE training trajectories with GradSVD.}
\label{app:gradsvd_trajectory}
Figure~\ref{fig:glue_training_dynamics} includes a GradSVD full-training trajectory run under the same matched protocol (learning rate \(5\times10^{-4}\) on CoLA and \(4\times10^{-4}\) on STS-B), which tests whether the figure only reflects the use of downstream gradients in general. All supports share the same epoch-0 initialization (STS-B loss \(8.675\) / Pearson \(0.387\); CoLA loss \(0.664\) / Matthews \(0.000\)); Table~\ref{tab:gradsvd_trajectory} reports the checkpoints.
\begin{center}
\captionsetup{type=table}
\captionof{table}{Full-training checkpoints under the matched GLUE protocol for the four supports (one fixed training seed). Best value per column and curve in bold.}
\label{tab:gradsvd_trajectory}
\small
\setlength{\tabcolsep}{5pt}
\begin{tabular}{llccccc}
\toprule
Curve & Support & E1 & E5 & E10 & E15 & E20 \\
\midrule
STS-B loss \(\downarrow\) & Random & 5.9905 & 0.4245 & 0.2987 & 0.2523 & 0.2334 \\
 & Principal & 5.1458 & 0.4216 & 0.2487 & 0.1888 & 0.1499 \\
 & GradSVD & 5.3133 & 0.6770 & 0.3549 & 0.2525 & 0.1481 \\
 & SkewGrad & \textbf{4.3117} & \textbf{0.3338} & \textbf{0.2044} & \textbf{0.1533} & \textbf{0.1309} \\
\midrule
STS-B Pearson \(\uparrow\) & Random & 0.6762 & 0.8960 & 0.9048 & 0.9057 & 0.9052 \\
 & Principal & \textbf{0.8342} & 0.9022 & 0.9038 & 0.9081 & 0.9071 \\
 & GradSVD & 0.7382 & 0.8683 & 0.9050 & 0.9107 & 0.9133 \\
 & SkewGrad & 0.8225 & \textbf{0.9098} & \textbf{0.9111} & \textbf{0.9187} & \textbf{0.9220} \\
\midrule
CoLA loss \(\downarrow\) & Random & 0.5165 & 0.3217 & 0.2723 & 0.2488 & 0.2298 \\
 & Principal & 0.4792 & 0.3001 & 0.2033 & 0.1578 & 0.1379 \\
 & GradSVD & 0.5368 & 0.3161 & 0.2062 & \textbf{0.1372} & 0.1320 \\
 & SkewGrad & \textbf{0.4647} & \textbf{0.2889} & \textbf{0.1975} & 0.1417 & \textbf{0.1303} \\
\midrule
CoLA Matthews \(\uparrow\) & Random & 0.4121 & 0.6367 & 0.6415 & 0.6569 & 0.6620 \\
 & Principal & 0.5754 & 0.6663 & \textbf{0.7155} & 0.7247 & 0.7159 \\
 & GradSVD & 0.5646 & 0.6919 & 0.6584 & 0.6919 & 0.7011 \\
 & SkewGrad & \textbf{0.5996} & \textbf{0.7005} & 0.7079 & \textbf{0.7254} & \textbf{0.7297} \\
\bottomrule
\end{tabular}
\end{center}
GradSVD improves over Principal on three of the four endpoints, but SkewGrad has the strongest endpoint on all four curves, and in this fixed-seed trajectory GradSVD is less consistent mid-training: its STS-B loss is above Random at epochs 5 and 10 and nearly tied at epoch 15, and its Pearson is below Random at epoch 5. The result therefore suggests an additional benefit from the skew-specific construction, although this trajectory uses one fixed training seed and the endpoint gap between GradSVD and SkewGrad stays small.

\paragraph{VTAB ablations.}
Table~\ref{tab:app_vtab_full} reports the full 19-task VTAB-1K results, including auxiliary LOFT support and transform variants omitted from the main text.

\par\noindent\begin{minipage}{\linewidth}
\begin{center}
\captionsetup{type=table}
\captionof{table}{
Full VTAB-1K results of fine-tuned ViT-B/16, as top-1 accuracy (\%), with Avg.\ the category-balanced mean of the three group means (FFT's average is recomputed on this basis). LOFT rows are averaged over five random seeds measured in our codebase; baseline rows marked $^{\dagger}$ repeat the published values of \citet{wu2025memory} in Table~\ref{tab:vtab_main}. We report trainable parameter count and peak memory to show the efficiency--performance trade-off across PEFT methods. For LOFT, the subscript denotes rank \(r\), the superscript denotes support \(P_r\), and \(T_r=\Torth\) is used by default unless \(T_r=\Tfree\) is shown.
}
\label{tab:app_vtab_full}
\tiny
\renewcommand{\arraystretch}{1.03}
\setlength{\tabcolsep}{1.25pt}
\begin{adjustbox}{width=\textwidth}
\renewcommand{\arraystretch}{1.25}
\begin{tabular}{lcc|ccccccc|cccc|cccccccc|c}
\toprule
& & & \multicolumn{7}{c|}{Natural} & \multicolumn{4}{c|}{Specialized} & \multicolumn{8}{c|}{Structured} & \\
\cmidrule(lr){4-10} \cmidrule(lr){11-14} \cmidrule(lr){15-22}
Method & \vthead{Trainable Params} & \vthead{Mem.}
& \vthead{Cifar100}
& \vthead{Caltech101}
& \vthead{DTD102}
& \vthead{Flower102}
& \vthead{Pets}
& \vthead{SVHN}
& \vthead{Sun397}
& \vthead{Camelyon}
& \vthead{EuroSAT}
& \vthead{Resisc45}
& \vthead{Retinopathy}
& \vthead{Clevr-Count}
& \vthead{Clevr-Dist}
& \vthead{DMLab}
& \vthead{KITTI-Dist}
& \vthead{dSpr-Loc}
& \vthead{dSpr-Ori}
& \vthead{sNORB-Azi}
& \vthead{sNORB-Ele}
& \cellcolor{avgblue}\vthead{Avg.} \\
\midrule
FFT$^{\dagger}$ & 85.9M & 8.2
& 70.7 & 89.3 & 69.5 & 99.0 & 90.4 & 81.7 & 54.9
& 85.4 & 93.6 & 83.8 & 74.5
& 58.3 & 51.5 & 43.2 & 75.0 & 73.1 & 48.7 & 16.4 & 30.0
& \cellcolor{avgblue}71.1 \\

GOFTv2$^{\dagger}$ & 0.08M & OOM
& \multicolumn{20}{c}{N/A} \\

qGOFTv2$^{\dagger}$ & 0.33M & OOM
& \multicolumn{20}{c}{N/A} \\

BOFT$_{m=2}^{b=8}$$^{\dagger}$ & 1.41M & 10.9
& 70.6 & 88.2 & 69.8 & 99.0 & 91.4 & 77.4 & 55.1
& 85.1 & 93.6 & 82.3 & 74.9
& 61.8 & 50.4 & 42.9 & 76.1 & 73.7 & 48.8 & 15.7 & 30.8
& \cellcolor{avgblue}70.9 \\

OFTv2$_{b=32}$$^{\dagger}$ & 1.29M & 7.7
& 68.5 & 88.9 & 67.5 & 98.4 & 89.5 & \textbf{86.9} & 53.6
& 86.0 & 94.1 & \textbf{84.2} & 74.6
& 58.7 & 56.4 & \textbf{46.7} & 78.5 & 81.1 & 48.1 & 17.3 & 32.5
& \cellcolor{avgblue}72.1 \\

LoRA$_{r=8}$$^{\dagger}$ & 1.33M & 9.9
& 71.4 & 88.4 & 70.1 & 99.0 & 91.4 & 76.6 & 55.7
& 85.9 & 94.2 & 83.3 & 74.1
& \textbf{72.0} & 54.3 & 43.0 & 76.6 & 74.8 & 48.6 & 16.4 & 31.8
& \cellcolor{avgblue}71.8 \\

PiSSA$_{r=8}$$^{\dagger}$ & 1.33M & 9.9
& 70.7 & 88.7 & 68.9 & 99.2 & 91.0 & 81.9 & 53.3
& 82.6 & 93.4 & 83.0 & 74.0
& 71.0 & \textbf{60.2} & 44.0 & 77.1 & 81.9 & 51.8 & 18.1 & 33.1
& \cellcolor{avgblue}72.3 \\

DoRA$_{r=8}$$^{\dagger}$ & 1.41M & 17.8
& 70.7 & 89.0 & 69.8 & 98.9 & 91.0 & 81.7 & 55.5
& 85.7 & 94.2 & 83.5 & 74.8
& 67.3 & 54.2 & 45.1 & 77.4 & \textbf{82.0} & 48.5 & 16.9 & 31.5
& \cellcolor{avgblue}72.3 \\

LoRA-XS$_{r=136}$$^{\dagger}$ & 1.33M & 6.6
& 68.5 & 89.4 & 68.4 & 98.7 & 90.9 & 84.5 & 54.1
& 84.0 & 94.3 & 80.8 & 73.6
& 60.0 & 57.7 & 45.8 & 79.6 & 80.6 & 48.1 & 17.4 & 30.8
& \cellcolor{avgblue}71.6 \\

PSOFT$_{r=46}$$^{\dagger}$ & 0.08M & 6.2
& \textbf{71.9} & 89.6 & \textbf{70.3} & 99.1 & 91.8 & \textbf{86.9} & \textbf{55.9}
& 84.6 & 94.2 & 82.4 & 75.2
& 71.2 & 59.9 & 45.7 & 79.6 & 80.9 & \textbf{52.9} & \textbf{20.0} & 32.9
& \cellcolor{avgblue}73.4 \\

\midrule
\rowcolor{psoftgreen}
\loft$_{r=42}^{\scriptscriptstyle \pprin}$ & 0.06M & 5.7
& 70.3 & 90.1 & 69.2 & 99.1 & 91.1 & 86.0 & 54.9
& 86.5 & 94.5 & 83.8 & 75.3
& 70.9 & 59.9 & 45.0 & 79.6 & 79.5 & 51.8 & 19.9 & 32.6
& \cellcolor{avgblue}73.3 \\

\rowcolor{psoftgreen}
\loft$_{r=42}^{\scriptscriptstyle \pprin,\,\Tfree}$ & 0.13M & 5.7
& 70.2 & \textbf{90.5} & 70.0 & 99.2 & 91.5 & 86.2 & 55.2
& 86.3 & 94.7 & 83.5 & 75.2
& 68.7 & 60.1 & 46.1 & 79.7 & 79.7 & 50.6 & 19.9 & 32.1
& \cellcolor{avgblue}73.3 \\

\rowcolor{psoftgreen}
\loft$_{r=46}^{\scriptscriptstyle \pprin}$ & 0.07M & 5.7
& 68.0 & 89.5 & 69.1 & 99.1 & 91.2 & 86.1 & 54.5
& 85.0 & 94.3 & 80.7 & 75.3
& 65.2 & 58.0 & 45.2 & 79.2 & 79.4 & 48.3 & 19.5 & 31.7
& \cellcolor{avgblue}72.3 \\

\rowcolor{psoftgreen}
\loft$_{r=46}^{\scriptscriptstyle \pprin,\,\Tfree}$ & 0.15M & 5.7
& 69.1 & 89.8 & 69.5 & 99.1 & 91.5 & 86.2 & 54.9
& 84.9 & 94.3 & 81.5 & 74.9
& 67.3 & 56.8 & 45.6 & 77.0 & 79.2 & 48.1 & 19.0 & 30.8
& \cellcolor{avgblue}72.3 \\

\rowcolor{psoftgreen}
\loft$_{r=46}^{\scriptscriptstyle \pgrad}$ & 0.07M & 5.7
& 69.9 & 89.6 & 69.4 & 99.2 & 92.0 & 85.2 & 55.2
& 86.4 & \textbf{95.7} & 82.7 & 75.3
& 68.9 & 59.8 & 46.2 & 79.8 & 79.8 & 50.2 & 19.8 & \textbf{33.2}
& \cellcolor{avgblue}73.3 \\

\rowcolor{psoftgreen}
\loft$_{r=46}^{\scriptscriptstyle \pskew}$ & 0.07M & 5.7
& 71.8 & 90.1 & 70.2 & \textbf{99.3} & \textbf{92.1} & 86.2 & 55.7
& \textbf{86.7} & 95.0 & 83.7 & \textbf{75.5}
& 71.2 & 59.9 & 46.1 & \textbf{80.1} & 80.1 & 52.5 & \textbf{20.0} & 33.1
& \cellcolor{avgblue}\textbf{73.8} \\
\bottomrule
\end{tabular}
\end{adjustbox}
\end{center}
\end{minipage}\par\medskip

\begin{center}
\captionsetup{type=table}
\captionof{table}{
Task-specific standard deviations for selected LOFT variants on VTAB-1K. Values are reported in percentage points and computed over five random seeds. We report the main principal-support row, its free-transform counterpart, and the two gradient-informed orthogonal supports used in the VTAB ablation.
}
\label{tab:app_vtab_loft_stds}
\tiny
\renewcommand{\arraystretch}{1.03}
\setlength{\tabcolsep}{2.0pt}
\begin{adjustbox}{width=\textwidth}
\renewcommand{\arraystretch}{1.25}
\begin{tabular}{lccccccccccccccccccc}
\toprule
Method
& \vthead{Cifar100}
& \vthead{Caltech101}
& \vthead{DTD102}
& \vthead{Flower102}
& \vthead{Pets}
& \vthead{SVHN}
& \vthead{Sun397}
& \vthead{Camelyon}
& \vthead{EuroSAT}
& \vthead{Resisc45}
& \vthead{Retinopathy}
& \vthead{Clevr-Count}
& \vthead{Clevr-Dist}
& \vthead{DMLab}
& \vthead{KITTI-Dist}
& \vthead{dSpr-Loc}
& \vthead{dSpr-Ori}
& \vthead{sNORB-Azi}
& \vthead{sNORB-Ele} \\
\midrule
\loft$_{r=42}^{\scriptscriptstyle \pprin}$
& 1.91 & 0.74 & 0.32 & 0.20 & 0.13 & 0.92 & 0.64 & 0.22 & 0.18 & 1.24 & 0.48 & 2.25 & 0.52 & 0.41 & 1.75 & 0.36 & 1.68 & 0.72 & 0.66 \\

\loft$_{r=42}^{\scriptscriptstyle \pprin,\,\Tfree}$
& 1.43 & 0.77 & 0.70 & 0.16 & 0.40 & 0.19 & 0.85 & 0.93 & 0.59 & 1.19 & 0.29 & 2.55 & 0.66 & 0.68 & 1.40 & 0.34 & 3.27 & 1.57 & 1.09 \\

\loft$_{r=46}^{\scriptscriptstyle \pgrad}$
& 0.23 & 0.19 & 0.14 & 0.02 & 0.20 & 0.22 & 0.13 & 0.17 & 0.04 & 0.34 & 0.68 & 1.93 & 0.55 & 0.20 & 0.72 & 0.39 & 0.34 & 0.23 & 0.53 \\

\loft$_{r=46}^{\scriptscriptstyle \pskew}$
& 0.18 & 0.22 & 0.25 & 0.08 & 0.37 & 0.32 & 0.11 & 0.38 & 0.15 & 0.29 & 0.28 & 1.50 & 0.45 & 0.39 & 0.42 & 0.51 & 0.68 & 0.32 & 0.48 \\
\bottomrule
\end{tabular}
\end{adjustbox}
\end{center}

\paragraph{VTAB ablation analysis.}
The VTAB ablation shows a more heterogeneous pattern than GLUE, which is expected because VTAB aggregates natural, specialized, and structured visual tasks. The \(r=42\) principal-support row is the closest match to the PSOFT regime, up to PSOFT's scaling and relaxation terms: \(\loft_{r=42}^{\pprin}\) comes close to PSOFT's average (\(73.3\) vs.\ \(73.4\)) while using fewer parameters and less memory. Increasing the principal rank to \(r=46\) does not improve the average, indicating that simply increasing the subspace width is not sufficient. GradSVD gives strong task-level results on several datasets, such as Flower102, Pets, EuroSAT, and sNORB-Ele, but its average remains below SkewGrad. The best overall average is obtained by \(\loft_{r=46}^{\pskew}\), which reaches \(73.8\) and improves over both PSOFT and the principal-support LOFT variants. This comparison is especially informative because the \(r=46\) principal-support rows do not improve over the \(r=42\) principal-support row, whereas the SkewGrad support at \(r=46\) gives the strongest average. Thus, the improvement is better explained by support selection than by simply increasing rank. The result supports the same qualitative conclusion as GLUE: generic gradient awareness helps, while the orthogonal-specific skew-gradient support is strongest in this VTAB setting. Free-transform variants do not improve the VTAB average, suggesting that visual transfer benefits more from support selection than from relaxing the in-subspace orthogonality constraint. Table~\ref{tab:app_vtab_loft_stds} reports the corresponding across-seed variability for selected LOFT variants. The gradient-informed supports are generally stable: the average task-level standard deviation is about \(0.38\) percentage points for \(P_r=\pgrad\) and \(0.39\) percentage points for \(P_r=\pskew\). The largest variances occur on a small number of structured tasks, especially Clevr-Count and dSpr-Ori, while most natural and specialized tasks have substantially smaller variation. This helps rule out the concern that the VTAB average gain is driven by highly unstable seed behavior, although we still interpret the VTAB improvement as modest and not uniform across tasks (\(10\) wins, \(3\) ties and \(6\) losses against the published PSOFT results).

\section{Additional Robustness Analyses}
\label{app:robustness}

\subsection{Low-Resource Language OOD Adaptation}
\label{app:lowres-lang}

We further evaluate whether gradient-informed support selection remains effective in a low-resource language adaptation setting. We fine-tune Llama-3.2-3B \citep{meta2024llama32} on Bactrian-X \citep{li2023bactrianx} and report held-out response-only negative log-likelihood (NLL). We use the base-model bits-per-byte (BPB) on held-out raw text as a rough proxy for language OOD strength. The shared experimental setup is summarized in Table~\ref{tab:lowres_lang_setup}.

\begin{table}[H]
\centering
\small
\caption{Shared hyperparameter settings for low-resource language adaptation.}
\label{tab:lowres_lang_setup}
\begin{tabular}{p{0.34\linewidth}p{0.48\linewidth}}
\toprule
Hyperparameter & Value \\
\midrule
Base model & Llama-3.2-3B \\
Training data & Bactrian-X, 5,000 examples per language \\
Hold-out data & 500 examples per language \\
Evaluation metric & Response-only NLL \(\downarrow\) \\
OOD proxy & Base-model BPB on held-out raw text \\
Optimizer & AdamW \\
Learning-rate schedule & Cosine, warmup ratio \(0.1\) \\
Default learning rate & \(2\times 10^{-4}\) \\
Weight decay & \(0.0\) \\
Training epochs & 2 \\
Batch size & 1 \\
Gradient accumulation & 64 \\
Precision & bf16 \\
GPU & NVIDIA A100 \\
Seeds & three \\
LOFT / PSOFT rank & 354 \\
Gradient calibration & 16 examples \\
PSOFT extras & Magnitude vectors and Cayley--Neumann approximation with 5 terms \\
\bottomrule
\end{tabular}
\end{table}

\begin{table}[H]
\centering
\small
\setlength{\tabcolsep}{4pt}
\caption{Low-resource language adaptation on Bactrian-X. We report held-out response-only NLL
(lower is better), averaged over three seeds. All methods use learning rate \(2\times 10^{-4}\).
BPB is measured before fine-tuning and serves as an approximate OOD-strength proxy. Subscripts are the across-seed standard deviation over the three seeds.}
\label{tab:lowres_lang_main}
\begin{tabular}{lcccccc}
\toprule
Method &
\begin{tabular}[c]{@{}c@{}}Indonesian\\(id)\end{tabular} &
\begin{tabular}[c]{@{}c@{}}English\\(en)\end{tabular} &
\begin{tabular}[c]{@{}c@{}}Tagalog\\(tl)\end{tabular} &
\begin{tabular}[c]{@{}c@{}}Afrikaans\\(af)\end{tabular} &
\begin{tabular}[c]{@{}c@{}}Swahili\\(sw)\end{tabular} &
\cellcolor{avgblue}Avg. \\
\midrule
BPB       & 0.596 & 0.626 & 0.811 & 0.831 & 1.065 & -- \\
\midrule
Principal & 1.0556{\scriptsize$\pm$0.0008} & 1.8715{\scriptsize$\pm$0.0054} & 1.3771{\scriptsize$\pm$0.0003} & 1.3513{\scriptsize$\pm$0.0010} & 1.3631{\scriptsize$\pm$0.0006} & \cellcolor{avgblue}1.4037 \\
GradSVD   & \textbf{1.0553}{\scriptsize$\pm$0.0014} & \textbf{1.8655}{\scriptsize$\pm$0.0061} & \textbf{1.3708}{\scriptsize$\pm$0.0007} & \textbf{1.3463}{\scriptsize$\pm$0.0009} & \textbf{1.3480}{\scriptsize$\pm$0.0004} & \cellcolor{avgblue}\textbf{1.3972} \\
SkewGrad  & 1.0666{\scriptsize$\pm$0.0015} & 1.9037{\scriptsize$\pm$0.0114} & 1.3768{\scriptsize$\pm$0.0008} & 1.3585{\scriptsize$\pm$0.0008} & 1.3503{\scriptsize$\pm$0.0002} & \cellcolor{avgblue}1.4112 \\
PSOFT     & 1.0557{\scriptsize$\pm$0.0010} & 1.8747{\scriptsize$\pm$0.0052} & 1.3770{\scriptsize$\pm$0.0005} & 1.3511{\scriptsize$\pm$0.0003} & 1.3628{\scriptsize$\pm$0.0006} & \cellcolor{avgblue}1.4043 \\
\bottomrule
\end{tabular}
\end{table}

Table~\ref{tab:lowres_lang_main} shows that GradSVD improves over the principal weight-based support on every language in this OOD adaptation setting. The gain is small when the base model is already close to the target language distribution, such as Indonesian, but becomes much
larger for the strongest OOD language, Swahili. PSOFT closely tracks the principal support, which is expected because both choose the adaptation support from pretrained-weight geometry rather than from downstream gradient information. SkewGrad trails the principal support on Indonesian, English and Afrikaans, matches it on Tagalog (\(1.3768\) vs.\ \(1.3771\)) and improves on it on Swahili (\(1.3503\) vs.\ \(1.3631\)), so on the five-language average it trails the principal support (\(1.4112\) vs.\ \(1.4037\)), suggesting that mixing pretrained-weight geometry into the skew signal can be harmful when pretrained directions are already well aligned with the task.

To check that the GradSVD advantage is not an artifact of a single learning rate, we run an extended learning-rate sweep on Swahili, the strongest OOD language in this set. We also include a parameter-matched LoRA baseline and a half-parameter tier.

\begin{table}[H]
\centering
\small
\setlength{\tabcolsep}{5pt}
\caption{Swahili learning-rate and parameter-budget ablation. We report each method at its own best learning rate, as a mean over three seeds with the across-seed standard deviation as a subscript. The 12M tier compares LOFT/PSOFT rank \(354\) with LoRA rank \(8\); the 6M tier
compares LOFT/PSOFT rank \(250\) with LoRA rank \(4\).}
\label{tab:lowres_sw_bestlr}
\begin{tabular}{lcccc}
\toprule
& \multicolumn{2}{c}{12M tier} & \multicolumn{2}{c}{6M tier} \\
\cmidrule(lr){2-3} \cmidrule(lr){4-5}
Method & Best lr & NLL \(\downarrow\) & Best lr & NLL \(\downarrow\) \\
\midrule
Principal & \(5\times 10^{-4}\) & 1.3307{\scriptsize$\pm$0.0014} & \(1\times 10^{-3}\) & 1.3546{\scriptsize$\pm$0.0016} \\
GradSVD   & \(5\times 10^{-4}\) & \textbf{1.3170}{\scriptsize$\pm$0.0009} & \(1\times 10^{-3}\) & \textbf{1.3395}{\scriptsize$\pm$0.0015} \\
SkewGrad  & \(3\times 10^{-4}\) & 1.3399{\scriptsize$\pm$0.0015} & \(5\times 10^{-4}\) & 1.3697{\scriptsize$\pm$0.0004} \\
PSOFT     & \(5\times 10^{-4}\) & 1.3300{\scriptsize$\pm$0.0019} & \(1\times 10^{-3}\) & 1.3534{\scriptsize$\pm$0.0031} \\
LoRA      & \(1\times 10^{-3}\) & 1.3362{\scriptsize$\pm$0.0008} & \(1\times 10^{-3}\) & 1.3646{\scriptsize$\pm$0.0009} \\
\bottomrule
\end{tabular}
\end{table}

Table~\ref{tab:lowres_sw_bestlr} confirms that LOFT's gradient-informed support selection continues to improve performance after each method is given its own tuned learning rate. At the 12M-parameter tier, GradSVD lowers NLL by about \(1.0\%\) compared with the best principal-support variant and by about \(1.4\%\) compared with the best parameter-matched LoRA baseline. The gain becomes larger at the 6M-parameter tier, where GradSVD improves over LoRA by about \(1.8\%\). This is consistent, on this one language and two tiers, with loss-informed supports becoming more valuable under tighter adaptation budgets. SkewGrad is competitive at lower learning rates but becomes less stable as the learning rate increases; therefore, for stability, we use GradSVD as the default gradient-informed support in this low-resource language setting.

\subsection{SkewGrad on MRPC and RTE}
\label{app:glue_ablation_skew}

We provide two additional robustness analyses on MRPC and RTE to stress-test the loss-informed support-selection claim beyond the main six-task GLUE table. These two tasks are useful probes because they are relatively small, seed-sensitive sentence-pair classification tasks, so improvements on them are less likely to be explained solely by large-data averaging effects. In all experiments below, the support is constructed using the SkewGrad rule and the resulting support is fixed during training; we then compare the default orthogonal transform, the unconstrained free transform within the same support, and PSOFT under the corresponding setting.

The first analysis is a low-resource data-fraction sweep at fixed rank \(r=46\), where we vary the amount of training data while keeping the adaptation budget fixed. This tests whether SkewGrad remains useful when the calibration and downstream training signal are estimated from reduced supervision. The second analysis is a full-data rank sweep over \(r\in\{16,46,64\}\), where we vary the adaptation budget while keeping the full training set available. This tests whether the observed gains depend on the main rank \(r=46\), or whether they persist across a broader parameter range. Unless otherwise stated, each setting uses the task-specific learning rate selected for that setting, and all reported values are averaged over five random seeds.

These experiments should be interpreted as robustness checks for the support-selection mechanism rather than as additional tests of the principal-support regime. In particular, the comparisons isolate whether the loss-informed SkewGrad support remains competitive against PSOFT when supervision is reduced or rank is changed, and whether the orthogonal/free distinction within the same support changes the outcome.

\par\noindent\begin{minipage}{\linewidth}
\begin{center}
\captionsetup{type=table}
\captionof{table}{
Low-resource data-fraction sweep on MRPC and RTE for SkewGrad-support variants at fixed rank \(r=46\). All values are mean \(\pm\) standard deviation over five random seeds. \(\loft^{\pskew}\) uses the default orthogonal transform \(T_r=\Torth\); \(\loft^{\pskew,\Tfree}\) uses the unconstrained free transform.
}
\label{tab:skew_data_fraction}
\scriptsize
\setlength{\tabcolsep}{4.2pt}
\renewcommand{\arraystretch}{1.04}
\begin{adjustbox}{max width=\linewidth}
\begin{tabular}{llccc}
\toprule
Task & Fraction & \(\loft_{r=46}^{\scriptscriptstyle \pskew}\) & \(\loft_{r=46}^{\scriptscriptstyle \pskew,\,\Tfree}\) & PSOFT$_{r=46}$ \\
\midrule
MRPC & 10\%  & \(0.8394 \pm 0.0264\) & \(\mathbf{0.8502 \pm 0.0144}\) & \(0.8171 \pm 0.0199\) \\
MRPC & 20\%  & \(\mathbf{0.8640 \pm 0.0107}\) & \(0.8636 \pm 0.0159\) & \(0.8487 \pm 0.0102\) \\
MRPC & 40\%  & \(\mathbf{0.8952 \pm 0.0201}\) & \(0.8840 \pm 0.0104\) & \(0.8692 \pm 0.0087\) \\
MRPC & 60\%  & \(\mathbf{0.8988 \pm 0.0068}\) & \(0.8984 \pm 0.0164\) & \(0.8751 \pm 0.0085\) \\
MRPC & 80\%  & \(\mathbf{0.9080 \pm 0.0085}\) & \(0.9017 \pm 0.0073\) & \(0.8982 \pm 0.0129\) \\
MRPC & 100\% & \(0.9236 \pm 0.0150\) & \(\mathbf{0.9265 \pm 0.0206}\) & \(0.9086 \pm 0.0176\) \\
\midrule
RTE & 10\%  & \(0.6619 \pm 0.0426\) & \(0.6691 \pm 0.0326\) & \(\mathbf{0.6705 \pm 0.0405}\) \\
RTE & 20\%  & \(0.7252 \pm 0.0186\) & \(\mathbf{0.7381 \pm 0.0242}\) & \(0.7262 \pm 0.0271\) \\
RTE & 40\%  & \(\mathbf{0.8086 \pm 0.0320}\) & \(0.8072 \pm 0.0276\) & \(0.7897 \pm 0.0078\) \\
RTE & 60\%  & \(0.8158 \pm 0.0187\) & \(\mathbf{0.8388 \pm 0.0064}\) & \(0.8115 \pm 0.0372\) \\
RTE & 80\%  & \(\mathbf{0.8302 \pm 0.0195}\) & \(0.8273 \pm 0.0176\) & \(0.8285 \pm 0.0168\) \\
RTE & 100\% & \(\mathbf{0.8683 \pm 0.0343}\) & \(0.8677 \pm 0.0262\) & \(0.8675 \pm 0.0236\) \\
\bottomrule
\end{tabular}
\end{adjustbox}
\end{center}
\end{minipage}\par\medskip

\begin{center}
\captionsetup{type=figure}
\begin{minipage}{0.475\linewidth}
    \centering
    \includegraphics[width=\linewidth]{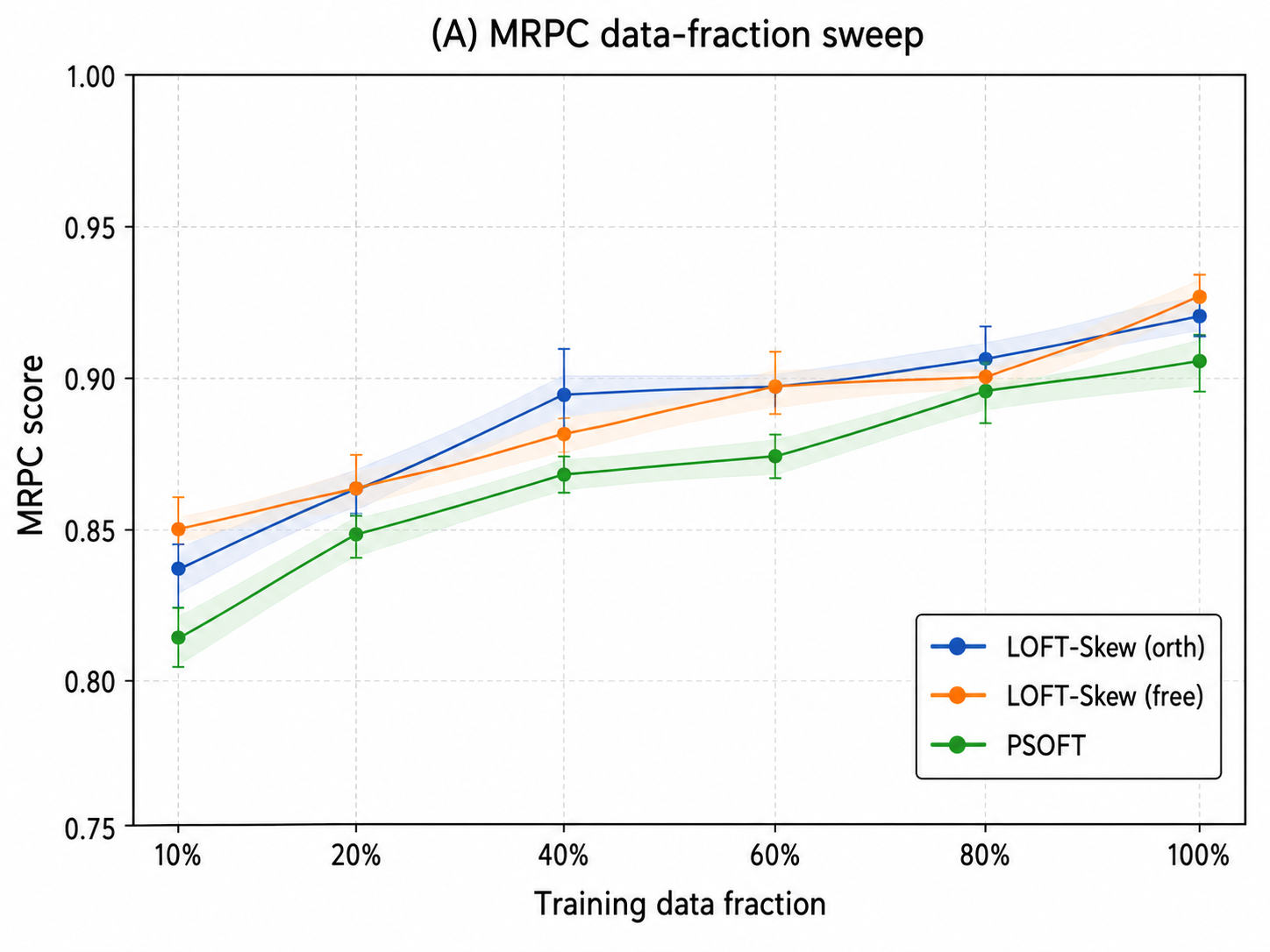}
\end{minipage}
\hfill
\begin{minipage}{0.475\linewidth}
    \centering
    \includegraphics[width=\linewidth]{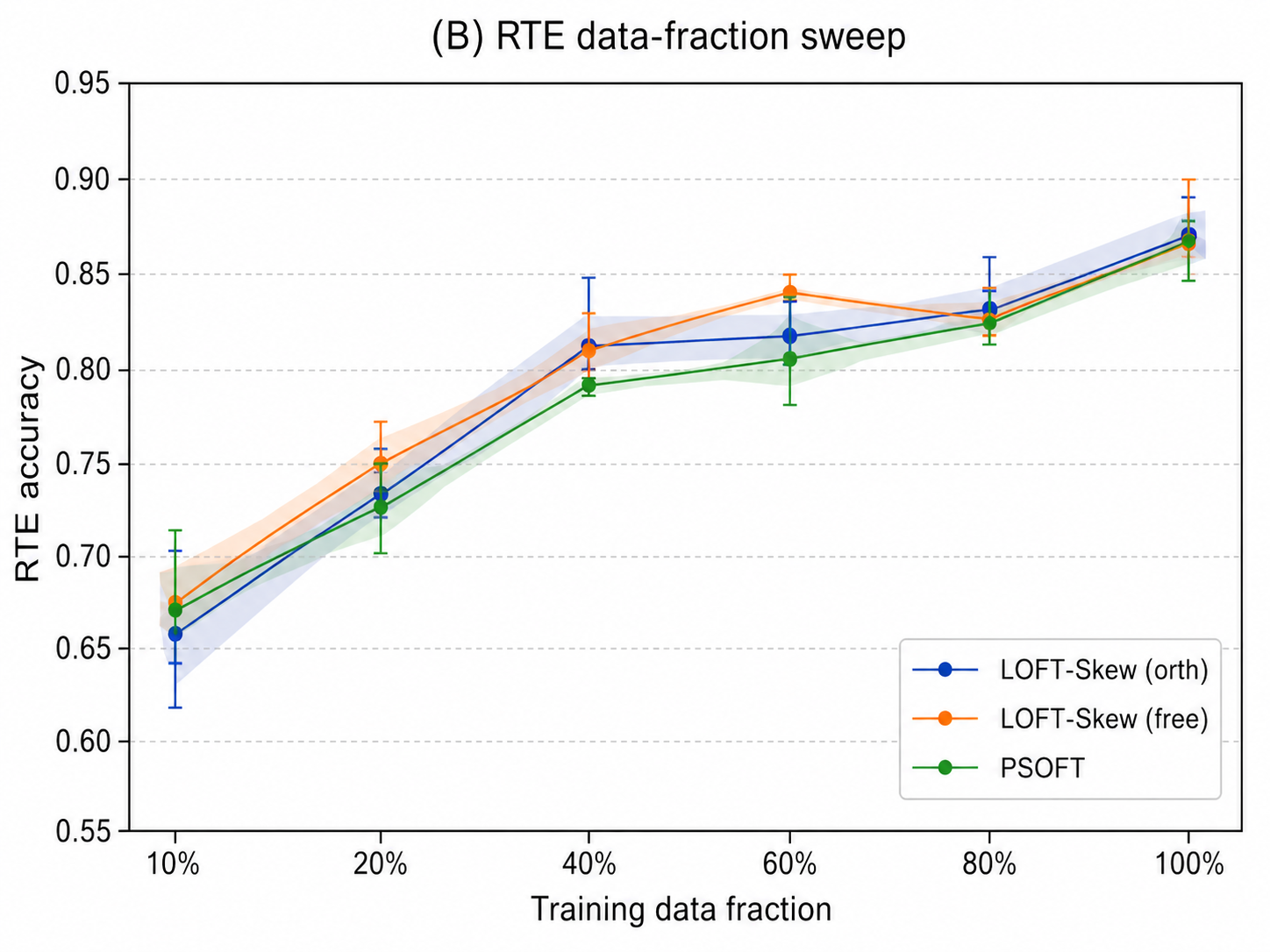}
\end{minipage}
\captionof{figure}{
\textbf{Low-resource data-fraction robustness.}
Panels (A) and (B) visualize the MRPC and RTE data-fraction sweeps in Table~\ref{tab:skew_data_fraction}. Curves show mean performance over five seeds, with uncertainty indicating across-seed standard deviation.
}
\label{fig:skew_data_fraction}
\end{center}

Table~\ref{tab:skew_data_fraction} and Figure~\ref{fig:skew_data_fraction} show that the benefits of the theory-guided SkewGrad support are not confined to the full-data setting. On MRPC, both SkewGrad-support LOFT variants outperform PSOFT across all data fractions, with especially clear margins in the intermediate 40--60\% regime and at full data. On RTE, the very low-resource regime is less stable at 10--20\%, but from 40\% onward the SkewGrad variants are consistently competitive and often stronger than PSOFT. These results support the view that loss-informed support selection remains useful when supervision is reduced.

\begin{center}
\captionsetup{type=table}
\captionof{table}{
Full-data rank sweep on MRPC and RTE for SkewGrad-support variants. All values are mean \(\pm\) standard deviation over five random seeds. The sweep tests whether SkewGrad's gains depend on the main rank \(r=46\) or persist across a broader parameter budget.
}
\label{tab:skew_rank_sweep}
\scriptsize
\setlength{\tabcolsep}{4.2pt}
\renewcommand{\arraystretch}{1.04}
\begin{adjustbox}{max width=\linewidth}
\begin{tabular}{llccc}
\toprule
Task & Rank \(r\) & \(\loft^{\scriptscriptstyle \pskew}\) & \(\loft^{\scriptscriptstyle \pskew,\,\Tfree}\) & PSOFT \\
\midrule
MRPC & 16 & \(0.9048 \pm 0.0072\) & \(\mathbf{0.9189 \pm 0.0106}\) & \(0.8756 \pm 0.0068\) \\
MRPC & 46 & \(0.9236 \pm 0.0150\) & \(\mathbf{0.9265 \pm 0.0206}\) & \(0.9086 \pm 0.0176\) \\
MRPC & 64 & \(\mathbf{0.9183 \pm 0.0087}\) & \(0.9148 \pm 0.0103\) & \(0.9085 \pm 0.0068\) \\
\midrule
RTE & 16 & \(0.7986 \pm 0.0337\) & \(\mathbf{0.8259 \pm 0.0298}\) & \(0.7330 \pm 0.1507\) \\
RTE & 46 & \(\mathbf{0.8683 \pm 0.0343}\) & \(0.8677 \pm 0.0262\) & \(0.8675 \pm 0.0236\) \\
RTE & 64 & \(\mathbf{0.8660 \pm 0.0316}\) & \(0.8631 \pm 0.0186\) & \(0.8414 \pm 0.0194\) \\
\bottomrule
\end{tabular}
\end{adjustbox}
\end{center}

\par\noindent\begin{minipage}{\linewidth}
\begin{center}
\captionsetup{type=figure}
\begin{minipage}{0.475\linewidth}
    \centering
    \includegraphics[width=\linewidth]{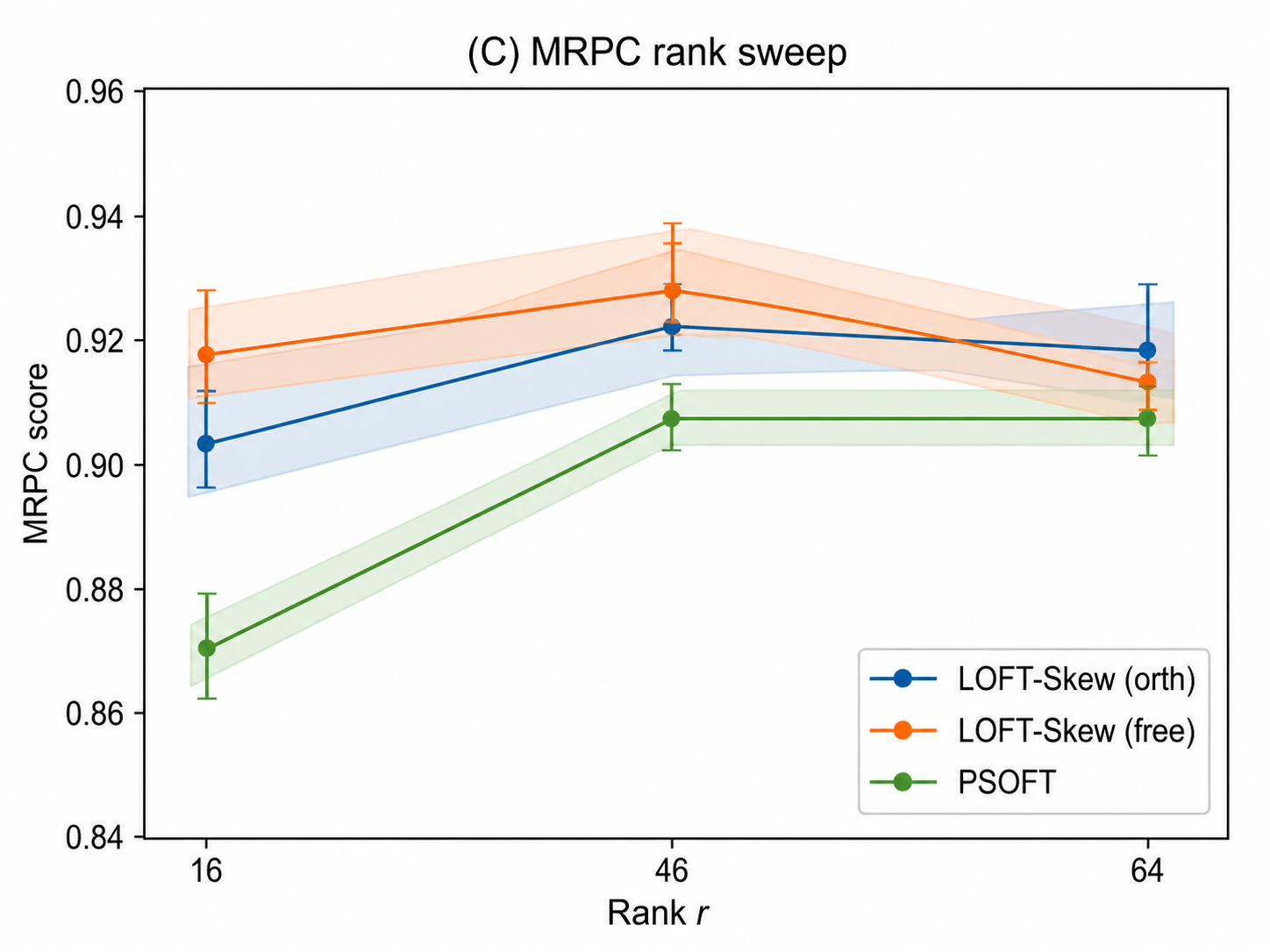}
\end{minipage}
\hfill
\begin{minipage}{0.475\linewidth}
    \centering
    \includegraphics[width=\linewidth]{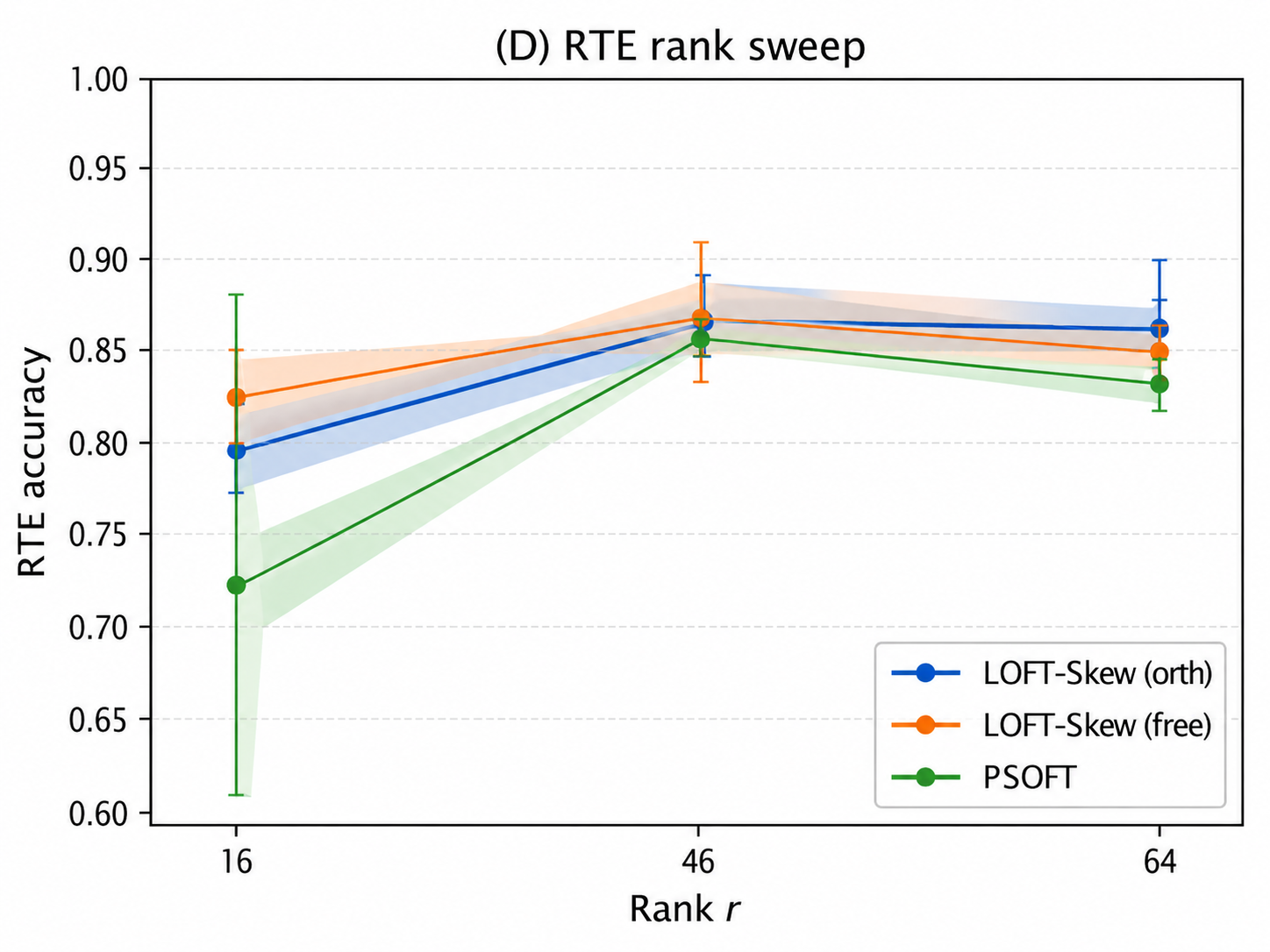}
\end{minipage}
\captionof{figure}{
\textbf{Rank-sweep robustness.}
Panels (C) and (D) visualize the MRPC and RTE rank sweeps in Table~\ref{tab:skew_rank_sweep}. Curves show mean performance over five seeds, with uncertainty indicating across-seed standard deviation.
}
\label{fig:skew_rank_sweep}
\end{center}
\end{minipage}\par\medskip

Table~\ref{tab:skew_rank_sweep} and Figure~\ref{fig:skew_rank_sweep} show that the SkewGrad gains are not tied to a single narrowly tuned rank. On MRPC, both SkewGrad-support LOFT variants outperform PSOFT at all three ranks, with the strongest results around the main rank \(r=46\). On RTE, SkewGrad substantially improves over PSOFT at \(r=16\), remains competitive at \(r=46\), and stays clearly stronger at \(r=64\). Taken together, the data-fraction and rank sweeps indicate that the benefit of the theory-guided support persists across both supervision levels and parameter budgets, rather than arising only from one carefully chosen operating point.

\section{Additional Efficiency Measurements}
\label{app:glue_efficiency_grad}

We further report supplementary wall-clock measurements on the six-task GLUE setup at rank \(r=46\). 
The goal is to quantify the training-time and parameter-cost profile of three support families under the same task-specific training protocol: principal support \(P_r=\pprin\), GradSVD support \(P_r=\pgrad\), and SkewGrad support \(P_r=\pskew\). 
For principal support, we additionally compare against PSOFT at the same rank. 
For GradSVD and SkewGrad, the reported wall-clock runtime includes the one-off calibration stage used to construct the fixed support before training the PEFT patch.

Because the support only changes the fixed basis \(P_r\), the trainable parameter count is unchanged across support families for a fixed transform class and rank. 
Accordingly, all orthogonal LOFT variants at rank \(r=46\) use \(74{,}520\) non-classifier trainable parameters, while all free-transform LOFT variants at the same rank use \(152{,}352\) non-classifier trainable parameters. 
PSOFT at rank \(r=46\) uses \(81{,}144\) non-classifier trainable parameters.

\begin{center}
\captionsetup{type=table}
\captionof{table}{
End-to-end wall-clock training runtime in seconds on the six-task GLUE setup at rank \(r=46\).
All runs use the same task-specific training protocol and the same NVIDIA Tesla V100-SXM2 GPU setting as Table~\ref{tab:glue_main}.
Runtime is measured from single-seed runs and rounded to the nearest second.
Trainable Params reports non-classifier trainable parameters.
For GradSVD and SkewGrad supports, runtime includes the one-off calibration pass used to construct the support.
For LOFT, \(T_r=\Torth\) is used by default unless \(T_r=\Tfree\) is shown.
}
\label{tab:glue_efficiency_grad}
\scriptsize
\setlength{\tabcolsep}{3.0pt}
\renewcommand{\arraystretch}{1.05}
\begin{adjustbox}{width=\textwidth}
\renewcommand{\arraystretch}{1.25}
\begin{tabular}{l|lr|ccccccc}
\toprule
Method / support & Transform & Trainable Params & CoLA & STS-B & RTE & MRPC & SST-2 & QNLI & \cellcolor{avgblue}Avg. \\
\midrule
PSOFT$_{r=46}$ & Ortho & 81,144
& 3242 & 2879 & 3927 & 5793 & 15609 & 29694 & \cellcolor{avgblue}10191 \\ 
\midrule
\loft$_{r=46}^{\scriptscriptstyle \pprin}$ & \(\Torth\) & 74,520
& 2824 & 2556 & 3926 & 5479 & 13660 & 28726 & \cellcolor{avgblue}9528 \\
\loft$_{r=46}^{\scriptscriptstyle \pgrad}$ & \(\Torth\) & 74,520
& 2811 & 2610 & 3973 & 5666 & 13933 & 28274 & \cellcolor{avgblue}9544 \\
\loft$_{r=46}^{\scriptscriptstyle \pskew}$ & \(\Torth\) & 74,520
& 2819 & 2640 & 3936 & 5765 & 13954 & 28305 & \cellcolor{avgblue}9570 \\
\midrule
\loft$_{r=46}^{\scriptscriptstyle \pprin}$ & \(\Tfree\) & 152,352
& 1395 & 1893 & 3372 & 4911 & 10295 & 23161 & \cellcolor{avgblue}7504 \\
\loft$_{r=46}^{\scriptscriptstyle \pgrad}$ & \(\Tfree\) & 152,352
& 1451 & 2030 & 3574 & 5176 & 10412 & 26069 & \cellcolor{avgblue}8119 \\
\loft$_{r=46}^{\scriptscriptstyle \pskew}$ & \(\Tfree\) & 152,352
& 1483 & 1957 & 3628 & 5188 & 10321 & 26074 & \cellcolor{avgblue}8108 \\
\bottomrule
\end{tabular}
\end{adjustbox}
\end{center}

\paragraph{Efficiency analysis.}
Table~\ref{tab:glue_efficiency_grad} highlights three points:

First, in the principal-support regime, orthogonal LOFT uses slightly fewer non-classifier trainable parameters than PSOFT at the same rank: \(74{,}520\) versus \(81{,}144\), an \(8.2\%\) reduction. 
It is also consistently faster than PSOFT in wall-clock time across the six GLUE tasks. 
This indicates that the additional PSOFT orthogonal transformation and magnitude factors introduce non-negligible per-step overhead, even though both methods operate at the same rank.

Second, \(T_r=\Tfree\) uses more trainable parameters than \(T_r=\Torth\), increasing the non-classifier LOFT parameter count from \(74{,}520\) to \(152{,}352\). 
Despite this larger parameter count, the free-transform variants are substantially faster than their orthogonal counterparts on all six tasks. 
This shows that runtime is not determined only by the number of trainable parameters; the cost of parameterizing the in-subspace orthogonal transform can dominate the small dense \(r\times r\) update.

Third, gradient-informed support construction introduces little additional runtime relative to principal support. 
Among the orthogonal LOFT variants, GradSVD and SkewGrad require only \(0.2\%\) and \(0.4\%\) additional average runtime relative to principal support, respectively. 
GradSVD and SkewGrad also have nearly identical wall-clock runtime at matched transform type, showing that the extra cost of forming \(\mathrm{skew}(W_0^\top G)\) is negligible once calibration gradients have been collected. 
Even SkewGrad, the slowest orthogonal LOFT variant in average wall-clock time, remains \(6.1\%\) faster on average than PSOFT.

\subsection{Support Construction Cost}
\label{app:construction_cost}
Calibration has two one-off phases before training: downstream-gradient collection, and an FP32 spectral construction run one target matrix at a time, after which \(P_r\) is fixed; \(S\) and \(C\) are materialized per target, once. Table~\ref{tab:construction_cost} reports both phases. LLaMA2-7B is measured on a single A100-SXM 80GB GPU with four size-4 calibration batches.
\begin{center}
\captionsetup{type=table}
\captionof{table}{One-off support construction: time and incremental peak GPU allocation per phase. Memory is reported in GiB as measured.}
\label{tab:construction_cost}
\small
\begin{adjustbox}{max width=\linewidth}
\begin{tabular}{lcccc}
\toprule
Setting & Target matrices & Gradient collection & Spectral construction & Total one-off cost \\
\midrule
ViT-B/16, \(r=46\) & 72 & 6.45\,s; \(+3.71\)\,GiB peak & 3.63\,s; \(+0.269\)\,GiB peak & \textbf{10.08\,s} \\
LLaMA2-7B, \(r=128\) & 64 & 29.21\,s; \(+19.30\)\,GiB peak & 7.75\,s; \(+0.623\)\,GiB peak & \textbf{36.96\,s} \\
\bottomrule
\end{tabular}
\end{adjustbox}
\end{center}
The eigendecomposition accounts for \(3.449\)s of the ViT construction and \(7.172\)s of the LLaMA one; the remainder is indexing and casting. LLaMA2-7B peaks at \(44.66\)GiB during gradient collection and \(26.04\)GiB during construction. Under the matched 19-task VTAB timing protocol on the same V100 (Table~\ref{tab:vtab_timing}), SkewGrad LOFT including calibration is \(10.8\%\) faster per task than PSOFT.
\begin{center}
\captionsetup{type=table}
\captionof{table}{Same-codebase VTAB-1K timing audit over all 19 tasks; times are per-task means.}
\label{tab:vtab_timing}
\small
\begin{tabular}{lccc}
\toprule
Method & Training loop per task & Calibration per task & Total \\
\midrule
PSOFT & 643.06\,s & -- & 643.06\,s \\
SkewGrad LOFT & 563.57\,s & 10.08\,s & 573.65\,s \\
\bottomrule
\end{tabular}
\end{center}
The \(14.3\%\) lower training-loop peak memory of SkewGrad LOFT on VTAB (\(5.70\) vs.\ \(6.65\)GB; \(8.1\%\) against the \(6.2\)GB PSOFT value published by \citet{wu2025memory}) is the training-loop figure only; the one-off calibration phase carries its own transient peak (gradient collection adds \(+3.71\)GiB on ViT-B/16 and reaches \(44.66\)GiB on LLaMA2-7B), which should be counted separately in an end-to-end budget. Construction and training scale differently with model size, so we report them separately.

\subsection{Peak Memory under One Instrument}
Table~\ref{tab:mem_remeasure} remeasures training-loop peak memory for each method under one instrumentation (same V100, same batch sizes, \texttt{max\_memory\_allocated()}). Among these methods, SkewGrad LOFT has the lowest peak on both benchmarks. The baseline memory values marked $^{\dagger}$ in Tables~\ref{tab:glue_main} and~\ref{tab:vtab_main} are the published ones.
\begin{center}
\captionsetup{type=table}
\captionof{table}{Training-loop peak GPU memory (GB) remeasured under one instrument.}
\label{tab:mem_remeasure}
\small
\begin{tabular}{lcc}
\toprule
Method & GLUE peak memory (GB) & VTAB peak memory (GB) \\
\midrule
PSOFT$_{r=46}$ & 4.10 & 6.65 \\
LoRA$_{r=8}$ & 4.49 & 9.94 \\
DoRA$_{r=8}$ & 5.81 & 17.82 \\
BOFT$_{m=2}^{b=8}$ & 6.34 & 9.29 \\
\loft$_{r=46}^{\scriptscriptstyle \pskew}$ & 4.05 & 5.70 \\
\bottomrule
\end{tabular}
\end{center}

\newpage
\section*{NeurIPS Paper Checklist}

\begin{enumerate}

\item {\bf Claims}
    \item[] Question: Do the main claims made in the abstract and introduction accurately reflect the paper's contributions and scope?
    \item[] Answer: \answerYes{}
    \item[] Justification: The abstract and introduction state the paper's main claims: LOFT separates adaptation support from the in-subspace transform, derives task-aware gradient-informed support criteria, and evaluates support selection at matched trainable-parameter and transform settings under a shared protocol.
    The empirical claim is limited to the controlled finding that support choice affects performance under a fixed transform and budget, with the preferred gradient-informed construction task-dependent. These claims are supported by the methodology, theory, and experiments in Sections 2-3 and the appendices.

\item {\bf Limitations}
    \item[] Question: Does the paper discuss the limitations of the work performed by the authors?
    \item[] Answer: \answerYes{}
    \item[] Justification: The Limitations paragraph in Section 4 states that the support-selection criterion is first-order and initialization-local and covers only the orthogonal, skew-parameterized transform; that the preferred gradient-informed construction is task-dependent, including the multilingual setting where SkewGrad trails the principal support; that the fixed support is stored task-specific state; that gradient-informed supports need a one-off calibration step whose cost is reported separately; and that the decoder-side reasoning experiments use one matched low-budget setting with single-run baselines and no model larger than LLaMA2-7B.

\item {\bf Theory assumptions and proofs}
    \item[] Question: For each theoretical result, does the paper provide the full set of assumptions and a complete (and correct) proof?
    \item[] Answer: \answerYes{}
    \item[] Justification: The paper states the assumptions for the LOFT parameterization and first-order support-selection analysis, including row-orthonormal support bases, orthogonal in-subspace transforms, and initialization at the pretrained weight. Proofs and derivations are provided in the methodology section and appendix.

    \item {\bf Experimental result reproducibility}
    \item[] Question: Does the paper fully disclose all the information needed to reproduce the main experimental results of the paper to the extent that it affects the main claims and/or conclusions of the paper (regardless of whether the code and data are provided or not)?
    \item[] Answer: \answerYes{}
    \item[] Justification: The paper reports the model backbones, datasets, evaluation metrics, trainable parameter budgets, support-construction rules, checkpoint-selection protocols, learning rates, ranks, batch sizes, training epochs, and evaluation pipelines in Section 3 and the appendices.

\item {\bf Open access to data and code}
    \item[] Question: Does the paper provide open access to the data and code, with sufficient instructions to faithfully reproduce the main experimental results, as described in supplemental material?
    \item[] Answer: \answerYes{}
    \item[] Justification: The LOFT adapter code, all four support constructions (principal, random, GradSVD and SkewGrad) and the GLUE training and evaluation pipeline are publicly available at \url{https://github.com/Kris-camp/loft}. The experiments use publicly available datasets and pretrained models, with dataset access and preprocessing details described in the appendix.

\item {\bf Experimental setting/details}
    \item[] Question: Does the paper specify all the training and test details (e.g., data splits, hyperparameters, how they were chosen, type of optimizer) necessary to understand the results?
    \item[] Answer: \answerYes{}
    \item[] Justification: The experimental settings are described in Section 3 and the appendices, including backbone models, datasets, task metrics, trainable parameter counts, memory measurements, support-selection procedures, learning rates, ranks, batch sizes, seeds, and checkpoint-selection rules.

\item {\bf Experiment statistical significance}
    \item[] Question: Does the paper report error bars suitably and correctly defined or other appropriate information about the statistical significance of the experiments?
    \item[] Answer: \answerYes{}
    \item[] Justification: GLUE and VTAB results are five-seed means; across-seed sample standard deviations are reported for the GLUE LOFT variants, for four VTAB LOFT variants and for the same-codebase GLUE baselines (Appendices E and G). For mathematical reasoning, LOFT rows are means \(\pm\) sample standard deviation over three runs and baseline rows are single runs; the multilingual results are three-seed means with across-seed standard deviations, and data-fraction and rank sweeps report mean and standard deviation over five seeds. Every \(\pm\) is a sample standard deviation across runs.

\item {\bf Experiments compute resources}
    \item[] Question: For each experiment, does the paper provide sufficient information on the computer resources (type of compute workers, memory, time of execution) needed to reproduce the experiments?
    \item[] Answer: \answerYes{}
    \item[] Justification: The paper reports GPU types, peak memory, trainable parameter counts, and additional runtime measurements. The supplementary GLUE efficiency study explicitly reports end-to-end runtime and notes that GradSVD and SkewGrad include the one-off calibration pass used to construct support. The one-off support-construction time and transient memory for ViT-B/16 and LLaMA2-7B are reported separately from the training loop (Appendix~\ref{app:construction_cost}).
    
\item {\bf Code of ethics}
    \item[] Question: Does the research conducted in the paper conform, in every respect, with the NeurIPS Code of Ethics \url{https://neurips.cc/public/EthicsGuidelines}?
    \item[] Answer: \answerYes{}
    \item[] Justification: The work follows the NeurIPS Code of Ethics. It studies parameter-efficient fine-tuning algorithms using public benchmark datasets and pretrained models, and does not involve private user data, human-subject experiments, surveillance, deception, or deployment on sensitive populations.

\item {\bf Broader impacts}
    \item[] Question: Does the paper discuss both potential positive societal impacts and negative societal impacts of the work performed?
    \item[] Answer: \answerYes{}
    \item[] Justification: The Broader impacts paragraph at the end of Appendix~\ref{sec:related_work} discusses that LOFT reduces the trainable parameters and training-loop memory needed to adapt pretrained models, which may lower adaptation cost and improve accessibility, and that, as with other PEFT methods, easier adaptation could also enable harmful uses or adaptation on biased data.
    
\item {\bf Safeguards}
    \item[] Question: Does the paper describe safeguards that have been put in place for responsible release of data or models that have a high risk for misuse (e.g., pre-trained language models, image generators, or scraped datasets)?
    \item[] Answer: \answerNA{}
    \item[] Justification: The paper does not release a pretrained foundation model, scraped dataset, image generator, or other high-risk asset. The released material consists of PEFT training code and experiment scripts.

\item {\bf Licenses for existing assets}
    \item[] Question: Are the creators or original owners of assets (e.g., code, data, models), used in the paper, properly credited and are the license and terms of use explicitly mentioned and properly respected?
    \item[] Answer: \answerYes{}
    \item[] Justification: We use DeBERTaV3-base \citep{he2021debertav3} (MIT license), ViT-B/16 \citep{dosovitskiy2021vit} (Apache-2.0), LLaMA2-7B \citep{touvron2023llama2} (Llama 2 Community License) and Llama-3.2-3B \citep{meta2024llama32} (Llama 3.2 Community License); the MetaMathQA \citep{yu2024metamath}, GSM8K \citep{cobbe2021gsm8k} and MATH \citep{hendrycks2021math} datasets (MIT) and Bactrian-X \citep{li2023bactrianx} (CC BY-NC 4.0, used for non-commercial research); and the GLUE \citep{wang2018glue} and VTAB-1K \citep{zhai2019vtab} benchmarks, whose tasks are used under the terms of their constituent datasets as specified by the benchmark releases. Each asset is cited where it is used. Our GLUE training script is modified from the HuggingFace Transformers example (Apache-2.0) via the PSOFT implementation (MIT); both notices are retained in the code release. Baseline numbers taken from prior work are attributed to \citet{wu2025memory}.

\item {\bf New assets}
    \item[] Question: Are new assets introduced in the paper well documented and is the documentation provided alongside the assets?
    \item[] Answer: \answerYes{}
    \item[] Justification: The paper introduces LOFT adapter code and support-selection procedures. The public release at \url{https://github.com/Kris-camp/loft} documents the GLUE training commands, evaluation scripts, and hyperparameters; limitations are discussed in Section 4.

\item {\bf Crowdsourcing and research with human subjects}
    \item[] Question: For crowdsourcing experiments and research with human subjects, does the paper include the full text of instructions given to participants and screenshots, if applicable, as well as details about compensation (if any)? 
    \item[] Answer: \answerNA{}
    \item[] Justification: The paper does not involve crowdsourcing, newly collected human-subject data, or human-subject experiments.

\item {\bf Institutional review board (IRB) approvals or equivalent for research with human subjects}
    \item[] Question: Does the paper describe potential risks incurred by study participants, whether such risks were disclosed to the subjects, and whether Institutional Review Board (IRB) approvals (or an equivalent approval/review based on the requirements of your country or institution) were obtained?
    \item[] Answer: \answerNA{}
    \item[] Justification: The paper does not involve human-subject research, crowdsourcing, or newly collected participant data, so IRB approval is not applicable.

\item {\bf Declaration of LLM usage}
    \item[] Question: Does the paper describe the usage of LLMs if it is an important, original, or non-standard component of the core methods in this research? Note that if the LLM is used only for writing, editing, or formatting purposes and does \emph{not} impact the core methodology, scientific rigor, or originality of the research, declaration is not required.
    \item[] Answer: \answerNA{}
    \item[] Justification: LLMs are used only as pretrained model backbones in standard benchmark fine-tuning experiments, not as tools for generating datasets, making research decisions, or producing experimental labels.

\end{enumerate}

\end{document}